\documentclass[jair,twoside,11pt,theapa]{article}
\usepackage{jair, theapa, rawfonts}

\jairheading{77}{2023}{1539--1589}{10/2022}{08/2023}
\ShortHeadings{Certified Dominance and Symmetry Breaking for Combinatorial Optimisation}
{Bogaerts, Gocht, McCreesh, \& Nordstr\"om}
\firstpageno{1539}

\usepackage[utf8]{inputenc}
\PassOptionsToPackage{hyphens}{url}

\usepackage{amsthm}
\usepackage{amssymb}
\usepackage{amsmath}
\usepackage{bussproofs}
\usepackage{braket}

\usepackage[capitalize,noabbrev]{cleveref}

\usepackage{paralist} 
 \usepackage{url}
 \usepackage{amsmath}
 \usepackage{amssymb}
 \usepackage{ifthen}
 \usepackage{xspace}
 \usepackage[dvipsnames]{xcolor}
\newcommand{\m}[1]{\ensuremath{#1}\xspace}

 \usepackage[normalem]{ulem}

	\newcommand{\lequiv}{\Leftrightarrow}

	\newcommand{\ltrue}{\trval{t}}
	\newcommand{\lfalse}{\trval{f}}

	\newcommand{\ignore}[1]{}

	\newboolean{nocomments}
	\setboolean{nocomments}{false}

	\newboolean{commentmargin}
	\setboolean{commentmargin}{true}

	\newcommand{\namedcomment}[3]{%
		\ifthenelse{\boolean{nocomments}}%
		{}
		{
			\ifthenelse{\boolean{commentmargin}}%
				{ {\color{#3} \marginpar{\color{#3}\sc #2}
\sffamily
#1}  }
				{  {\color{#3} {\sc #2}:
\sffamily
 #1}  }
		}%
	}

\usepackage{etoolbox}

\newcommand\setcitation[2]{%
  \csdef{mycommoncitation#1}{#2}}
\newcommand\getcitation[1]{%
  \csuse{mycommoncitation#1}}

\setcitation{IDP}{WarrenBook/DeCatBBD14}
\setcitation{idp}{WarrenBook/DeCatBBD14}
\setcitation{fodot}{tocl/DeneckerT08}
\setcitation{foid}{tocl/DeneckerT08}
\setcitation{FOID}{tocl/DeneckerT08}
\setcitation{cplogic}{journal/tplp/VennekensDB10}
\setcitation{CPlogic}{journal/tplp/VennekensDB10}
\setcitation{CPLogic}{journal/tplp/VennekensDB10}
\setcitation{CP}{fai/Rossi06}
\setcitation{cp}{fai/Rossi06}
\setcitation{EZCSP}{lpnmr/Balduccini11}
\setcitation{KR}{Baral:2003}
\setcitation{ASPComp2}{lpnmr/DeneckerVBGT09}
\setcitation{ASPComp3}{journals/tplp/CalimeriIR14}
\setcitation{ASPComp4}{conf/lpnmr/AlvianoCCDDIKKOPPRRSSSWX13}
\setcitation{ASPComp5}{journals/ai/CalimeriGMR16}
\setcitation{ASPComp6}{jair/GebserMR17}
\setcitation{ASPComp7}{tplp/GebserMR20}
\setcitation{CPSupport}{ictai/DeCat13}
\setcitation{CPsupport}{ictai/DeCat13}
\setcitation{functionDetection}{iclp/DeCatB13}
\setcitation{FunctionDetection}{iclp/DeCatB13}
\setcitation{fodot2asp}{corr/DeneckerLTV19} 
\setcitation{Tarskian}{corr/DeneckerLTV19} 
\setcitation{TarskianSemanticsASP}{corr/DeneckerLTV19} 
\setcitation{Inca}{iclp/DrescherW12}
\setcitation{csp2asp}{ijcai/DrescherW11}
\setcitation{DPLLT}{cav/GanzingerHNOT04}
\setcitation{AspInPractice}{synthesis/2012Gebser}
\setcitation{ASPInPractice}{synthesis/2012Gebser}
\setcitation{clasp}{ai/GebserKS12}
\setcitation{oclingo}{kr/GebserGKOSS12}
\setcitation{clingo}{iclp/GebserKKOSW16}
\setcitation{gringo}{lpnmr/GebserST07}
\setcitation{cmodels}{aaai/GiunchigliaLM04}
\setcitation{inputster}{tplp/Jansen13}
\setcitation{DLV}{tocl/LeonePFEGPS06}
\setcitation{LearningPaper}{TPLP/BruynoogheBBDDJLRDV} 
\setcitation{clog}{iclp/BogaertsVDV14} 
\setcitation{foc}{iclp/BogaertsVDV14}
\setcitation{FOC}{iclp/BogaertsVDV14}
\setcitation{inferenceClog}{ecai/BogaertsVDV14}
\setcitation{examplesClog}{nmr/BogaertsVDV14b} 
\setcitation{AFT}{DeneckerMT00}
\setcitation{KBS}{iclp/DeneckerV08}
\setcitation{KBS-invitedtalk}{jelia/Denecker16}
\setcitation{KBPE}{inap/DePooterWD11}
\setcitation{lazyGrounding}{jair/CatDBS15} 
\setcitation{LazyGrounding}{jair/CatDBS15} 
\setcitation{lazygrounding}{jair/CatDBS15} 
\setcitation{lazygroundingASP}{ijcai/BogaertsW18} 
\setcitation{justifications}{lpnmr/DeneckerBS15} 
\setcitation{justificationsAlpha}{ijcai/BogaertsW18} 
\setcitation{ASP}{marek99stable}
\setcitation{satid}{sat/MarienWDB08}
\setcitation{lazyclausegeneration}{constraints/OhrimenkoSC09}
\setcitation{FP}{ACMCS/Hudak89}
\setcitation{GroundingWithBounds}{jair/WittocxMD10}
\setcitation{GroundWithBounds}{jair/WittocxMD10}
\setcitation{SAT}{faia/SilvaLM09}
\setcitation{HandbookOfSAT}{faia/2009-185}
\setcitation{LTC}{iclp/Bogaerts14}
\setcitation{SPSAT}{ictai/DevriendtBMDD12}
\setcitation{BreakID}{DBBD16ImprovedStatic} 
\setcitation{breakid}{DBBD16ImprovedStatic}
\setcitation{LCG}{stuckeyLCG}
\setcitation{MiniZinc}{conf/cp/NethercoteSBBDT07}
\setcitation{minizinc}{conf/cp/NethercoteSBBDT07}
\setcitation{amadini}{cpaior/AmadiniGM13}
\setcitation{bootstrapping}{ngc/BogaertsJDJBD16}
\setcitation{Bootstrapping}{ngc/BogaertsJDJBD16}
\setcitation{GroundedFixpoints}{ai/BogaertsVD15}
\setcitation{PartialGroundedFixpoints}{ijcai/BogaertsVD15}
\setcitation{LogicBlox}{datalog/GreenAK12}
\setcitation{proB}{journals/sttt/LeuschelB08}
\setcitation{NaturalInductions}{KR/DeneckerV14} 
\setcitation{LP}{jacm/EmdenK76}
\setcitation{SMT}{faia/BarrettSST09}
\setcitation{AF}{ai/Dung95}
\setcitation{ADF}{kr/BrewkaW10}
\setcitation{af}{ai/Dung95}
\setcitation{adf}{kr/BrewkaW10}
\setcitation{ADFRevisited}{ijcai/BrewkaSEWW13}
\setcitation{adfrevisited}{ijcai/BrewkaSEWW13}
\setcitation{DefaultLogic}{ai/Reiter80}
\setcitation{DL}{ai/Reiter80}
\setcitation{AEL}{mo85}
\setcitation{minisat}{sat/EenS03}
\setcitation{completion}{adbt/Clark78}
\setcitation{ClarkCompletion}{adbt/Clark78}
\setcitation{wasp}{lpnmr/AlvianoDFLR13}
\setcitation{minisatid}{ictai/DeCat13}
\setcitation{lcg}{stuckeyLCG}
\setcitation{CEGAR}{jacm/ClarkeGJLV03}
\setcitation{cegar}{jacm/ClarkeGJLV03}
\setcitation{CuttingPlane}{or/DantzigFJ54}
\setcitation{kodkod}{tacas/TorlakJ07}
\setcitation{cdcl}{Marques-SilvaS99}
\setcitation{CDCL}{Marques-SilvaS99}
\setcitation{1UIP}{iccad/ZhangMMM01}
\setcitation{relevance}{ijcai/JansenBDJD16}
\setcitation{relevance-implementation}{aspocp/JansenBDJD16}
\setcitation{WFS}{GelderRS91}
\setcitation{wfs}{GelderRS91}
\setcitation{UnfoundedSet}{GelderRS91}
\setcitation{UFS}{GelderRS91}
\setcitation{stablesemantics}{iclp/GelfondL88}
\setcitation{StableSemantics}{iclp/GelfondL88}
\setcitation{shatter}{Shatter}
\setcitation{sbass}{drtiwa11a}
\setcitation{lparsemanual}{url:lparsemanual}
\setcitation{AIC}{ppdp/FlescaGZ04}
\setcitation{templates}{tplp/DassevilleHJD15}
\setcitation{templates2}{iclp/DassevilleHBJD16}
\setcitation{sat-to-sat}{aaai/JanhunenTT16}
\setcitation{sat-to-sat-qbf}{bnp/BogaertsJT16}
\setcitation{sat-to-sat-QBF}{bnp/BogaertsJT16}
\setcitation{sat-to-sat-SO}{kr/BogaertsJT16}
\setcitation{XSB}{SwiW12}
\setcitation{KCmap}{jair/DarwicheM02}
\setcitation{TLA}{DBLP:books/aw/Lamport2002}
\setcitation{EventB}{BookAbrial2010}
\setcitation{MX}{MitchellT05}
\setcitation{MIP}{Sierksma96}
\setcitation{perefectmodel}{minker88/Przymusinski88}
\setcitation{PerfectModel}{minker88/Przymusinski88}
\setcitation{SafeInductions}{ijcai/BogaertsVD17}
\setcitation{AIC}{ppdp/FlescaGZ04}
\setcitation{aic}{ppdp/FlescaGZ04}
\setcitation{alpha}{lpnmr/Weinzierl17}
\setcitation{omiga}{jelia/Dao-TranEFWW12}
\setcitation{gasp}{fuin/PaluDPR09}
\setcitation{asperix}{lpnmr/LefevreN09a}
\setcitation{CTL}{lop/ClarkeE81}
\setcitation{AFT-AIC}{ai/BogaertsC18}
\setcitation{UltimateApproximator}{DeneckerMT04}
\setcitation{KripkeKleene}{Fitting85}
\setcitation{AFT-HO}{corr/CharalambidisRS18} 
\setcitation{HereThere}{Heyting30}
\setcitation{dAEL}{ijcai/HertumCBD16}
\setcitation{SDD}{ijcai/Darwiche11}
\setcitation{HEX}{ijcai/EiterIST05}
\setcitation{wADF}{aaai/BrewkaSWW18}
\setcitation{wADFfix}{corr/BrewkaSWW18}
\setcitation{TransitionSystems}{jacm/NieuwenhuisOT06}
\setcitation{galliwasp}{lopstr/MarpleG12}
\setcitation{GalliWasp}{lopstr/MarpleG12}
\setcitation{clingcon}{tplp/BanbaraKOS17}
\setcitation{lp2sat}{birthday/JanhunenN11}
\setcitation{lp2mip}{LIU12}
\setcitation{lp2diff}{lpnmr/JanhunenNS09}
\setcitation{lp2acyc}{ecai/GebserJR14}
\setcitation{PB}{faia/RousselM09}
\setcitation{pb}{faia/RousselM09}
\setcitation{CuttingPlanes}{dam/CookCT87}
\setcitation{RoundingSAT}{ijcai/ElffersN18}
\setcitation{PRS}{aaai/DixonG02}
\setcitation{sat4j}{jsat/BerreP10}
\setcitation{SAT4J}{jsat/BerreP10}
\setcitation{mingo}{LIU12}
\setcitation{pbmodels}{lpnmr/LiuT05}
\setcitation{HEF-LP}{lpnmr/GebserLL07}
\setcitation{HCF-LP}{amai/Ben-EliyahuD94}
\setcitation{aspcore2}{tplp/CalimeriFGIKKLM20}
\setcitation{AspCore2}{tplp/CalimeriFGIKKLM20}
\setcitation{}{}
\setcitation{}{}
\setcitation{}{}
\setcitation{}{}
\setcitation{}{}
\setcitation{}{}
\setcitation{}{}
\setcitation{}{}
\setcitation{}{}
\setcitation{}{}
\setcitation{}{}
\setcitation{}{}
\setcitation{}{}
\setcitation{}{}
\setcitation{}{}
\setcitation{}{}
\setcitation{}{}
\setcitation{}{}
\setcitation{}{}
\setcitation{}{}
\setcitation{}{}
\setcitation{}{}
  
\newcommand\mycite[1]{%
      \ifcsname mycommoncitation#1\endcsname%
   \cite{\getcitation{#1}}%
  \else%
    \cite{#1}%
  \fi%
}	
  
\newcommand\mycitet[1]{%
      \ifcsname mycommoncitation#1\endcsname%
   \citet{\getcitation{#1}}%
  \else%
    \citet{#1}
  \fi%
}

\usepackage{ifthen}

\provideboolean{detectedSTOC}
\provideboolean{detectedFOCS}
\provideboolean{detectedElsevier}
\provideboolean{detectedNOW}
\provideboolean{detectedLMCS}
\provideboolean{detectedIEEE}
\provideboolean{detectedPoster}
\provideboolean{detectedSIAM}
\provideboolean{detectedLNCS}
\provideboolean{detectedACM}
\provideboolean{detectedACMconf}
\provideboolean{detectedSigplanconf}
\provideboolean{detectedToC}
\provideboolean{detectedLIPIcs}
\provideboolean{detectedAAAI}
\provideboolean{detectedIJCAI}
\provideboolean{detectedCompCplx}
\provideboolean{detectedEasyChair}
\provideboolean{detectedArticle}
\provideboolean{detectedReport}
\provideboolean{detectedThesis}

\makeatletter

\@ifclassloaded{sig-alternate}
{\setboolean{detectedSTOC}{true}}
{\setboolean{detectedSTOC}{false}}

\@ifclassloaded{elsarticle}
{\setboolean{detectedElsevier}{true}}
{\setboolean{detectedElsevier}{false}}

\@ifclassloaded{now}
{\setboolean{detectedNOW}{true}}
{\setboolean{detectedNOW}{false}}

\@ifclassloaded{lmcs}
{\setboolean{detectedLMCS}{true}}
{\setboolean{detectedLMCS}{false}}

\@ifclassloaded{IEEEtran} {
  \setboolean{detectedIEEE}{true}
  \ifCLASSOPTIONconference {
    \setboolean{detectedFOCS}{true}
  }
  \else {
    \setboolean{detectedFOCS}{false}
  }
  \fi
}
{
  \setboolean{detectedFOCS}{false}
  \setboolean{detectedIEEE}{false}
}

\@ifclassloaded{siamart171218}
{\setboolean{detectedSIAM}{true}}
{\setboolean{detectedSIAM}{false}}

\@ifclassloaded{llncs}
{\setboolean{detectedLNCS}{true}}
{\setboolean{detectedLNCS}{false}}

\@ifclassloaded{acmsmall}
{\setboolean{detectedACM}{true}}
{\setboolean{detectedACM}{false}}

\@ifclassloaded{acmart}
{\setboolean{detectedACMconf}{true}}
{\setboolean{detectedACMconf}{false}}

\@ifclassloaded{sigplanconf}
{\setboolean{detectedSigplanconf}{true}}
{\setboolean{detectedSigplanconf}{false}}

\@ifclassloaded{toc}
{\setboolean{detectedToC}{true}}
{\setboolean{detectedToC}{false}}

\@ifclassloaded{lipics}
{\setboolean{detectedLIPIcs}{true}}
{\@ifclassloaded{lipics-v2019}
  {\setboolean{detectedLIPIcs}{true}}
  {\@ifclassloaded{oasics-v2019}
    {\setboolean{detectedLIPIcs}{true}}
    {\setboolean{detectedLIPIcs}{false}}
  }
}

\@ifclassloaded{cc}
{\setboolean{detectedCompCplx}{true}}
{\setboolean{detectedCompCplx}{false}}

\@ifclassloaded{easychair}
{\setboolean{detectedEasyChair}{true}}
{\setboolean{detectedEasyChair}{false}}

\@ifpackageloaded{aaai}
{\setboolean{detectedAAAI}{true}}        
{\@ifpackageloaded{aaai18}
  {\setboolean{detectedAAAI}{true}}
  {\@ifpackageloaded{aaai20}       
    {\setboolean{detectedAAAI}{true}}
    {\setboolean{detectedAAAI}{false}}}}

\@ifpackageloaded{ijcai18}
{\setboolean{detectedIJCAI}{true}}        
{\@ifpackageloaded{ijcai19}
  {\setboolean{detectedIJCAI}{true}}        
  {\setboolean{detectedIJCAI}{false}}}

\@ifclassloaded{sciposter}
{\setboolean{detectedPoster}{true}}
{\setboolean{detectedPoster}{false}}

\@ifclassloaded{article}
{\setboolean{detectedArticle}{true}}
{\setboolean{detectedArticle}{false}}

\makeatother

\ifthenelse{\not \isundefined{\examen} 
  \and \not \isundefined{\disputationsdatum} 
  \and \not \isundefined{\disputationslokal}}   
  {\setboolean{detectedThesis}{true}}
  {\setboolean{detectedThesis}{false}}

\ifthenelse{\boolean{detectedArticle} \or \boolean{detectedThesis}
  \or \boolean{detectedSTOC} \or \boolean{detectedFOCS}
  \or \boolean{detectedSIAM} \or \boolean{detectedIEEE}
  \or \boolean{detectedPoster}}
{\setboolean{detectedReport}{false}}
{\setboolean{detectedReport}{true}}

\usepackage{ifthen}
\usepackage{xspace}
\ifthenelse
{\boolean{detectedElsevier} \or \boolean{detectedSIAM} 
  \or \boolean{detectedLIPIcs}}
{}
{
}

\DeclareMathAlphabet{\mathsfsl}{OT1}{cmss}{m}{sl}

\ifthenelse{\boolean{detectedToC}}{}
{

}

\newcommand{\sumnodisplay}{{\textstyle \sum}}

\newcommand{\ceiling}[1]{\lceil #1 \rceil}

\newcommand{\MAXOFEXPR}[2][]{\max_{#1} \left\{ #2 \right\}}
\newcommand{\MINOFEXPR}[2][]{\min_{#1} \left\{ #2 \right\}}
\newcommand{\Maxofexpr}[2][]{\max_{#1} \bigl\{ #2 \bigr\}}
\newcommand{\Minofexpr}[2][]{\min_{#1} \bigl\{ #2 \bigr\}}

\newcommand{\MAXOFSET}[3][:]%
     {\ifthenelse{\equal{#1}{;}}%
     {\MAXOFEXPR{ #2 \,;\, #3 }}
     {\ifthenelse{\equal{#1}{:}}%
     {\MAXOFEXPR{ #2 \,:\, #3 }}
     {\max \twincommandJN{\left\{}{#2}{\left#1}{\right}{\,#3}{\right\}}}}}
\newcommand{\MINOFSET}[3][:]%
     {\ifthenelse{\equal{#1}{;}}%
     {\MINOFEXPR{ #2 \,;\, #3 }}
     {\ifthenelse{\equal{#1}{:}}%
     {\MINOFEXPR{ #2 \,:\, #3 }}
     {\min \twincommandJN{\left\{}{#2}{\left#1}{\right}{\,#3}{\right\}}}}}

\newcommand{\Maxofset}[3][:]%
     {\ifthenelse{\equal{#1}{;}}%
     {\Maxofexpr{ #2 \,;\, #3 }}
     {\ifthenelse{\equal{#1}{:}}%
     {\Maxofexpr{ #2 \,:\, #3 }}
     {\max \twincommandJN{\bigl\{}{#2}{\bigl#1}{\bigr}{\,#3}{\bigr\}}}}}
\newcommand{\Minofset}[3][:]%
     {\ifthenelse{\equal{#1}{;}}%
     {\Minofexpr{ #2 \,;\, #3 }}
     {\ifthenelse{\equal{#1}{:}}%
     {\Minofexpr{ #2 \,:\, #3 }}
     {\min \twincommandJN{\bigl\{}{#2}{\bigl#1}{\bigr}{\,#3}{\bigr\}}}}}

\DeclareMathOperator{\Expop}{E}

\ifthenelse{\boolean{detectedLMCS}}
{}
{}

\newcommand{\twincommandJN}[6]%
    {#1#2#3\vphantom{#2#5}\mspace{-2.05mu}#4.#5#6}

\newcommand{\CondExp}[2]%
    {\Expop\twincommandJN{\bigl[}{#1}{\bigl|}{\bigr}{\,#2}{\bigr]}}
\newcommand{\CONDEXP}[2]%
     {\Expop\twincommandJN{\left[}{#1}{\left|}{\right}{\,#2}{\right]}}

\newcommand{\Condprob}[3][]%
    {\Pr_{#1}\twincommandJN{\bigl[}{#2}{\bigl|}{\bigr}{\,#3}{\bigr]}}
\newcommand{\CONDPROB}[3][]%
    {\Pr_{#1}\twincommandJN{\left[}{#2}{\left|}{\right}{\,#3}{\right]}}

\newcommand{\vneighbour}[2][]{N_{#1}({#2})}

\providecommand{\set}[1]{\{ #1 \}}
\providecommand{\Set}[1]{\bigl\{ #1 \bigr\}}

\renewcommand{\set}[1]{\{ #1 \}}
\renewcommand{\Set}[1]{\bigl\{ #1 \bigr\}}

\newcommand{\setdescr}[3][\mid]{\set{ #2 #1 #3 }}
\newcommand{\Setdescr}[3][|]%
     {\ifthenelse{\equal{#1}{;}}%
     {\Set{ #2 \,;\, #3 }}
     {\ifthenelse{\equal{#1}{:}}%
     {\Set{ #2 \,:\, #3 }}
     {\twincommandJN{\bigl\{}{#2\,}{\bigl#1}{\bigr}{\,#3}{\bigr\}}}}}
\newcommand{\SETDESCR}[3][|]%
     {\twincommandJN{\left\{}{#2\,}{\left#1}{\right}{\,#3}{\right\}}}

\newcommand{\Setdescrbrackets}[3][|]%
     {\twincommandJN{\bigl[}{#2}{\bigl#1}{\bigr}{\,#3}{\bigr]}}
\newcommand{\SETDESCRBRACKETS}[3][|]%
     {\twincommandJN{\left[}{#2}{\left#1}{\right}{\,#3}{\right]}}

\newcommand{\setsize}[1]{\lvert#1\rvert}

\newcommand{\intersection}{\cap}

\newcommand{\union}{\cup}
\newcommand{\Union}{\bigcup}

\newcommand{\Land}{\bigwedge}

\providecommand{\lequiv}{\leftrightarrow}
\renewcommand{\lequiv}{\leftrightarrow}

\newcommand{\olnot}[1]{\overline{#1}}

\newcommand{\nvar}{n}

\newcommand{\nclause}{m}

\newcommand{\clwidth}{k}

\newcommand{\randkcnfnclwrepl}[3][\clwidth]%
        {\ensuremath{\mathcal{F}^{#2, #3}_{#1}}}
\newcommand{\randkcnfnclwreplstd}%
        {\randkcnfnclwrepl{\clwidth}{\nvar}{\nclause}}

\newcommand{\complclassformat}[1]%
        {\textrm{\upshape{\textsf{#1}}}\xspace}
\newcommand{\cocomplclass}[1]%
        {\textrm{\upshape{\textsf{co#1}}}\xspace}

\newcommand{\DTIMEadviceclass}[2]%
    {\ensuremath{\complclassformat{DTIME}\bigl(#1\bigr)/{#2}}}

\newcommand{\PCPalph}[5]%
    {\ensuremath{\complclassformat{PCP}_{{#1},{#2}}[{#3}, {#4}, {#5}]}}
\newcommand{\PCP}[4]%
    {\ensuremath{\complclassformat{PCP}_{{#1},{#2}}[{#3}, {#4}]}}

\newcommand{\introduceterm}[1]{{\emph{#1}}}

\newcommand{\eqperiod}{\enspace .}
\newcommand{\eqcomma}{\enspace ,}

\newcommand{\wrt}{with respect to\xspace}

\ifthenelse{\boolean{detectedToC}}{}
  { 
    }
\ifthenelse{\boolean{detectedToC}}{}{
\newcommand{\ie}{i.e.,\ }

}

\ifthenelse{\boolean{detectedLIPIcs} \or \boolean{detectedIJCAI}}
{\renewcommand{\st}{\errmessage{Please do not use st}}}
{\ifthenelse{\isundefined{\st}}
  {\newcommand{\st}{such that\xspace}}
  {}}      

\ifthenelse{\boolean{detectedIEEE}}{}{}

\newcommand{\refsec}[1]{Section~\ref{#1}}

\newcommand{\reftwoths}[2]{Theorems~\ref{#1} and~\ref{#2}}

\ifthenelse
{\isundefined{\refeq}}
{\newcommand{\refeq}[1]{\eqref{#1}}}
{\renewcommand{\refeq}[1]{\eqref{#1}}}

\newcommand\entails\models

\newcommand{\supplementarymaterial}%
{{\color{red}supplementary material 
  with the official   conference version}}
\newcommand{\supplementarymateriallater}%
{{\color{red}supplementary material}}

\newcommand{\vdomrel}{\succ_{G}}
\providecommand{\vneighbour}[1]{N(#1)}

\newcommand{\algnameformat}[1]{\mathsf{#1}}
\newcommand{\algsuffixformat}[1]{\mathrm{#1}}

\newcommand{\ltvert}[1]{\mathit{LT}(#1)}
\newcommand{\gtvert}[1]{\mathit{GT}(#1)}

\newcommand{\cliquesearch}{\algnameformat{MaxCliqueSearch}}
\newcommand{\cliquecurr}{C_{\algsuffixformat{curr}}}
\newcommand{\cliquebest}{C_{\algsuffixformat{best}}}
\newcommand{\cliquec}{C}
\newcommand{\vrem}{V_{\algsuffixformat{rem}}}
\newcommand{\vnew}{V'}
\newcommand{\erem}{E_{\algsuffixformat{rem}}}
\newcommand{\grem}{G_{\algsuffixformat{rem}}}
\newcommand{\colclass}{S}

\newcommand{\domrulelong}{dominance-based strength\-en\-ing\xspace}
\newcommand{\Domrulelong}{Dominance-based strength\-en\-ing\xspace}
\newcommand{\DOMRULELONG}{Dominance-Based Strengthening\xspace}
\newcommand{\domruleshort}{dominance\xspace}

\newcommand{\redrulelong}{redundance-based strength\-en\-ing\xspace}
\newcommand{\Redrulelong}{Redundance-based strength\-en\-ing\xspace}
\newcommand{\REDRULELONG}{Redundance-Based Strengthening\xspace}
\newcommand{\redruleshort}{redundance\xspace}

\newcommand{\uncheckeddeletion}{unchecked deletion\xspace}
\newcommand{\Uncheckeddeletion}{Unchecked deletion\xspace}
\newcommand{\checkeddeletion}{checked deletion\xspace}

\newcommand{\syntimplication}{literal-axiom implication\xspace}

\newcommand{\syntimplies}{literal-axiom-implies\xspace}
\newcommand{\syntimplied}{literal-axiom-implied\xspace}

\newcommand{\fminimal}{\mbox{$\objfunc$-minimal}\xspace}
\newcommand{\anfminimal}{an \fminimal}

\newcommand{\subst}[2]{{{#1}_{\upharpoonright#2}}}
\newcommand{\restrict}[2]{\subst{#1}{#2}}

\newcommand{\baseset}{\m{\mathcal{F}}}

\newcommand{\redundantset}{\m{\mathcal{R}}}

\newcommand{\orderSymbol}{\m{\preceq}}
\newcommand{\better}{\m{\orderSymbol_{\objfunc}}}
\newcommand{\strictlybetter}{\m{\prec_{\objfunc}}}
\newcommand{\orderConstraint}{\m{\mathcal{O}_{\orderSymbol{}}}}
\newcommand{\order}[2]{\m{\orderConstraint(#1,#2)}}

\newcommand{\verifierStateInitial}[1]{(\formf{}, \emptyset{}, \bot{}, \emptyset{})}

\newcommand{\vecu}{\vec{u}}
\newcommand{\vecv}{\vec{v}}
\newcommand{\vecw}{\vec{w}}
\newcommand{\vecx}{\vec{x}}
\newcommand{\vecy}{\vec{y}}
\newcommand{\vecz}{\vec{z}}

\newcommand{\vecend}{n}

\newcommand{\olnotwindex}[2]{\olnot{#1}_{#2}}

\newcommand{\domain}[0]{\textit{dom}}

\usepackage{xspace}

\newcommand{\litell}{\ell}
\newcommand{\varu}{u}
\newcommand{\varv}{v}
\newcommand{\varw}{w}
\newcommand{\varx}{x}
\newcommand{\vary}{y}
\newcommand{\varz}{z}
\newcommand{\constrc}{C}
\newcommand{\constrd}{D}
\newcommand{\coeffa}{a}
\newcommand{\coeffb}{b}
\newcommand{\coeffw}{w}
\newcommand{\consta}{A}
\newcommand{\dega}{\consta}
\newcommand{\divd}{d}
\newcommand{\formf}{F}
\newcommand{\formg}{G}
\newcommand{\assmntrho}{\rho}
\newcommand{\assmntalpha}{\alpha}
\newcommand{\assmntbeta}{\beta}

\newcommand{\witness}{\omega}

\newcommand{\mapassmnt}{\witness}
\newcommand{\boolval}{b}

\newcommand{\negc}[1]{\neg#1}
\newcommand{\negcwp}[1]{\neg(#1)}

\newcommand{\proofsystemformat}[1]{\textit{#1}\@}
\newcommand{\toolformat}[1]{\textit{#1}}

\newcommand{\rat}{\proofsystemformat{RAT}\xspace}
\newcommand{\drat}{\proofsystemformat{DRAT}\xspace}
\newcommand{\lrat}{\proofsystemformat{LRAT}\xspace}
\newcommand{\grit}{\proofsystemformat{GRIT}\xspace}

\newcommand{\tracecheck}{\proofsystemformat{TraceCheck}\xspace}
\newcommand{\rupname}{\proofsystemformat{RUP}\xspace}
\newcommand{\dsrup}{\proofsystemformat{DSRUP}\xspace}

\newcommand{\veripb}{\toolformat{VeriPB}\xspace}

\newcommand{\kissat}{\toolformat{Kissat}\xspace}

\newcommand\lex{\m{\prec_{\mathit{lex}}}}

\newcommand{\objective}{\m{f_o}}

\newcommand{\implicationalderivation}{implicational derivation\xspace}
\newcommand{\Implicationalderivation}{Implicational derivation\xspace}
\newcommand{\IMPLICATIONALDERIVATION}{Implicational Derivation\xspace}

\providecommand{\synteq}{\,\doteq\,}
\renewcommand{\synteq}{\,\doteq\,}
\newcommand{\synteqnospace}{\doteq}

\newcommand{\cpderives}{\vdash}

\newcommand{\objfunc}{\m{f}}
\renewcommand{\objective}{\objfunc}
\newcommand{\objval}{\m{v}}
\newcommand{\objvalprime}{\m{v'}}
\newcommand{\bestobjval}{v^*}
\newcommand{\objsol}{\alpha}
\newcommand{\evalobj}[2]{\subst{#1}{#2}}

\newcommand{\objwitnesscondition}{\subst{\objfunc}{\witness}\, \leq \objfunc}
\newcommand{\objwitnessconditionset}{\set{\objwitnesscondition}}

\renewcommand{\subst}[2]{#1\!\!\upharpoonright_{#2}}
\renewcommand{\subst}[2]{{#1\!\!\upharpoonright}_{#2}}

\newcommand{\compassmt}[2]{{#1} \compassmtop {#2}}
\renewcommand{\compassmt}[2]{{#2} \circ {#1}}
\newcommand{\compassmtafterassmt}[2]{{#1} \circ {#2}}

\usepackage{mathrsfs}

\newcommand{\proofsetformat}[1]{\mathcal{#1}}
\renewcommand{\proofsetformat}[1]{\mathfrak{#1}}
\renewcommand{\proofsetformat}[1]{\mathbb{#1}}
\renewcommand{\proofsetformat}[1]{\mathscr{#1}}

\newcommand{\coreset}{\proofsetformat{C}}
\renewcommand{\baseset}{\coreset}
\newcommand{\coreconstraint}{core constraint\xspace}

\newcommand{\derivedset}{\proofsetformat{D}}
\renewcommand{\redundantset}{\derivedset}
\newcommand{\derivedconstraint}{derived constraint\xspace}

\newcommand{\assmntorder}{\preceq}

\newcommand{\lexorder}{\preceq_{\mathrm{lex}}}
\renewcommand{\lex}{\lexorder}

\newcommand{\plexorder}{\proofsetformat{O}_{\!\lexorder}}
\newcommand{\plexorderform}[2]{\plexorder({#1},{#2})}
\newcommand{\plexorderformstd}{\plexorderform{\vecu}{\vecv}}

\newcommand{\ovars}{\vecz}
\newcommand{\pordernot}{\proofsetformat{O}_{\!\assmntorder}}
\newcommand{\pordernotone}{\proofsetformat{O}_{\!\assmntorder^1}}
\newcommand{\pordernottwo}{\proofsetformat{O}_{\!\assmntorder^2}}
\newcommand{\pordernotprime}{\proofsetformat{O}_{\!\assmntorder}'}
\newcommand{\porderform}[2]{\pordernot({#1},{#2})}

\newcommand{\porderformsub}{\porderform{\subst{\ovars}{\mapassmnt}}{\ovars}}
\newcommand{\porderformdom}{\porderform{\subst{\ovars}{\mapassmnt}}{\ovars}}
\newcommand{\porderformdomneg}{\porderform{\ovars}{\subst{\ovars}{\mapassmnt}}}

\newcommand{\porderbot}{\proofsetformat{O}_{\!\bot}}
\newcommand{\pordertop}{\proofsetformat{O}_{\!\top}}

\newcommand{\betterobjval}{\objfunc \leq \objval - 1}
\newcommand{\betterobjvalprime}{\objfunc \leq \objval' - 1}
\newcommand{\betterobjvalset}{\set{\betterobjval}}

\newcommand{\coreunionderived}{\coreset \union \derivedset}

\newcommand{\coreunionobj}{\coreset \union \betterobjvalset}
\newcommand{\coreunionderivedunionobj}%
    {\coreset \union \derivedset \union \betterobjvalset}
\newcommand{\coreunionderivedunionobjnobreak}%
    {\coreset \union \derivedset \union \mbox{$\betterobjvalset$}}
\newcommand{\coreunionderivedunionobjprime}%
    {\coreset \union \derivedset \union \betterobjvalprime}

\newcommand{\coreunionderivedprimeunionobj}%
    {\coreset' \union \derivedset' \union \betterobjvalset}
\newcommand{\coreunionderivedprimeunionobjprime}%
    {\coreset' \union \derivedset' \union \betterobjvalprime}

\newcommand{\pconfstd}%
{\m{(\coreset, \derivedset, \pordernot, \ovars,  \objval)}}
\newcommand{\pconfpreorderchange}%
{(\coreset, \emptyset, \pordernotone, \ovars_1,  \objval)}
\newcommand{\pconfpreorderchangealt}%
{(\coreset, \derivedset, \pordernotone, \ovars_1,  \objval)}
\newcommand{\pconfpostorderchange}%
{(\coreset, \emptyset, \pordernottwo, \ovars_2,  \objval)}
\newcommand{\pconfcdprime}%
{(\coreset', \derivedset', \pordernot, \ovars,  \objval)}
\newcommand{\pconfcprime}%
{(\coreset', \derivedset, \pordernot, \ovars,  \objval)}
\newcommand{\pconfprime}%
{(\coreset', \derivedset', \pordernotprime, \ovars',  \objval')}
\newcommand{\pconforderbotstd}%
{(\coreset, \derivedset, \porderbot, \ovars,  \objval)}
\newcommand{\pconfordertopstd}%
{(\coreset, \derivedset, \pordertop, \ovars,  \objval)}
\newcommand{\pconfaddcore}[1]%
{(\coreset \union {#1}, \derivedset, \pordernot, \ovars,  \objval)}
\newcommand{\pconfaddderived}[1]%
{(\coreset, \derivedset \union {#1}, \pordernot, \ovars,  \objval)}
\newcommand{\pconfaddderivednobreak}[1]%
{(\coreset, \mbox{$\derivedset \union {#1}$}, \pordernot, \ovars,  \objval)}
\newcommand{\pconfneworder}[2]%
{(\coreset, \derivedset, {#1}, {#2},  \objval)}
\newcommand{\pconfinit}%
{(\formf, \emptyset, \pordertop, \emptyset,  \infty)}
\newcommand{\pconfdelete}%
{(\coreset', \emptyset, \pordernot, \ovars,  \objval)}

\newcommand{\pconffinal}%
{(\coreset, \derivedset, \pordernot, \ovars,  \bestobjval)}

\newcommand{\pconfboundupdate}%
{(\coreset, \derivedset, \pordernot, \ovars,  \objval')}

\newtheorem{theorem}{Theorem}
\newtheorem{definition}[theorem]{Definition}
\newtheorem{example}[theorem]{Example}
\newtheorem{proposition}[theorem]{Proposition}

\newtheorem{fact}[theorem]{Fact}
\newtheorem{remark}[theorem]{Remark}
\newcommand{\subderivpi}{\pi}
\newcommand{\symsigma}{\sigma}

\newcommand\perm{\m{\pi}}
\newcommand\assignment{\assmntalpha}

\newcommand\generatorset{\m{G}}

\newcommand\LLconstr{\m{\psi_{LL}}}
\newcommand\LLc{\m{\constrc_{LL}}}
\renewcommand\LLconstr{\m{\psi_{\mathit{LL}}}}
\renewcommand\LLc{\m{\constrc_{\mathit{LL}}}}

\newcommand\breakid{\textsc{BreakID}\xspace}
\renewcommand{\breakid}{\toolformat{BreakID}\xspace}

\newcommand\safe{$(\formf,\objfunc)$-safe\xspace}
\newcommand\safety{$(\formf,\objfunc)$-safety\xspace}

\renewcommand\safe{$(\formf,\objfunc)$-valid\xspace}

\newcommand\weaksafe{weakly $(\formf,\objfunc)$-valid\xspace}
\renewcommand\safety{$(\formf,\objfunc)$-validity\xspace}
\newcommand\weaksafety{weak $(\formf,\objfunc)$-validity\xspace}

\provideboolean{briefconfversion}
\provideboolean{techappendixversion}

\newcommand\shortpaper[1]{%
\ifthenelse{\boolean{briefconfversion}}%
{#1}%
{\ignorespaces}}
\newcommand\longpaper[1]{%
\ifthenelse{\boolean{briefconfversion}}%
{\ignorespaces}%
{{\color{NavyBlue} #1}}} 

\newcommand{\textrefzenodo}%
{All code for our implementations and experiments,
  as well as data and scripts for all plots,
  can be found at
  \url{https://doi.org/10.5281/zenodo.6373986}.\xspace}

\newcommand\amodelimproving{an objective-improving\xspace}
\newcommand\modelimproving{objective-improving\xspace}

\usepackage{numname}
\newcounter{authorcount}
\newcommand{\newauthor}[3]{%
\newcounter{#1comment}
\expandafter\newcommand\csname #1comment\endcsname[1]{%
\ifdefined\DONOTINSERTCOMMENTS\relax\else%
\par\noindent
{
\scshape \footnotesize #2's comment
\stepcounter{#1comment}\csname the#1comment\endcsname}:
{\sffamily \itshape
\textcolor{#3}{##1}\par}
\fi}
\stepcounter{authorcount}
\expandafter\edef\csname #1ordinal\endcsname{\theauthorcount}
\expandafter\newcommand\csname theauthor#1\endcsname{%
the \ordinaltoname{\csname #1ordinal\endcsname} author\xspace}
\expandafter\newcommand\csname Theauthor#1\endcsname{%
The \ordinaltoname{\csname #1ordinal\endcsname} author\xspace}
}

\definecolor{airforceblue}{rgb}{0.36, 0.54, 0.66}
\definecolor{amethyst}{rgb}{0.6, 0.4, 0.8}
\definecolor{asparagus}{rgb}{0.53, 0.66, 0.42}
\definecolor{brass}{rgb}{0.71, 0.65, 0.26}
\definecolor{brown}{rgb}{0.59, 0.29, 0.0}
\definecolor{darkolivegreen}{rgb}{0.33, 0.42, 0.18}
\definecolor{darkorange}{rgb}{1.0, 0.55, 0.0}
\definecolor{darkcyan}{rgb}{0, 0.5, 0.5}

\newauthor{bb}{Bart}{darkolivegreen}
\newauthor{sg}{Stephan}{darkcyan}
\newauthor{cm}{Ciaran}{darkorange}
\newauthor{jn}{Jakob}{blue}

\renewcommand\eqperiod{\,.}
\renewcommand\eqcomma{\,,}

\setboolean{briefconfversion}{false}
\setboolean{techappendixversion}{true}
\usepackage{tikz} 

\newcommand{\citet}[1]{\citeA{#1}\xspace}

\usepackage{listings}
\usepackage[linesnumbered,ruled,vlined]{algorithm2e}

\begin{document}

\title{Certified Dominance and Symmetry Breaking for Combinatorial Optimisation}

\author{\name  Bart Bogaerts \email bart.bogaerts@vub.be \\
       \addr Vrije Universiteit Brussel, Brussels, Belgium
       \AND
       \name Stephan Gocht \email stephan.gocht@cs.lth.se \\
       \addr Lund University, Lund, Sweden\\
       University of Copenhagen, Copenhagen, Denmark
       \AND
       \name Ciaran McCreesh \email Ciaran.McCreesh@glasgow.ac.uk \\
       \addr University of Glasgow, Glasgow, UK
     		\AND
     	\name Jakob Nordstr\"om \email jn@di.ku.dk \\ 
      \addr        University of Copenhagen, Copenhagen, Denmark\\
       Lund University, Lund, Sweden
    }


\maketitle

\begin{abstract}

Symmetry and dominance breaking can be crucial for solving hard
combinatorial search and optimisation problems,
but the correctness of these techniques sometimes relies on subtle
arguments.  For this reason, it is desirable to produce efficient,
machine-verifiable certificates that solutions have been computed
correctly.
Building on the cutting planes proof system, we develop a
certification method  for optimisation problems 
in which symmetry and dominance breaking 
is
easily expressible.
Our experimental evaluation demonstrates that we can efficiently verify 
fully general symmetry breaking in Boolean satisfiability (SAT) solving, 
thus providing, for the first time, a unified method to certify a
range of advanced SAT techniques that also includes 
cardinality and parity (XOR) reasoning.
In addition, we apply our method to maximum clique solving and
constraint programming as a proof of concept that the approach applies
to a wider range of combinatorial problems.

\end{abstract}


\providecommand{\statusjn}{\textbf{Status of next section (Jakob's assessment/guess):}\xspace}

%
%

\section{Introduction}
\label{sec:introduction}

Symmetries pose a challenge when solving hard combinatorial problems.
As an illustration of this, 
consider the Crystal Maze
puzzle\footnote{\url{https://theconversation.com/what-problems-will-ai-solve-in-future-an-old-british-gameshow-can-help-explain-49080}}
shown in \cref{crystalmaze}, which is often used in introductory constraint modelling courses. A
human modeller might notice that the puzzle is the same
after flipping vertically, 
and could introduce the constraint $A < G$ to eliminate
this symmetry.
Or, they may notice 
that flipping horizontally induces a symmetry, 
which could be broken with $A <
B$.
Alternatively, they might spot that the \emph{values} are symmetrical, and
that we can interchange $1$ and $8$, $2$ and $7$, and so on; this can be
eliminated by saying that $A \le 4$. In each case a constraint is being added that preserves
satisfiability overall, but that restricts a solver to finding (ideally) just one witness from each
equivalence class of solutions---the hope is that this will improve solver performance.
However, although we may be reasonably sure that any of these three constraints is correct
individually, are combinations of these constraints valid simultaneously? What if we had said $F
< C$ instead
\mbox{of $A < B$?}  
And what if we could use numbers more than once? 
Getting symmetry elimination constraints right can be error-prone even
for experienced modellers.  And when dealing with larger problems with
many constraints and interacting symmetries it can be hard to tell
whether a problem instance is genuinely unsatisfiable, or was made so by an
incorrectly added symmetry breaking constraint.

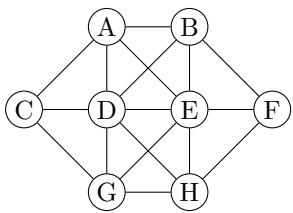
\begin{figure}[t] 
  \begin{center}\begin{tikzpicture}%
      [scale=1.1]
        \node [draw, circle, inner sep=1pt] (A) at (0, 0) { \phantom{X} }; \node at (A) { \small A };
        \node [draw, circle, inner sep=1pt] (B) at (1, 0) { \phantom{X} }; \node at (B) { \small B };
        \node [draw, circle, inner sep=1pt] (C) at (-1, -1) { \phantom{X} }; \node at (C) { \small C };
        \node [draw, circle, inner sep=1pt] (D) at (0, -1) { \phantom{X} }; \node at (D) { \small D };
        \node [draw, circle, inner sep=1pt] (E) at (1, -1) { \phantom{X} }; \node at (E) { \small E };
        \node [draw, circle, inner sep=1pt] (F) at (2, -1) { \phantom{X} }; \node at (F) { \small F };
        \node [draw, circle, inner sep=1pt] (G) at (0, -2) { \phantom{X} }; \node at (G) { \small G };
        \node [draw, circle, inner sep=1pt] (H) at (1, -2) { \phantom{X} }; \node at (H) { \small H };
        \draw (A) -- (B);
        \draw (A) -- (C);
        \draw (A) -- (D);
        \draw (A) -- (E);
        \draw (B) -- (D);
        \draw (B) -- (E);
        \draw (B) -- (F);
        \draw (C) -- (D);
        \draw (C) -- (G);
        \draw (D) -- (E);
        \draw (D) -- (G);
        \draw (D) -- (H);
        \draw (E) -- (F);
        \draw (E) -- (G);
        \draw (E) -- (H);
        \draw (F) -- (H);
        \draw (G) -- (H);
    \end{tikzpicture}\end{center}
    \caption{The Crystal Maze puzzle. Place numbers 1 to 8 in the
    circles, with every circle getting a different number, 
    so that adjacent circles do not
    have consecutive numbers.}\label{crystalmaze}
\end{figure}

Despite these difficulties, symmetry elimination using both manual and automatic techniques has been
key to many successes across modern combinatorial optimisation paradigms such as constraint
programming (CP) \cite{GSVW14Future}, Boolean satisfiability (SAT)~%
solving~\cite{Sakallah21SymmetrySatisfiabilityPlusCrossref},
and mixed-integer programming (MIP) \cite{AW13MIP}. 
As these
optimisation technologies are increasingly being used for high-value and life-affecting
decision-making processes, it becomes vital that we can trust their
outputs---and unfortunately, current solvers do not always produce
correct answers 
\cite{BLB10AutomatedTesting,CKSW13Hybrid,AGJMN18Metamorphic,GSD19SolverCheck,BMN22CPtutorial}.


The most promising way to address 
this problem 
of correctness 
appears to be to use
\emph{certification},
or \emph{proof logging},
where a solver
must produce an efficiently machine-verifiable certificate that the 
answer to the problem was computed correctly
\cite{ABMRS11IntroCertifyingAlgo,MMNS11CertifyingAlgorithms}.
This approach has been successfully used in the
SAT community, 
which has developed
numerous proof logging formats such as
\rupname~\cite{GN03Verification},
\tracecheck~\cite{Biere06TraceCheck},
\drat~\cite{HHW13Trimming,HHW13Verifying,WHH14DRAT},
\grit~\cite{CMS17EfficientCertified}, and
 \lrat~\cite{CHHKS17EfficientCertified}.
However, currently used methods 
work only for decision problems, and do not support
the full range of SAT solving techniques, let alone CP and MIP
solving.
As a case in point,
there is no efficient proof logging for symmetry breaking, except for
limited cases with small symmetries which can interact only in simple ways
\cite{HHW15SymmetryBreaking}. 
\citet{TD20CertifyingSymmetry} recently proposed a proof logging method
\dsrup for symmetric learning of variants of derived clauses,
but this format does not support symmetry breaking (in the sense just
discussed) and is also inherently unable to support pre- and
inprocessing techniques, which are crucial in
state-of-the-art SAT solvers.


\subsection{Our Contribution}

In this work, we develop a proof logging method 
for \emph{optimisation} problems---i.e., problems where we are given
both a formula~$\formf$ and an objective
function~$\objfunc$ to be optimised by some assignment
satisfying~$\formf$---that  can deal with 
\emph{dominance}, a generalization of symmetry.
Dominance breaking starts from the observation that we can strengthen
$\formf$ by imposing 
an additional constraint~$\constrc$
if every solution of $\formf$ that does not satisfy~$\constrc$ is dominated by another solution of $\formf$. 
This technique 
is used in many fields of combinatorial optimisation 
\cite{%
Walsh06GeneralSymmetry,Walsh12Symmetry,GPG06Symmetry,MP16Finding
,DBLP:journals/eor/JougletC11
,GKS11ComplexOptimization
,BSU18Branch-and-price
,DBLP:journals/transci/HoogeboomDLV20
,BP97ConstraintPropagation,DH02ProjectSceduling
}.
And even if we are given just a decision problem for a
formula~$\formf$, we can still use this framework 
by inventing an objective  function minimizing the lexicographic
order of assignments, and then do symmetry breaking by adding
dominance constraints with respect to this order.

The core idea 
to make our method produce efficiently verifiable proofs
is to have it present an 
\emph{explicit construction} of
a
dominating solution, so that a
verifier can check that this construction strictly improves the
objective value and preserves satisfaction of~$\formf$.
This constructed solution might itself be dominated, and hence not
satisfy~$\constrc$, but since the objective 
value   
decreases with every
application of the construction,
we can be sure, without performing the repeated process, 
that it would be guaranteed to eventually terminate.  
Importantly, 
verifying the correctness of adding such a constraint~$\constrc$
does not require the construction of an actual assignment satisfying~$\constrc$,
and can 
be performed efficiently even when
multiple constraints are to be added.
This resolves a practical issue with earlier approaches like that of
\citet{HHW15SymmetryBreaking}.
Such approaches
struggle with large or overlapping symmetries,
because adding a new symmetry breaking constraint might make it
necessary to re-break previously added constraints, and to understand
precisely how the different symmetries interact in order to be able to
do so. With our method, we never need to revisit previously broken
symmetries.


In addition,
although we do not develop this direction in the current paper,
it should   also
be noted that a quite intriguing feature of our method
is that we could
break symmetries of
a problem with respect to a
higher-order representation in pseudo-Boolean form
in a provably correct way,
even if the
encoding in the  conjunctive normal form (CNF) used by SAT solvers is
not symmetric.
%
If the solver is given a problem in pseudo-Boolean representation,
then it could add symmetry breaking constraints based on this
pseudo-Boolean representation, then translate the PB representation to
CNF in a certified way~\cite{GMNO22CertifiedCNFencodingPB}, and
finally produce a proof that the combination of all of these
constraints is a valid encoding of the original problem.
Following preliminaries in \cref{sec:prelims}, we describe 
in full detail   in \cref{sec:proof-system}
our dominance-based method of reasoning and prove that it is sound.

We have developed a proof format and  verifier
for this proof logging method    
by extending 
the tool  
\veripb~\cite{EGMN20Justifying,GN21CertifyingParity,GMN20SubgraphIso,GMMNPT20CertifyingSolvers}
with dominance reasoning.
The pseudo-Boolean  constraints and cutting planes proof
system \cite{CCT87ComplexityCP} used by \veripb 
turn out to be  quite 
convenient to
express and reason with dominance
constraints,
and moreover also make
it possible to certify 
cardinality and parity (XOR) reasoning
\cite{GN21CertifyingParity}, two other advanced 
SAT solving   
techniques which previous 
proof logging methods have not been able to support
efficiently.

After introducing the proof system
that we have developed as the foundation of our proof logging method,
we exhibit three applications that have
not previously admitted efficient
certification, and demonstrate that our  new method
can support simple, practical  proof logging in each case.%
\footnote{All code for our implementations and experiments,
as well as data and scripts for all plots,
can be found in the repository~\cite{BGMN22DominanceData}.}
First, in \cref{sec:symmetrybreaking}, we
demonstrate
that, by
enhancing the \breakid tool for SAT
solving~\cite{DBBD16ImprovedStatic} with \veripb proof logging, 
we can cover the entire solving toolchain when symmetries are
involved. 
We
show in full generality, and for the
first time, that proof logging is practical by running experiments on
SAT competition benchmarks.
Second, in \cref{sec:cpProblems}, we revisit the Crystal Maze example
and
describe
a tool that provides proof logging 
for the kind of symmetry breaking constraints discussed above.
Third, in \cref{sec:max-clique}, we discuss 
how adding a rule for deriving dominance breaking constraints
can be used  to support vertex domination reasoning in a 
maximum clique solver.
We conclude 
the paper
proper  
with a brief discussion of future research directions
in \cref{sec:conclusion}.
\cref{sec:sym-appendix}
contains a worked-out example of proof logging for symmetry breaking.

\subsection{Publication History}
An extended abstract of this paper was presented at the 
\emph{36th AAAI Conference on Artificial Intelligence}~\cite{BGMN22Dominance}.
The current manuscript extends the previous work with full proofs of
all formal claims, and includes a more detailed exposition of the
proof system and
proof logging  
applications.


\section{Preliminaries}
\label{sec:prelims}

Let us start with a brief review of 
some standard material, referring the reader  to, 
e.g.,~\citet{BN21ProofCplxSATplusCrossref}   for more details.
A \emph{literal}~$\litell$ over a Boolean variable~$\varx$
is $\varx$ itself 
or its negation
\mbox{$\olnot{\varx} = 1 - \varx$},
where variables take values~$0$ (false) or $1$~(true).
A \emph{pseudo-Boolean (PB) constraint}
is a \mbox{$0$--$1$} linear inequality
\begin{equation}
  \label{eq:normalized}
  \constrc
  \synteq
  \sumnodisplay_i \coeffa_i \litell_i \geq \dega
  \eqcomma
\end{equation}
where $\coeffa_i$ and~$\dega$ are integers
(and $\synteqnospace$ denotes syntactic equality). 
We can assume without
loss of generality that
pseudo-Boolean
constraints are 
\emph{normalized};
i.e., that all literals $\litell_i$ are over distinct variables and that the
\emph{coefficients}~$\coeffa_i$ and the \emph{degree (of falsity)}~$\dega$
are non-negative, 
but most of the time we will not need this. 
Instead, we will write PB constraints  in more relaxed form as
$
\sum_i \coeffa_i \litell_i 
\geq 
\dega + \sum_j \coeffb_j \litell_j 
$
or
$
\sum_i \coeffa_i \litell_i 
\leq 
\dega + \sum_j \coeffb_j \litell_j 
$
when convenient,
or even use equality
$
\sum_i \coeffa_i \litell_i = \dega
$
as syntactic sugar for the pair of inequalities
$
\sum_i \coeffa_i \litell_i \geq \dega 
$
and
$
\sum_i -\coeffa_i \litell_i \geq -\dega
$,
assuming that all constraints are implicitly 
normalized 
if needed.

The \emph{negation}
$\negc{\constrc}$ of the constraint~$\constrc$ in 
\refeq{eq:normalized} is 
(the normalized form of)   
\begin{equation}
  \label{eq:negation}
  \negc{\constrc}
  \synteq
  \sumnodisplay_i -\coeffa_i \litell_i \geq -\dega + 1
  \eqperiod
\end{equation}
A \emph{pseudo-Boolean formula} is a conjunction
$\formf \synteq \Land_j \constrc_j$
of PB constraints,
which we can also think of as the set
$\Union_j \set{\constrc_j}$ of constraints in the formula,
choosing whichever viewpoint seems most convenient.
Note that a \emph{(disjunctive) clause}
$\litell_1 \lor \cdots \lor \litell_k$
is equivalent to the 
pseudo-Boolean  
constraint
\mbox{$\litell_1 + \cdots + \litell_k \geq 1$}, so formulas in
\emph{conjunctive normal form (CNF)} are special cases of PB formulas.

A \emph{(partial) assignment} is a (partial) function from variables
to $\set{0,1}$. A partial assignment is also referred to as a
\emph{restriction}. A  \emph{substitution}  (or \emph{affine restriction})
can also map variables to literals.
We extend an assignment or substitution~$\assmntrho{}$ from variables
to literals in the natural way by respecting the meaning of negation,
and  for literals~$\litell$ over variables~$\varx$ not in the domain
of~$\assmntrho$,  
denoted $\varx{} \not \in \domain(\assmntrho{})$,
we use the convention $\assmntrho(\litell) = \litell$.
(That is, we can consider all assignments and substitution to be
total, but to be the identity outside of their specified domains.
Strictly speaking, we also require that all substitutions be defined
on the truth constants~$\set{0,1}$ and be the identity on these
constants.)
We sometimes write $\varx \mapsto \boolval$ when
$\assmntrho(\varx) = \boolval$, for $\boolval$ 
a literal or 
truth value,
to specify parts of~$\assmntrho$
or when~$\assmntrho$ is clear from context. 

%
%

%

We write $\compassmt{\witness}{\assmntrho}$
to denote the composed substitution resulting from applying
first~$\witness$ and then~$\assmntrho$, i.e.,
$\compassmt{\witness}{\assmntrho}(\varx) =
\assmntrho(\witness(\varx))$.
As an example, for
$
\witness =
\set{
\varx_1 \mapsto 0,
\varx_3 \mapsto \olnot{\varx}_4,
\varx_4 \mapsto \varx_3
}
$
and
$
\assmntrho = 
\set{
\varx_1 \mapsto 1,
\varx_2 \mapsto 1,
\varx_3 \mapsto 0,
\varx_4 \mapsto 0
}
$
we have
$
\compassmt{\witness}{\assmntrho}
=
\set{
\varx_1 \mapsto 0,
\varx_2 \mapsto 1,
\varx_3 \mapsto 1,
\varx_4 \mapsto 0
}
$.
Applying $\witness$ to a constraint~$\constrc$ as
in~\refeq{eq:normalized}
yields 
the \emph{restricted} constraint
\begin{equation}
  \label{eq:restricted-constraint}
  \restrict{\constrc}{\witness} \synteq
  \sumnodisplay_{i} \coeffa_i \witness(\litell_i) 
  \geq
  \dega
\eqcomma
\end{equation}
substituting literals or values as specified by~$\witness$.
For a formula~$\formf$ we define
$
\restrict{\formf}{\witness}
\synteq
\Land_j  \restrict{\constrc_j}{\witness}
$.

Since we will sometimes have to make fairly elaborate use of
substitutions, let us discuss some further notational conventions. 
If $\formf$ is a formula over variables
$\vecx = \set{\varx_1, \ldots, \varx_m}$,
we can write
$\formf(\vecx)$ when we want to stress the set of variables over
which~$\formf$ is defined.
For a substitution $\witness$ with domain (contained in)~$\vecx$,
the notation
$\formf\bigl( \restrict{\vecx}{\witness} \bigr)$
is understood to be
a synonym of
$\restrict{\formf}{\witness}$.
For the same formula~$\formf$ and
$\vecy = \set{\vary_1, \ldots, \vary_m}$,
the notation
$\formf(\vecy)$
is syntactic sugar for
$\restrict{\formf}{\witness}$
with $\witness$ denoting the substitution (implicitly) defined by
\mbox{$\witness(\varx_i) = \vary_i$} for 
$i = 1, \ldots, n$.    
Finally, for a formula
$\formg = \formg(\vecx,\vecy)$
over $\vecx \union \vecy$
and substitutions
$\assmntalpha$ and $\assmntbeta$  defined on
$\vecz = \set{\varz_1, \ldots, \varz_n}$
(either of which could be the identity), 
the notation
$\formg(\restrict{\vecz}{\assmntalpha}, \restrict{\vecz}{\assmntbeta})$
should be understood  as
$\restrict{\formg}{\witness}$ for $\witness$
defined by
$\witness(\varx_i) = \assmntalpha(\varz_i)$
and
$\witness(\vary_i) = \assmntbeta(\varz_i)$
for
$i = 1, \ldots, n$.    

The (normalized) constraint~$\constrc$
in~\refeq{eq:normalized}
is \emph{satisfied} by~$\assmntrho$
if
$\sum_{\assmntrho(\litell_i) = 1} \coeffa_i \geq \dega$.
%
%
A
pseudo-Boolean
formula~$\formf$ is satisfied by~$\assmntrho$ if all
constraints in it are, 
in which case it is \emph{satisfiable}.
If there is no satisfying assignment, 
$\formf$~is
\emph{unsatisfiable}. Two formulas are
\emph{equisatisfiable} if they are both
satisfiable or both unsatisfiable.
In this paper, we 
also consider optimisation problems, where in addition
to~$\formf$ we are given an integer linear objective function
\mbox{$
\objective
\synteq
\sum_i \coeffw_i \litell_i
$}
and
the task is to find an assignment that satisfies~$\formf$
and minimizes~$\objective$. (To deal with maximization problems we can
just negate the objective function.)

\emph{Cutting planes}~\cite{CCT87ComplexityCP}
is a method for iteratively deriving constraints~$\constrc$
from a pseudo-Boolean formula~$\formf$.
We write
$\formf \cpderives \constrc$
for any constraint~$\constrc$ derivable as follows.
Any \emph{axiom constraint}
$\constrc \in \formf$
is trivially derivable,
as is any \emph{literal axiom}
$\litell \geq 0$.
If
$\formf \cpderives \constrc$ and
$\formf \cpderives \constrd$,
then any positive integer
\emph{linear combination} of 
$\constrc$ and $\constrd$ 
is derivable.
Finally, from a constraint in normalized form
\mbox{$\sum_i \coeffa_i \litell_i \geq \dega$}
we can use
\emph{division} by a positive integer $\divd$ to derive
\mbox{$\sum_i \ceiling{\coeffa_i / \divd} \litell_i \geq \ceiling{\dega / \divd}$},
dividing and rounding up the degree and coefficients.
For a set of
pseudo-Boolean  
constraints~$\formf'$ we
write
$\formf \cpderives \formf'$
if
$\formf \cpderives \constrc$
for all
$\constrc \in \formf'$.


For 
pseudo-Boolean  
formulas $\formf$, $\formf'$ and constraints $\constrc$,
$\constrc'$, we say that
$\formf$ \emph{implies} or \emph{models}~$\constrc$,
denoted $\formf \models \constrc$, if
any  assignment satisfying~$\formf$ also satisfies~$\constrc$,
and write $\formf \models \formf'$ if
$\formf \models \constrc'$  for all $\constrc' \in \formf'$.
It is easy to see that if
$\formf \cpderives \formf'$
then
$\formf \models \formf'$,
and so
$\formf$
and
$\formf \land \formf'$
are equisatisfiable.
A piece of non-standard terminology that will be convenient for us is
that we will say that 
a constraint~$\constrc$ 
\emph{\syntimplies{}}
another constraint~$\constrc'$
if $\constrc'$ can be derived from~$\constrc$ 
using only
addition of literal axioms $\litell \geq 0$.

An assignment $\assmntrho$ 
\emph{falsifies} or    
\emph{violates} the constraint~$\constrc$ if
the restricted constraint 
$\restrict{\constrc}{\assmntrho}$ can not be satisfied. 
For the normalized constraint $C$ in \eqref{eq:normalized}, this is the case if 
$\sum_{\assmntrho(\litell_i) \neq  0} \coeffa_i < \dega$.
A constraint~$\constrc$ \emph{unit propagates}
the literal~$\litell$ under~$\assmntrho$ if 
if $\restrict{\constrc}{\assmntrho}$
cannot be satisfied unless
$\litell \mapsto 1$.
During \emph{unit propagation}
on $\formf$ under~$\assmntrho$,
the assignment   
$\assmntrho$ is extended 
iteratively by any propagated literals
until an assignment~$\assmntrho'$ is reached under
which no constraint in $\formf$ 
is propagating, or until $\assmntrho'$ violates some constraint~$\constrc\in\formf$. 
The latter scenario is referred to as a
\emph{conflict}.
Using the generalization
of \emph{reverse unit propagation clauses}~\cite{GN03Verification}
to pseudo-Boolean constraints by~\citet{EGMN20Justifying},
we
say that  $\formf$
\emph{implies~$\constrc$ by reverse unit propagation (RUP)},
and that
$\constrc$ is a 
\emph{RUP constraint} with respect to~$\formf$,
if
$\formf \land \negc{\constrc}$
unit propagates to conflict under the empty assignment.
If $\constrc$ is a RUP constraint with respect to~$\formf$,
then it can be proven that there is also a derivation
$\formf \cpderives \constrc$.
More generally, it can be shown that
$\formf \cpderives \constrc$
if and only if
$\formf \land \negc{\constrc} \cpderives \bot$,
where $\bot$ is a shorthand for the
\introduceterm{trivially false constraint}~$0 \geq 1$.
Therefore, we will extend  the    notation and write
$\formf \cpderives \constrc$
also when
$\constrc$ is derivable from~$\formf$
by RUP or by contradiction.
It is worth noting here again that, as shown in~\eqref{eq:negation},
the negation of any 
pseudo-Boolean  
constraint can also be expressed syntactically as a 
pseudo-Boolean  
constraint---this fact will be convenient in what follows.


\section{A Proof System for Dominance Breaking}
\label{sec:proof-system}

\providecommand\label[1]{{\color{red}TODO: \ifthenelse{\boolean{techappendixversion}}{\label{inappendix:#1}}{\label{#1}}}}
\providecommand\localref[1]{{\color{red}TODO:\ifthenelse{\boolean{techappendixversion}}{\ref{inappendix:#1}}{\ref{#1}}\xspace}}
\providecommand\localcref[1]{{\color{red}TODO:\ifthenelse{\boolean{techappendixversion}}{\cref{inappendix:#1}}{\cref{#1}}\xspace}}
\providecommand\localeqref[1]{{\color{red}TODO:\ifthenelse{\boolean{techappendixversion}}{\eqref{inappendix:#1}}{\eqref{#1}}\xspace}}

We proceed to develop our formal proof system for verifying
  dominance breaking, which we have implemented on  top of the tool
  \veripb as developed in 
%
the
sequence of papers
  \cite{%
  	EGMN20Justifying,%
  	GMN20SubgraphIso,%
  	GMMNPT20CertifyingSolvers,%
  	GN21CertifyingParity}.
We remark that for  
applications it is absolutely crucial not only
that the proof system  be sound, but that all proofs 
be efficiently machine-verifiable.  There are significant challenges involved in
making proof logging and verification efficient, but 
in this section we mostly ignore these 
more applied  
aspects of our
work and focus on the theoretical underpinnings.

Our foundation is the cutting planes proof system described in
\cref{sec:prelims}.
However, in a proof in our system for~$(\formf, \objfunc)$,
where $\objective$~is a linear objective function
to be minimized under the pseudo-Boolean formula~$\formf$
(or where $\objfunc \synteq 0$ for 
decision problems),
we also allow strengthening~$\formf$ by adding constraints~$\constrc$
that are not  implied by the formula.
Pragmatically,
adding such non-implied constraints~$\constrc$
should be in order 
as long as we keep some optimal solution,
\ie a satisfying assignment to~$\formf$ that minimizes~$\objfunc$,
which we will refer to as 
an \emph{\fminimal{} solution}
for~$\formf$.
We will formalize this idea by allowing the use 
of an additional pseudo-Boolean formula
$\pordernot(\vecu,\vecv)$ that, together with
an ordered set
of variables~$\vecz$, 
defines  a relation 
$\assmntalpha\preceq\assmntbeta$
to hold between assignments $\assmntalpha$ and~$\assmntbeta$ if
$
  \pordernot(\subst{\vecz}{\assmntalpha},\subst{\vecz}{\assmntbeta})
$
evaluates to
true.
We require (a cutting planes proof) 
that $\pordernot$ is such that 
this defines a preorder, \ie a reflexive and transitive relation.
Adding new constraints~$\constrc$ will be valid as long as we
guarantee to preserve some \fminimal solution that is also minimal
\wrt~$\preceq$.
In other words,
the preorder~$\preceq$ can be combined with the objective~function~$\objfunc$ 
to define a preorder~$\better$ on assignments by
\begin{equation}
  \label{eq:def-better-relation}
  \assmntalpha\better\assmntbeta 
  \quad
  \text{if}
  \quad
  \assmntalpha\preceq\assmntbeta \text{ and }
  \subst{f}{\assmntalpha}\leq \subst{f}{\assmntbeta}
  \eqcomma
\end{equation}
and we require that all derivation steps in the proof should
preserve some solution that is  minimal \wrt~$\better$.
The preorder defined by~$\pordernot(\vecu,\vecv)$
will only become important once we introduce our new
\emph{\domrulelong rule}
later in this section.
For simplicity,  up until that point the reader can assume 
that the pseudo-Boolean formula is
$\pordertop \synteq \emptyset$ 
inducing the trivial preorder relating all assignments, 
though all proofs 
presented below
work in full generality for the orders that will be
introduced later.

A proof  for~$(\formf, \objfunc)$ in our proof system 
consists of a sequence of
\introduceterm{proof configurations}
$\pconfstd$, where
\begin{itemize}
\item $\coreset$ is a set of 
pseudo-Boolean
 \introduceterm{\coreconstraint{}s};
\item
  $\derivedset$ is 
  another set of pseudo-Boolean
 \introduceterm{\derivedconstraint{}s};

\item
 $\pordernot$ is a 
 pseudo-Boolean
 formula 
 encoding a preorder
 and $\vecz$ a 
 set 
 of literals
 on which this preorder will be applied; and
\item
  $\objval$ is the  best value found so far for $\objfunc$. 

\end{itemize}
The initial configuration 
is $\pconfinit$.
%
The distinction between $\coreset$ and $\derivedset$ is only relevant
when a nontrivial 
  preorder
  is used; 
  we will elaborate on  this 
  when discussing
  the \domrulelong rule.
The intended 
relation between
$\objfunc$ and~$\objval$ is that if
$\objval < \infty$,
then there    exists  a solution~$\objsol$
satisfying~$\formf$ such that
$\evalobj{\objfunc}{\objsol} \leq \objval$,
and in this case the proof can make use of the constraint
$\objfunc \leq \objval - 1$
in the search for better solutions. 
As long as the optimal solution has not been found, 
it should hold that \fminimal  solutions
for~%
$\coreset\cup\derivedset$ 
have the same objective value as \fminimal solutions
for~$\formf$.  
The precise relation is formalized in the notion of
\emph{\safe{} configurations},
which we will define next.
%
In some cases, it will be convenient to also have a less stringent
notion of \emph{\weaksafety{}}, which holds even if constraints have
been removed from the original formula.
Jumping ahead a bit, what this means is that proofs preserving
\weaksafety can only be used to show that no solutions better than a
given value exist, while proofs preserving \safety can establish that
the optimal value \emph{equals} a certain value.

\begin{definition}
  \label{def:validity}
  A configuration $\pconfstd$ is \emph{\weaksafe}
  if the following conditions hold:
  \begin{enumerate}
  \item
    \label{def:weak_FC}
    For every
    $\objvalprime < \objval$,
    it holds that
    if
    $\formf \union \set{\objfunc\leq \objvalprime}$
    is satisfiable, 
    then
    \mbox{$\coreset \union \set{\objfunc\leq \objvalprime}$}
    is satisfiable.
		
  \item
    \label{def:weak_CD}\label{def:CD}
    For every 
    total assignment~$\assmntrho$
    satisfying
    the constraints
    $\coreunionobj$, there 
    exists
    %
    %
    a total assignment
    $\assmntrho'\better\assmntrho$
    satisfying \mbox{$\coreunionderivedunionobj$},
    where
    $\better$
    is the relation defined in~\eqref{eq:def-better-relation}.

  \end{enumerate}
  The configuration $\pconfstd$ 
  is \emph{\safe} if
  in addition
  the following conditions hold:
  \begin{enumerate}
    \setcounter{enumi}{2}
  \item  
    \label{def:v-value} 
    If $\objval < \infty$, 
    then 
    $\formf \cup \{\objfunc\leq \objval\}$ is satisfiable.
		
  \item  
    \label{def:FC} 
    %
    %
    For every
    $\objvalprime < \objval$,
    it holds that
    if
    $\coreset \union \set{\objfunc\leq \objvalprime}$
    is satisfiable, 
    then
    $\formf \union \set{\objfunc\leq \objvalprime}$
    is satisfiable.
    
    
  \end{enumerate}
\end{definition}

%
When we present the derivation rules in our proof system below,
we will show that \safety is an invariant of the proof system,  i.e.,
that it is preserved by all derivation rules.
For the \emph{deletion} rule, which can remove constraints in the
input formula, we will present a version of the rule that 
only preserves \weaksafety. This alternative rule is introduced for
pragmatic reasons to support proof logging for SAT
solvers. In such a setting, deletions can be treated in a more relaxed
fashion, since the generated proofs are only used to establish
unsatisfiability.

%
To see why 
\cref{def:validity}
is relevant, 
note that items \ref{def:weak_FC}, \ref{def:weak_CD}, and \ref{def:FC}
together imply that if the configuration $\pconfstd$ is such that
$\objval$ is not yet the value of an optimal solution, then \fminimal
solutions
for~$\formf$ and~$\coreset\cup\derivedset$
have the same
objective value, just as desired.
A proof 
in our proof system
ends when the configuration
$\pconffinal$
is such that
$\coreset \union \derivedset$ contains
contradiction
$\bot \synteq 0 \geq 1$.
If the resulting state is \safe, either $\bestobjval = \infty$ and $\formf$ is
unsatisfiable, 
or $\bestobjval$ is
the optimal value 
(or $\bestobjval = 0$ for a satisfiable decision problem).
If the resulting state is only \weaksafe, we get slightly weaker
conclusions.
We collect the precise statements in two formal theorems.

\begin{theorem}\label{thm:weak_inv_sound}
  Let $\formf$ be a 
  pseudo-Boolean  
  formula and 
  $\objfunc$~an objective function.
  If the configuration~$\pconffinal$
  is  \weaksafe  and is such that
  $\coreset \cup \derivedset$
  contains the contradictory constraint
  \mbox{$0 \geq 1$},
  then it holds that
  \begin{itemize}
  \item for any solution~$\assmntalpha$
    for~$\formf$
    we have
    $\evalobj{\objfunc}{\assmntalpha} \geq \bestobjval$;
  \item
    in particular,
    if $\bestobjval = \infty$, then $\formf$ is unsatisfiable.
  \end{itemize}
\end{theorem}
\begin{proof}
  If $\formf$ is satisfiable, then let $\assmntalpha$ be a satisfying
  assignment
  for~$\formf$. 
  If
  $
  \evalobj{\objfunc}{\assmntalpha}
  <
  \bestobjval 
  $,
  then $\alpha$ satisfies $\formf\cup\{f\leq\bestobjval-1\}$.
  Hence, item~\ref{def:weak_FC}
  in  \cref{def:validity}  
  says that there exists an
  $\alpha'$ that satisfies
  $\coreset\cup\{f\leq\bestobjval-1\}$ and
  item~\ref{def:weak_CD}
  then implies the existence of   
  an $\alpha''$ that satisfies
  $\coreset\cup\derivedset\cup\{f\leq\bestobjval-1\}$.
  But this contradicts
  the assumption that
  the unsatisfiable constraint   
  $0 \geq 1$ is in $\coreset\cup\derivedset$.
  It follows that
  $
  \evalobj{\objfunc}{\assmntalpha}
  \geq
  \bestobjval 
  $,
  and that no satisfying assignments for~$\formf$ can exist
  \mbox{if
    $\bestobjval = \infty$.}
\end{proof}

\begin{theorem}\label{thm:inv_sound}
Let $\formf$ be a pseudo-Boolean   formula and 
$\objfunc$~an objective function.
If the configuration~$\pconffinal$
is \safe and is such that
$\coreset \cup \derivedset$
contains
$0 \geq 1$,
then 
\begin{itemize}
\item $\formf$ is unsatisfiable if and only if $\bestobjval =
  \infty$;
  and
\item if $\formf$ is satisfiable, then there is 
  \anfminimal solution~$\assmntalpha$
  for~$\formf$
  with objective value
  $\evalobj{\objfunc}{\assmntalpha} = \bestobjval$.
\end{itemize}
\end{theorem}
\begin{proof}
  By appealing to 
  \cref{thm:weak_inv_sound},
  we can conclude that  if
  $\bestobjval = \infty$, then $\formf$ is unsatisfiable.
  If $\formf$ is unsatisfiable, then 
    we must have $\bestobjval = \infty$ 
    due to
    item~\ref{def:v-value}
    in  \cref{def:validity}.
    This establishes the first part of
    \cref{thm:inv_sound}.

    For the second part,
    suppose that $\formf$ is satisfiable.
    Then $\bestobjval < \infty$ by the preceding paragraph, and so by
    item~\ref{def:v-value} of
    \cref{def:validity}
    there is a  solution~$\assmntalpha$
    for~$\formf$
    such that
    $
    \evalobj{\objfunc}{\assmntalpha}
    \leq
    \bestobjval
    $.
    But     
    \cref{thm:weak_inv_sound} says that
    for any solution~$\assmntalpha$
    for~$\formf$
    it holds that
    $\evalobj{\objfunc}{\assmntalpha} \geq \bestobjval$.
    Hence,
    $\assmntalpha$~is
    an \fminimal{} solution
    for~$\formf$
    with objective value
    $\evalobj{\objfunc}{\assmntalpha} = \bestobjval$
    as claimed.
\end{proof}

We are now ready to give a formal description of
the rules in our proof system and argue that these rules
preserve  \safety (or, in the case of the alternative deletion rule,
weak \safety).
The attentive reader might have noted that we did not use
item~\ref{def:FC}
in
\cref{def:validity}
in our proofs of
\reftwoths{thm:weak_inv_sound}{thm:inv_sound},
but this condition will be critical to argue that
\safety is an invariant of our proof system.

\subsection{\IMPLICATIONALDERIVATION Rule}

If we can exhibit a derivation of the pseudo-Boolean 
constraint~$\constrc$
from
$\coreunionderivedunionobj$ 
in our (slightly extended) version of 
the cutting planes proof system
as described in  \refsec{sec:prelims}
(\ie in formal notation, if
$\coreunionderivedunionobj \cpderives \constrc$),
%
then we should be allowed to add the implied constraint~$\constrc$
to our collection of derived constraints.
Let us write this down as our first formal derivation rule.

\begin{definition}[\Implicationalderivation rule]
  If $\coreunionderivedunionobj \cpderives \constrc$,
  then we can
  transition from the configuration~$\pconfstd$
  to the configuration $\pconfaddderivednobreak{\set{\constrc}}$
  by the
  \emph{\implicationalderivation{} rule}.
  In this case, we will also say that
  $\constrc$ is derivable from $\pconfstd$
  by the \implicationalderivation{} rule.
%
%
\end{definition}

%


The \implicationalderivation rule preserves \safety (and \weaksafety)
since $\coreunionderivedunionobj \models \constrc$ holds by soundness
of the cutting planes proof system, but, more importantly,
the cutting planes derivation provides a simple and efficient way for
an algorithm to \emph{verify} that this implication holds.  This is a
key feature of all rules in our proof system---not only are they
sound, but the soundness of every rule application can be efficiently
verified by checking a simple, syntactic object.

%
When 
doing proof logging, the solver would need to specify
by which sequence of cutting planes derivation rules $\constrc$ 
was
obtained. For practical purposes, though, 
it greatly simplifies matters that
in many cases the verifier can figure out the required proof details
automatically, meaning that the proof logger can just state the
desired constraint without any further information.  One important
example of this is when $\constrc$ is a reverse unit propagation
(RUP) constraint \wrt $\coreunionderivedunionobj$.  Another case is
when $\constrc$ is \syntimplied by some other constraint.
%
%

\subsection{Objective Bound Update Rule}

The  \emph{objective  bound  update rule}
allows improving 
the   
estimate of what
value can be achieved for the objective function~$\objfunc$.

\begin{definition}[Objective  bound  update rule]
  We say that we can transition from
  the configuration
  $\pconfstd$ to
  the configuration
  $\pconfboundupdate$ by the \emph{objective bound update rule} if
  there is an  assignment $\assmntalpha$ satisfying $\coreset$ such
  that 
  $\evalobj{\objfunc}{\assmntalpha} = \objval' < \objval$. 
\end{definition}

%

When actually doing proof logging, the solver would specify such an
assignment~$\assmntalpha$, which would then be checked by the proof
verifier (in our case \veripb).

To 
argue
that this rule preserves
\safety (and weak \safety),
note that 
items  \ref{def:weak_FC}, \ref{def:CD}, and  \ref{def:FC}
%
are trivially satisfied.
(For item~\ref{def:CD}, observe that whenever
$\assmntrho'\better\assmntrho$ and $\assmntrho$ satisfies
$\{f\leq \objval'-1\}$, so does $\assmntrho'$.)  
Item~\ref{def:v-value}
is satisfied after updating the objective bound since 
item~\ref{def:FC} 
guarantees the existence of an $\alpha'$ satisfying $\formf$ with an
objective value that is at least as good 
as~$\objval'$.


Note that we have no guarantee that~$\alpha$ itself will be
a solution
for~$\formf$.
However, although we will not emphasize this
point here,
it follows from our formal treatment below that the proof system
guarantees that 
such a solution~$\alpha'$ 
for
the original formula~$\formf$ can be efficiently reconstructed from the
proof (where efficiency is measured in the size of the proof).

\subsection{\REDRULELONG Rule}
\label{sec:red-rule}

The \redrulelong rule allows deriving a
constraint~$\constrc$ from $\coreset \cup \redundantset$
even if~$\constrc$ is
not implied, provided that it can be shown that any 
assignment $\assmntalpha$ that satisfies $\coreset \cup \redundantset$
can be transformed into another 
assignment~$\assmntalpha' \better \assmntalpha$ 
that satisfies both
$\coreset \cup \redundantset$
and~$\constrc$
(in case $\pordernot = \pordertop$, the 
condition $\assmntalpha' \better \assmntalpha$ 
just means that
$\evalobj{\objfunc}{\assmntalpha'}\leq\evalobj{\objfunc}{\assmntalpha}$). 
This rule is borrowed from~\citet{GN21CertifyingParity}, 
who in turn rely heavily 
on~\citet{HKB17ShortProofs} and \citet{BT19DRAT}. 
We extend this rule here 
from decision problems to 
optimization problems in the natural way.

\begin{definition}[\Redrulelong rule]
  If for a pseudo-Boolean constraint~$\constrc$
  there is a substitution~$\witness$  such that
  \begin{equation} 
    \label{eq:substredundant}
    \coreunionderivedunionobj \union \set{\negc{\constrc}}
    \cpderives
    \subst
    {(\coreset \union \derivedset \union \constrc)}
    {\witness}
    \union 
    \objwitnessconditionset
    \union
    \porderformsub
    \eqcomma
  \end{equation}
  then  we can transition from
  $\pconfstd$ to $\pconfaddderivednobreak{\set{\constrc}}$ by
  \emph{\redrulelong{}},
  or just \introduceterm{\redruleshort{}} for brevity.
  We refer to the substitution~$\witness$
  as the \introduceterm{witness}, and also say that
  $\constrc$
  can be derived 
  from $\pconfstd$  by \redruleshort.
\end{definition}

%

Intuitively, \eqref{eq:substredundant} says that if some assignment
$\assmntalpha$ satisfies $\coreset \union \derivedset$ but
falsifies~$\constrc$,  
then
the assignment  
$\assmntalpha' = \assmntalpha \circ \witness$
still satisfies $\coreset \union \derivedset$ and
also satisfies $\constrc$. In addition, the condition
$\objwitnesscondition$ ensures that $\assmntalpha \circ \witness$
achieves
an objective function value that is at least as good
as that for~$\assmntalpha$.
This together with the constraints~$\porderformsub$
guarantees 
that $\assmntalpha'\better\assmntalpha$.
For proof
logging purposes, the witness~$\witness$ as well as any non-immediate
cutting planes 
derivations of constraints on the right-hand side of
\eqref{eq:substredundant} would have to be specified,
but, e.g., 
{all RUP constraints or \syntimplied constraints}
can be left to the verifier to check.

\begin{proposition}
  If we can transition from
  $\pconfstd$ to $\pconfaddderivednobreak{\set{\constrc}}$ by
  the \redrulelong{} rule
  and
  $\pconfstd$ is [weakly] \safe, then
  \mbox{$\pconfaddderived{\set{\constrc}}$}
  is also [weakly] \safe.
%
\end{proposition}

\begin{proof}
First assume $\pconfstd$ is \weaksafe. 
Item~\ref{def:weak_FC}
in \cref{def:validity} 
remains satisfied
for $\pconfaddderived{\set{\constrc}}$  
since $\formf$,  $\objval$, and 
$\coreset$ are unchanged. 


%
For item~\ref{def:CD} we can recycle proofs of similar properties
for decision problems
\cite{HKB17ShortProofs,BT19DRAT,GN21CertifyingParity}.
Consider  
a total  assignment~$\assmntrho$
satisfying
$\coreset \cup \betterobjvalset$.
Without loss of generality we can assume that
$\assmntrho$ also satisfies~$\derivedset$
(since $\pconfstd$ satisfies item~\ref{def:weak_CD}
in \cref{def:validity}).
%
We wish to  construct an 
assignment~$\assmntrho'$  that satisfies the constraints
$\coreset\cup\derivedset\cup\set{\constrc}$
and also the condition
$\assmntrho'\better\assmntrho$,
which is to say that
$\evalobj{\objfunc}{\assmntrho'} \leq 
\evalobj{\objfunc}{\assmntrho}$ and 
$\pordernot(\subst{\vecz}{\assmntrho'},\subst{\vecz}{\assmntrho})$
should hold. 

If $\assmntrho$ satisfies
$\constrc$, then we use $\assmntrho' = \assmntrho$ and all conditions
are satisfied (recall that $\pordernot$ induces a preorder, and hence
a reflexive relation: for any $\rho$,
$\pordernot(\subst{\vecz}{\assmntrho},\subst{\vecz}{\assmntrho})$
holds).  
Otherwise, choose 
$\assmntrho' = \assmntrho \circ \witness$.
Since $\assmntrho$ is total and does not satisfy~$\constrc$,
it satisfies~$\neg \constrc$.
This means that $\assmntrho$ satisfies
$\coreset \cup \derivedset \cup \neg \constrc$,
and hence by \eqref{eq:substredundant}
also
\begin{equation}
  \restrict{(\coreset \cup \derivedset \cup \constrc)}{\witness} 
  \cup \objwitnessconditionset \cup \porderformsub
  \eqperiod  
\end{equation}
Clearly,
for any constraint $\constrd$, 
it holds that
$\restrict{(\restrict{\constrd}{\witness})}{\assmntrho} =
\restrict{\constrd}{\assmntrho \circ \witness}$ and thus if $\assmntrho$
satisfies $\restrict{\constrd}{\witness}$, then 
$\assmntrho ' = \assmntrho \circ \witness$ 
satisfies $\constrd$. 
Therefore, 
$\assmntrho'$~satisfies 
$\coreset \cup
\derivedset \cup \constrc$.
By the same reasoning, since $\assmntrho$ satisfies
$\objwitnesscondition$
we have that
$
\evalobj{\objfunc}{\assmntrho \circ \witness}
=
\evalobj{\objfunc}{\assmntrho'}
\leq
\evalobj{\objfunc}{\assmntrho}$
holds.  
Finally, the fact that
$\assmntrho$ satisfies $\porderformsub$  means that
$\pordernot(\subst{\vecz}{\assmntrho'},\subst{\vecz}{\assmntrho})$ 
holds. This shows that
item~\ref{def:weak_CD} holds 
for $\pconfaddderived{\set{\constrc}}$,
which concludes our proof that \weaksafety is preserved.

To see that also \safety is preserved, it suffices to note that
items~\ref{def:v-value} and~\ref{def:FC} do not depend on
$\derivedset$ and hence are satisfied in
$\pconfaddderived{\set{\constrc}}$ whenever they are satisfied in
$\pconfstd$.
\end{proof}

\subsection{Deletion Rules}
\label{sec:del}


An interesting property of the 
\redruleshort rule introduced in 
\cref{sec:red-rule}
is that when the set of constraints in
$\coreunionderived$
grows, it can get harder to derive new constraints,
since every constraint~$\constrd$ already in~$\coreunionderived$
adds a new constraint~$\restrict{\constrd}{\witness}$
needing to be derived on the right-hand side of~\eqref{eq:substredundant}.
For this reason,
and also due to efficiency concerns during verification
of proofs, we  need to be able to delete previously derived constraints.

\begin{definition}[Deletion rule]
  \label{def:deletion-rule}
  We can transition from   $\pconfstd$ 
  to   $\pconfcdprime$
  by  the \emph{deletion rule} 
  if
  \begin{enumerate}
  \item
    $\derivedset' \subseteq \derivedset$ and
  \item%
    \label{item:core-condition}%
    $\coreset' = \coreset$ or $\coreset' = \coreset \setminus \set{\constrc}$
    for some constraint $\constrc$
    derivable via 
    the \redruleshort rule
    from $\pconfdelete$.
  \end{enumerate}
\end{definition}

%
%

The
last condition  above perhaps seems slightly odd, 
but
it is there since   
deleting arbitrary constraints could
violate \safety in two different ways.  
Firstly, it
could   
allow finding better-than-optimal solutions. 
Secondly, and perhaps surprisingly, in combination with the 
\domrulelong rule, which we will discuss below,
arbitrary deletion is unsound, 
as
it can turn satisfiable instances into unsatisfiable ones.  
We will discuss this further in 
\cref{ex:unsound_deletion}
once we have introduced the \domruleshort rule.

\begin{proposition}
  If we can transition from
  $\pconfstd$ to $\pconfcdprime$
  by 
  the deletion rule
  and
  $\pconfstd$ is [weakly] \safe, then
  \mbox{$\pconfcdprime$}
  is also [weakly] \safe.
\end{proposition}

\begin{proof}
  First assume that $\pconfstd$ is \weaksafe. 
  Item~\ref{def:weak_FC}
  in \cref{def:validity}   
  clearly remains satisfied
  for $\pconfcdprime$    
  since $\coreset'\subseteq\coreset$. 
  To prove that  item~\ref{def:CD}  is satisfied after applying the
  deletion   
  rule, let 
  $\assmntalpha$ be any assignment that satisfies
  $\coreset'\cup\{\objfunc\leq \objval-1\}$. 
  If $\coreset' = \coreset$, we find
  an~$\assmntalpha'\better\assmntalpha$
  that satisfies
  $\coreset'\cup\derivedset'\cup\{\objfunc\leq \objval-1\}$ using
  \weaksafety of the
  configuration~
  $\pconfstd$.  
  Hence, we can assume that 
  $\coreset' = \coreset
  \setminus \set{\constrc}$. If $\assmntalpha$ satisfies $\constrc$,
  this assignment satisfies all of~$\coreset$,
  and again the claim follows from \weaksafety
  of
  the configuration~$\pconfstd$ before deletion. 
  Assume therefore that
  $\assmntalpha$ does not
  satisfy~$\constrc$. 
  Since $\constrc$ is derivable via \redruleshort{} from
  $\pconfdelete$, it holds that 
  \begin{equation} 
    \coreset'  \cup\{\objfunc\leq \objval-1\} \cup \set{\negc{\constrc}}
    \cpderives
    \subst
    {(\coreset' \cup \constrc)}
    {\witness}
    \union 
    \objwitnessconditionset
    \union
    \porderformsub
  \end{equation}
  for some
  witness $\witness$. 
  We can use this witness to obtain an assignment  
  $\alpha'' = \alpha\circ \omega$ 
  satisfying $\coreset=\coreset' \cup \set{\constrc}$ 
  such that
  $
  \subst{\objfunc}{\alpha''}
  \leq\subst{\objfunc}{\alpha}\leq
  \objval-1$,
  which shows that
  $\coreset\cup\{\objfunc\leq \objval-1\}$ 
  is satisfiable. 
  Appealing to the \weaksafety   of the
  configuration~$\pconfstd$ before deletion,
  we then find an $\alpha'$ with
  $\subst{\objfunc}{\alpha'}\leq\objval-1$
  that satisfies $\coreset\cup\derivedset\cup\{\objfunc\leq
  \objval-1\}$.
  In this way, we establish that
  item~\ref{def:CD}  
  in \cref{def:validity}
  holds for the configuration~$\pconfcdprime$ after deletion.

  Now assume that $\pconfstd$ is \safe. 
  Item~\ref{def:v-value}
  clearly remains satisfied
  for $\pconfcdprime$  
  since $\objval$ is unchanged. 
  We show that item~\ref{def:FC} holds
  for~%
  $\pconfcdprime$; the proof is very similar to the proof of
  item~\ref{def:CD} above.
%
  Starting from an assignment~$\assmntalpha$ satisfying
  $\coreset'\cup\{\objfunc\leq \objval'\}$
  for some   $\objval'<\objval$, 
  we use~$\assmntalpha$ to construct a satisfying
  assignment~$\assmntalpha'$ for 
  $\formf\cup\{\objfunc\leq \objval'\}$.
  If $\coreset' = \coreset$,
  we get $\assmntalpha'$ from the
  \safety of~%
  $\pconfstd$,
  so assume
  $\coreset' = \coreset
  \setminus \set{\constrc}$. If $\assmntalpha$ satisfies $C$,
  the same assignment
  satisfies $\coreset$, and again the claim follows from \safety of
  the configuration before deletion. 
  Assume therefore that
  $\assmntalpha$ does not
  satisfy~$\constrc$. 
  Since $\constrc$ is derivable via \redruleshort{} from
  $\pconfdelete$, it holds that 
  \begin{equation} 
    \coreset' \cup\{\objfunc\leq \objval-1\} \cup \set{\negc{\constrc}}
    \cpderives
    \subst
    {(\coreset' \cup \constrc)}
    {\witness}
    \union 
    \objwitnessconditionset
    \union
    \porderformsub
    \eqperiod
  \end{equation}
  From this we can construct
  an assignment  
  $\alpha'' = \alpha\circ \omega$ 
  that satisfies 
  $\coreset=\coreset' \cup \set{\constrc}$
  and is   
  such that
  $
  \subst{\objfunc}{\alpha''}
  \leq\subst{\objfunc}{\alpha}\leq
  \objvalprime$,
  showing that 
  $\coreset\cup\{\objfunc\leq \objvalprime\}$ 
  is satisfiable. 
  Appealing to the \safety of the
  configuration~$\pconfstd$,
  we then find an $\alpha'$ with
  $\subst{\objfunc}{\alpha'}\leq\objvalprime$
  that satisfies~$\formf$.
  This proves that
  item~\ref{def:FC}
  holds
  for the configuration~$\pconfcdprime$ after deletion.
\end{proof}

%

The deletion rule in \cref{def:deletion-rule} is more cumbersome than
that in \drat-style proof systems, in that deletions of constraints in
the core set must be accompanied by proofs so that the validity of the
deletions can be checked.
When we want to highlight this property of the rule, we will refer to it as
\emph{\checkeddeletion}.
This
property 
can make it more difficult to
implement proof logging in a solver, and can also have negative
effects on the time required for proof generation and proof
verification.
If we are only interested in certifying unsatisfiability of decision
problem instances---which has traditionally been  the case in Boolean
satisfiability solving---then an alternative is to use a more liberal
deletion rule that allows unrestricted deletion of constraints from
the core set~$\coreset$ provided that the derived set~$\derivedset$ is
empty.
When such an \emph{\uncheckeddeletion rule} is used,
it is easy to see that unsatisfiable sets of constraints can turn
satisfiable, and for optimisation problems 
spurious solutions can appear that yield better objective function
values than are possible for the original input constraints.
But proofs of unsatisfiability are still valid, as are lower bounds on
the value of the objective function to be minimized. 
Phrased in the language of 
\cref{def:validity},
we are guaranteed to preserve
\weaksafety, but not necessarily \safety.

\begin{definition}[\Uncheckeddeletion rule]
  \label{def:unchecked_del}
  If  $\coreset'\subseteq \coreset$, then
  we can transition from  
  the configuration
  $\pconfstd$ to
  the configuration
  $\pconfdelete$
  using the \emph{\uncheckeddeletion rule}.
\end{definition}

%
\begin{proposition}
  If $\pconfstd$ is a \weaksafe configuration and we can transition
  from it to $\pconfdelete$ using \uncheckeddeletion, then
  $\pconfdelete$ is also \weaksafe.  
\end{proposition}

\begin{proof}
  Item~\ref{def:weak_FC} in \cref{def:validity} is clearly maintained
  when transitioning to~$\pconfdelete$ since the set of core
  constraints shrinks.  As to item~\ref{def:CD}, it is trivially
  satisfied when the set of derived constraints is empty.
%
\end{proof}

\subsection{Transfer Rule}
\label{sec:rule-transfer}

Constraints can always be moved from the derived
set~$\derivedset$ to the core set~$\coreset$
using the \introduceterm{transfer rule}.

\begin{definition}[Transfer rule]
  We can transition
  from~$\pconfstd$ to~$\pconfcprime$
  by the \emph{transfer rule}
  if
  $\baseset \subseteq \baseset' \subseteq \baseset \cup \derivedset$.
\end{definition}

This  transfer rule clearly preserves \safety (and weak \safety).

The transfer rule together with deletion allows replacing constraints
in the original formula with stronger constraints.
For example,
assume that $\varx + \vary \geq 1$ is in $\coreset$ and
that  
we derive $\varx \geq 1$.
Then we can
move $\varx \geq 1$ from
$\derivedset$ to~$\coreset$ and then delete $\varx + \vary \geq 1$.
The required \redruleshort{} check
\mbox{%
  $\set{\varx \geq 1,
    \neg(\varx + \vary \geq 1)} \cpderives \bot$}
is immediate.

The rules discussed so far 
in this section   
do not change~$\pordernot$,
and so any derivation using these rules only will operate with the
trivial preorder~$\pordertop$ imposing no conditions.
The proof system defined in terms of these rules is a straightforward
extension  of \veripb as developed
by \citet{%
  EGMN20Justifying,%
  GMMNPT20CertifyingSolvers,%
  GMN20SubgraphIso,%
  GN21CertifyingParity}
to an optimisation setting.
We next discuss the main contribution of this paper, namely the new
dominance rule making use of the
preorder encoding~$\pordernot$.

\subsection{\DOMRULELONG Rule}
\label{sec:dominiance-rule}

Any preorder~$\preceq$
induces a strict order~$\prec$
defined by $\alpha\prec \beta$ if
$\alpha\preceq\beta$ and $\beta\not\preceq\alpha$. (Note that if $\alpha\preceq\beta$ and $\beta\preceq\alpha$ both hold, this means that neither $\alpha\prec\beta$ nor $\beta\prec\alpha$ holds.)  
The relation~$\strictlybetter$ obtained in this way
from the preorder~\eqref{eq:def-better-relation} 
coincides with what 
\citet{CS15Dominance}
call
a \emph{dominance relation} in the context of constraint
optimisation. 
Our \domruleshort rule 
allows deriving a constraint~$\constrc$ from
$\coreset \cup \redundantset$
even if~$\constrc$ is not implied,
similar to the \redruleshort{} rule. 
%
However, for the \domruleshort rule an
assignment~$\assmntalpha$ satisfying $\coreset \cup \redundantset$ but 
falsifying~$\constrc$ 
only needs to be  
mapped to an 
assignment~$\assmntalpha'$ that satisfies
$\coreset$, but not necessarily~$\derivedset$ or~$\constrc$.
On the other hand, the new assignment~$\assmntalpha'$ should satisfy
the strict inequality~$\assmntalpha'\strictlybetter\assmntalpha$
and not just~$\assmntalpha'\better\assmntalpha$
as in the \redruleshort{} rule.
To show that this 
new \domruleshort   
rule preserves \safety, we will prove that it is
possible to construct an assignment that satisfies
$\coreset \cup \derivedset \cup \set{\constrc}$ by iteratively
applying the witness of the \domruleshort rule, in combination with
\safety of the configuration before application of the \domruleshort rule.  
As our base case, 
if $\assmntalpha'$ satisfies $\coreset \cup \derivedset \cup
\set{\constrc}$, 
we are done. Otherwise, since  
$\assmntalpha'$  satisfies~$\coreset$, 
by \safety we are guaranteed
the existence of an
assignment $\assmntalpha''$ 
satisfying  
$\coreset\cup\derivedset$ 
for which
$
\assmntalpha''
\strictlybetter   
\assmntalpha'
\strictlybetter
\assmntalpha
$
holds. 
If~$\assmntalpha''$ still does not satisfy~$\constrc$,
we can repeat the argument. 
In this way, we get a strictly decreasing sequence 
(\wrt~$\strictlybetter$) 
of assignments. Since the set of possible assignments is finite, 
this sequence will eventually terminate.

\providecommand{\union }{\!\cup\!}
\renewcommand{\union }{\cup}
 
Formally, we can derive~$\constrc$
by \domrulelong  
given  
a substitution~$\witness$ such that
\begin{equation} 
  \label{eq:dom_def_simple}
  \coreunionderivedunionobj \union 
  \set{\negc{\constrc}}
\cpderives
  \subst{\coreset}{\witness} \union 
  \porderform{\subst{\ovars}{\witness}}{\ovars} \cup
\neg \porderform{\ovars}{\subst{\ovars}{\witness}} \union 
  \objwitnessconditionset
  \eqcomma
\end{equation}
where $\porderform{\subst{\ovars}{\witness}}{\ovars} $ and $ \neg
\porderform{\ovars}{\subst{\ovars}{\witness}}$ 
together state that \mbox{$\alpha\circ\witness\prec \alpha$} for any
assignment~$\alpha$.
A minor technical problem is that the pseudo-Boolean formula  
$\porderform{\ovars}{\subst{\ovars}{\witness}}$ may contain multiple
constraints,
so that
the negation of it is no longer a PB formula. 
To get around this, we split~%
\eqref{eq:dom_def_simple} into two separate 
conditions
 and 
shift
$\neg\porderform{\ovars}{\subst{\ovars}{\witness}}$ 
to the premise of the implication, 
which eliminates the negation.

After this adjustment, the formal version of our
\introduceterm{\domrulelong rule}, or just
\introduceterm{\domruleshort rule} for brevity, can be stated as
follows.

\begin{definition}[\Domrulelong rule]
  \label{def:dominance-rule}
  If for a pseudo-Boolean constraint~$\constrc$ there is a
  \emph{witness} substitution~$\witness$ such that the conditions
  \begin{subequations}
    \begin{alignat}{1}
      & \coreunionderivedunionobj \union  
      \set{\negc{\constrc}} \cpderives \subst{\coreset}{\witness}
      \union 
      \porderformdom
      \union 
      \objwitnessconditionset\label{eq:dom:one}
      \\
      & \coreunionderivedunionobj
      \union   \set{\negc{\constrc}} \union   \porderformdomneg
      \cpderives
      \bot
      \label{eq:dom:two}
    \end{alignat}
  \end{subequations}
  are satisfied, then we can transition from~$\pconfstd$
  to~$\pconfaddderived{\set{\constrc}}$
  using \emph{\domrulelong}, or just
  \emph{\domruleshort} for brevity.
  In this case, we also say that $\constrc$ is derivable from
  $\pconfstd{}$ by \domruleshort.
%
%
\end{definition}

%

Just as for the \redruleshort{} rule, the witness~$\witness$
as well as any non-immediate derivations would have to be specified
in the proof log.

\begin{proposition}  
  If we can transition from
  $\pconfstd$ to $\pconfaddderivednobreak{\set{\constrc}}$ by
  the \domrulelong rule
  and
  $\pconfstd$ is [weakly] \safe, then
  \mbox{$\pconfaddderived{\set{\constrc}}$}
  is also [weakly] \safe.
%
\end{proposition}

\begin{proof}
  Since $\formf$, $\coreset$, and~$\objval$
  are not affected in a transition using the \domruleshort rule,
  items~\ref{def:weak_FC},~\ref{def:v-value}, and~\ref{def:FC}
  in  \cref{def:validity}   
  hold
  for~%
  $\pconfstd$ if and only if they hold
  for~%
  $\pconfaddderived{\set{\constrc}}$.
  Hence, we only need to prove that
  item~\ref{def:weak_CD} holds
  for~%
  $\pconfaddderived{\set{\constrc}}$.  
  Assume towards contradiction that it \emph{does not} hold.
  Let $S$ denote the set of
  total   
  assignments $\assmntalpha$  that 
  \begin{enumerate}[(1)]
  \item satisfy $\coreset\cup\{\objfunc\leq\objval-1\}$
    and 
  \item 
    admit no $\assmntalpha'\better\assmntalpha$ 
    satisfying $\coreset\cup\derivedset\cup\{\constrc\}$.
\end{enumerate}
By our assumption, $S$ is non-empty. 

Let $\assmntalpha$ be some $\strictlybetter$-minimal assignment in~$S$  (since $\strictlybetter$ is a strict order and $S$ is finite, minimal elements of $S$ exist).
Since $\pconfstd$ is \weaksafe, there exists 
some~$\assmntalpha_1 \better \assmntalpha$
that satisfies $\coreset\cup\derivedset$.
We know that $\assmntalpha_1$ cannot satisfy $\constrc$ since $\assmntalpha\in S$. 
Hence, $\assmntalpha_1$ satisfies
$\coreset\cup\derivedset\cup\{\neg \constrc\}$.
From \eqref{eq:dom:one} it 
follows that $\assmntalpha_1$ satisfies
$\porderformdom       \cup       \objwitnessconditionset$  
 and thus that
 $\pordernot(\subst{\vecz}{\assmntalpha_1\circ\witness},\subst{\vecz}{\assmntalpha_1})$
 and
 $\subst{\objfunc}{\assmntalpha_1\circ\witness}\leq\subst{\objfunc}{\assmntalpha_1}$ 
 hold.
 In other words, $\assmntalpha_1\circ\witness \better\assmntalpha_1$.  
 By~\eqref{eq:dom:two}, it follows that $\assmntalpha_1$ does not
 satisfy $\porderformdomneg$, i.e.,
 $\pordernot(\subst{\vecz}{\assmntalpha_1},\subst{\vecz}{\assmntalpha_1\circ\witness})$
 does not hold, 
 and thus
 $\assmntalpha_1\not\better\assmntalpha_1\circ\witness$.  
 Now let $\assmntalpha_2$ be $\assmntalpha_1\circ\witness$. 
 We showed that $\assmntalpha_2\strictlybetter\assmntalpha_1\better \assmntalpha$. 
 Furthermore, since $\assmntalpha_1$ satisfies
 $\coreset\cup\derivedset\cup\{\neg\constrc\}$, \eqref{eq:dom:one}
 yields that $\assmntalpha_2$ satisfies $\coreset$.  
 Thus $\assmntalpha_2$ satisfies $\coreset\cup\{\objfunc\leq\objval-1\}$. 
 Since $\assmntalpha_2\strictlybetter\assmntalpha$, and $\assmntalpha$ is a
 minimal element of $S$, it cannot be that $\assmntalpha_2\in S$.  
 Thus, there must exist a $\assmntalpha'\better \assmntalpha_2$ that
 satisfies  $\coreset\cup\derivedset\cup\{\constrc\}$.  
 However,
 it also holds that
 $\assmntalpha'\better\assmntalpha$, and since
 $\assmntalpha\in S$ this means that $\alpha'$ cannot satisfy
 $\coreset\cup\derivedset\cup\{\constrc\}$.
 This yields a
 contradiction, thereby finishing our proof.  
\end{proof}

%
%

When introducing the deletion rule, we already mentioned that deleting
arbitrary constraints can be unsound in combination with
\domrulelong. 
We now illustrate this phenomenon. 
\begin{example}\label{ex:unsound_deletion}
 Consider the formula 
 $F = \set{  p\geq 1}$
 with  
 objective $f\synteq 0$
 and the configuration 
 \begin{equation}
   (\coreset_1 = \{p\geq 1\},\derivedset_1 = \{p\geq 1\}, 
   \pordernot, \{p\}, \infty)
   \eqcomma   
 \end{equation}
 where $\pordernot(u,v)$ is defined as
 $\set{v+\olnot u\geq 1}$. 
 This configuration is \safe and  
  $\coreset\cup\derivedset$ is satisfiable. 
 \emph{If} we were allowed to delete constraints arbitrarily from
 $\coreset$, we could derive 
a configuration with
$\coreset_2 = \emptyset $ and $\derivedset_2= \{p\geq 1\}$. 
However, now
the \domruleshort rule can  derive
$C \synteq \olnot p \geq 1$,
using the 
witness
 $\witness = \set{p \mapsto 0}$. 
 To see that all conditions for applying \domrulelong are indeed
 satisfied, we notice that
\mbox{\eqref{eq:dom:one}--\eqref{eq:dom:two}}
simplify to
 \begin{subequations}
 \begin{alignat}{1}
      & \emptyset\union \{p\geq 1\} \cup \{p\geq 1\}\cpderives 
        \emptyset 
        \union \{p + 1 \geq 1\} \cup \emptyset 
\\
      & \emptyset\union \{p\geq 1\} \cup \{p\geq 1\} \union \{0 +\olnot p \geq 1\} \cpderives \bot
\end{alignat}
\end{subequations}
and both of these derivations are immediate
(e.g., by reverse unit propagation).
This means that we can derive a third configuration with
$\coreset_3 = \emptyset $ and
$\derivedset_3= \set{p\geq 1, \olnot{p} \geq 1}$,
from which we immediately get contradiction
$0 \geq 1$
by adding the two constraints in~$\derivedset_3$,
although the formula~$\formf$ that we started with is satisfiable.
\end{example}

\begin{remark}
  The \uncheckeddeletion rule introduced in \cref{def:unchecked_del}
  requires that the
  derived  
  set~$\derivedset$ be empty. 
  \cref{ex:unsound_deletion} gives the motivation for this,
  highlighting the complex interplay between
  \domrulelong and deletion  of constraints.
  However, in the special case where the pre-order is the trivial
  order $\pordertop$, the dominance rule can only be used to derive
  implied constraints. 
  In this case, we could in principle weaken the condition for
  \uncheckeddeletion to also allow transitioning  
  from   $(\coreset, \derivedset, \pordertop, \ovars,  \objval)$ to
  any configuration~$(\coreset', \derivedset, \pordertop, \ovars,
  \objval)$  whenever $\coreset'\subseteq \coreset$.  

  Allowing
  such a further relaxation of the deletion rule
  might be
  especially useful if one does not want to make a
  distinction between core and derived constraints in proof logging
  applications
  such as SAT solving. 
  However, if we would want to extend the \uncheckeddeletion rule to
  cover also this case, then such \uncheckeddeletion transitions
  would no longer
  preserve weak validity.
  In addition, our invariant would need to be extended with a special
  case stating that item~\ref{def:weak_CD}
  in  \cref{def:validity}  
  should hold only if   $\pordernot$ is nontrivial.
%
  For reasons of mathematical elegance, we  chose not to adopt such a
  specially tailored version of the \uncheckeddeletion rule in the
  current paper. We would like to point out, however, that such a
  version of \uncheckeddeletion has been implemented in the formally
  verified pseudo-Boolean proof checker
  used in 
  the SAT
  Competition 2023 \cite{BMMNOT23DocumentationVeriPB}.
\end{remark}

\subsection{Preorder Encodings}\label{sec:order_definition}
   
As mentioned before, $\pordernot$ is shorthand for a pseudo-Boolean
formula~$\porderform{\vecu}{\vecv}$
over two sets of 
placeholder variables
$\vecu = \set{\varu_1, \ldots, \varu_\vecend}$
and
$\vecv = \set{\varv_1, \ldots, \varv_\vecend}$
of equal size, which should also match the size of~$\vecz$
in the configuration.
To use 
$\pordernot$  in a proof,
it is required to 
show
that 
this formula encodes
a preorder.
This is done by providing 
(in a proof preamble) cutting planes derivations
establishing   
\begin{subequations}
 \begin{alignat}{1} 
\label{eq:order-reflex}
  &\emptyset \cpderives\porderform{\vecu}{\vecu}\\
%
  \label{eq:order-trans}
  &\porderform{\vecu}{\vecv}
  \union
  \porderform{\vecv}{\vecw}
  \cpderives
  \porderform{\vecu}{\vecw}
\end{alignat}
\end{subequations}
where \eqref{eq:order-reflex} formalizes reflexivity and
\eqref{eq:order-trans} transitivity
(and where  notation like 
$\porderform{\vecv}{\vecw}$
is shorthand for applying to 
$\porderform{\vecu}{\vecv}$
the substitution~$\witness$ that maps
$\varu_i$ to~$\varv_i$
and
$\varv_i$ to~$\varw_i$,
as discussed in \cref{sec:prelims}).
These two conditions guarantee that 
the relation~$\assmntorder$ defined by
$\assmntalpha \assmntorder \assmntbeta$
if 
$\porderform{\subst{\vecz}{\assmntalpha}}{\subst{\vecz}{\assmntbeta}}$
forms a preorder on the set of assignments. 

By way of example,
to encode the lexicographic order
$
\varu{}_1 \varu{}_2 \ldots \varu_\vecend 
\lexorder
\varv{}_1 \varv{}_2 \ldots \varv_\vecend,$
we can use a single constraint
\begin{equation}
  \label{eq:encode-LL}
  \plexorderformstd
  \synteq 
  \sumnodisplay_{i=1}^{\vecend} 2^{\vecend{} - i} \cdot (\varv_i - \varu_i) \geq 0
  \eqperiod
\end{equation}
Reflexivity is vacuously true since
$
\plexorderform{\vecu}{\vecu} \synteq 0 \geq 0
$, 
and transitivity also follows easily since adding
$\plexorderform{\vecu}{\vecv}$
and
$\plexorderform{\vecv}{\vecw}$
yields
\ifthenelse{\boolean{briefconfversion}}
{$\plexorderform{\vecu}{\vecw}$.}
{$\plexorderform{\vecu}{\vecw}$
  (where 
  we tacitly assume that the constraint resulting from this
  addition is implicitly simplified 
  by collecting like terms, performing any cancellations, and shifting
  any constants to the right-hand side of the inequality,
  as mentioned in
  {\cref{sec:prelims}}).}

A potential concern with encodings such as~\eqref{eq:encode-LL} is
that coefficients can become very 
large as the number of variables in the order grows.
It is perfectly possible to
address this by allowing order encodings 
using
auxiliary
variables in addition to~$\vecu$ and~$\vecv$. 
We have chosen not to develop the theory for this in the current
paper,
however,
since we feel that 
it would make the exposition significantly more complicated without
providing a commensurate gain in terms of scientific contributions.

\subsection{Order Change Rule}

The final proof rule that we need is a rule for introducing a nontrivial order, 
and it turns out that it can also be convenient to be able
to use different orders at different points in the proof.
Switching orders is possible, 
but to maintain soundness it is important to
first clear the set~$\derivedset$ 
(after transferring the constraints we want to keep to~$\baseset$
using the transfer rule in \cref{sec:rule-transfer}). 
The reason for this is simple: if we 
allow arbitrary order
changes,
then
item~\ref{def:CD} 
of \weaksafety would no longer hold,
and we could potentially derive contradiction even from satisfiable
formulas.
However,  when
$\derivedset=\emptyset$
the condition in  item~\ref{def:CD} 
is trivially true.

\begin{definition}[Order change rule]
  Provided that
  $\pordernottwo$ has been 
  established
  to be a
  preorder over $2\cdot n$ variables
  (by proofs of
  \eqref{eq:order-reflex} and~\eqref{eq:order-trans}
  with explicit cutting planes derivations)  
  and provided that $\ovars_2$ is a list of $n$ variables,
  we say that we can 
  transition from the
  configuration $\pconfpreorderchange$ to the configuration
  $\pconfpostorderchange$ by the \emph{order change rule}.
\end{definition}


As explained above, it is clear that this rule 
preserves  \safety (and weak \safety).

\subsection{Concluding Remarks
  on the Proof System for Dominance Breaking
}

This concludes the presentation of our proof system,
which we have defined in two slightly different flavours with
different conditions for the deletion rule.
In case we use the version with
\uncheckeddeletion, each rule in the proof system
has been shown to preserve
\weaksafety, where as in the version with
(standard) deletion each derivation rule has been shown to preserve
\safety.  
The initial
configuration is clearly \safe.   
Therefore, by \cref{thm:weak_inv_sound} 
our proof system is sound:
whenever we can derive a configuration $\pconfstd$ 
such that $\coreunionderived$ contains $0 \geq 1$, 
it holds that
$v$ is at least the value of~$\objfunc$ in any 
\fminimal solution
for~$\formf$
(or, for a decision problem, we have
$v = \infty$ only when $\formf$ is unsatisfiable).
If there are no applications of the \uncheckeddeletion rule, then
the final configuration is also \safe, and hence by
\cref{thm:inv_sound}
it holds that
$v$ is \emph{exactly} the value of~$\objfunc$ in any 
\fminimal solution
for~$\formf$
(or, for a decision problem, we have
$v = \infty$ if and only if $\formf$ is unsatisfiable).
As mentioned above,
in this latter case
the full sequence of
proof   
configurations~$\pconfstd$ together with annotations about the
derivation steps---including, in particular, any
witnesses~$\witness$---contains all information needed to efficiently
reconstruct such an \fminimal solution
for~$\formf$.
It is also straightforward to show that our proof system is complete:
after using the bound update rule to log an optimal
solution~$\bestobjval$,
it follows from the implicational completeness of 
cutting planes 
that contradiction can be derived from
$
\formf \union \set{\objfunc \leq \bestobjval - 1}
$.

Building on the ideas developed in this section, other variants of the proof system could be designed. 
For instance, we opted to define the relation $\better$ as 
\begin{equation}
  \assmntalpha\better\assmntbeta 
  \quad
  \text{if}
  \quad
  \assmntalpha\preceq\assmntbeta \text{ and }
  \subst{f}{\assmntalpha}\leq \subst{f}{\assmntbeta}
  \eqcomma
\end{equation}
but an alternative definition could be 
\begin{equation}
  \assmntalpha\better'\assmntbeta 
  \quad
  \text{if}
  \quad
  \subst{f}{\assmntalpha}\leq \subst{f}{\assmntbeta}
  \text{ and if }	\subst{f}{\assmntalpha}= \subst{f}{\assmntbeta} \text{ then also }	\assmntalpha\preceq\assmntbeta 
  \eqperiod
\end{equation}
What this 
says, essentially, 
is that we only compare assignments with respect to the
order~$\preceq$ in case they have the same objective value.  
Combining this with the intuition that 
a  
witness~$\witness$ for the \redruleshort rule should map each
assignment $\assmntalpha$ to an assignment that is at least as good
(in terms of $\better'$) as $\assmntalpha$  would here then yield
the conditions
\begin{subequations}
  \begin{alignat}{1}
    &\coreset \union \derivedset \union \set{\negc{\constrc}}
    \cpderives \subst
    {(\coreset \union \derivedset \union \constrc)}
    {\witness}
    \union 
    \objwitnessconditionset
    \text{ and}\\
    &\coreset \union \derivedset \union \set{\negc{\constrc}} \union \set{\subst{\objfunc}{\witness} = \objfunc }
    \cpderives 
    \porderformsub
  \end{alignat}
\end{subequations}
for the \redrulelong rule
instead of~\eqref{eq:substredundant}.  
It can be observed that this actually results in slightly weaker proof
obligations
on the right-hand side
and thus a more generally applicable rule.  
Similarly, for the dominance rule, the witness should map each
assignment $\assmntalpha$ to an assignment that is \emph{strictly
  better} than $\assmntalpha$ (again, in terms of $\better'$).  
This requirement would be encoded by the conditions
\begin{subequations}
  \begin{alignat}{1}
    & \coreset \union  
    \derivedset{} \union  
    \set{\negc{\constrc}} \cpderives \subst{\coreset}{\witness}
    \union 
    \objwitnessconditionset\eqcomma\label{eq:domalt:one}
    \\
    & \coreset \union  
    \derivedset{} \union  
    \set{\negc{\constrc}} \union \set{\subst{\objfunc}{\witness} = \objfunc }\cpderives 
    \porderformdom\eqcomma\text{ and}
    \label{eq:domalt:two}
    \\
    & \coreset
    \union  \derivedset
    \union   \set{\negc{\constrc}} \union \set{\subst{\objfunc}{\witness} = \objfunc } \union   \porderformdomneg
    \cpderives
    \bot
    \label{eq:domalt:three}
  \end{alignat}
\end{subequations}
instead of~\eqref{eq:dom:one} and~\eqref{eq:dom:two}. 
This again results in a slightly more generally applicable rule. 
All proofs of correctness go through with this slight modification of
the proof system, and in fact would go through for any alternative to
the order $\better$, as long as the \redruleshort and \domruleshort
rules are adapted accordingly.  
In this paper we opted for using $\better$, prioritizing simplicity of
exposition.

\renewcommand\lfalse{\m{0}}
\renewcommand\ltrue{\m{1}}
\renewcommand\perm{\m{\sigma}} 

\section{Symmetry Breaking in SAT Solvers}
\label{sec:symmetrybreaking} 

In this section we discuss the use of symmetry breaking in the context
of Boolean satisfiability (SAT) solving and how our proof system can
be used to provide efficiently verifiable proofs of correctness for
symmetry breaking constraints.
We also present results
from 
a thorough empirical evaluation.

\subsection{Certified Symmetry Breaking  with
  \DOMRULELONG
}
\label{sec:symbreaking-basic}

Symmetry handling has
a long and successful history in  SAT
solving, with a wide variety of techniques 
considered by, e.g.,
\citet{%
shatter,%
BS94Tractability,%
BNOS10Enhancing,%
DBDDM12Symmetry,%
DBB17SymmetricExplanationLearning,%
DBLP:conf/nfm/MetinBK19}, and by 
\citet{Sabharwal09SymChaff}. 
These techniques were used to great effect in, e.g., 
the 2013 and 2016 editions of the 
\emph{SAT competition},\footnote{\url{www.satcompetition.org}}
where the 
\emph{SAT+UNSAT hard  combinatorial track} 
and the 
\emph{no-limit} track, respectively,
were won by solvers employing symmetry breaking
techniques.  
However, it
later turned out that the victory in 2013 was partly
explained by a small parser bug in the symmetry breaking tool,
which could result in the solver taking a shortcut and declaring a
formula unsatisfiable before even starting to solve it.
For reasons such as this, proof logging is now obligatory in the main
track of the SAT competition. 
While it is hard to overemphasize the importance of this development,
and the value proof logging has brought to the SAT community,
it has unfortunately also meant that it has not been possible to use
symmetry breaking in the SAT competition, since there has been no way
of efficiently certifying the correctness of
such reasoning in \drat.
We will now explain how pseudo-Boolean reasoning with
the dominance rule  can provide proof logging for 
the \emph{static symmetry breaking}
techniques of~\citet{DBBD16ImprovedStatic}.

Let $\perm$ be a permutation of the set of literals in a given CNF
formula~$\formf$ (\ie a bijection on the set of literals), extended
to (sets of) clauses in the obvious way.
%
We say that~$\perm$
is a \emph{symmetry} of~$\formf$ if it 
commutes with negation,
\ie \mbox{$\perm (\olnot \litell)=\olnot {\perm (\litell)}$},
and preserves satisfaction of~$\formf$,
\ie 
$\compassmtafterassmt{\assignment}{\perm}$
satisfies $\formf$ if and only if $\assignment$~does.
A \emph{syntactic symmetry} 
in addition satisfies that
$
\perm(\formf) \synteq \restrict{\formf}{\perm} \synteq \formf$.
As is standard
in symmetry breaking,  
we
will   
only consider syntactic symmetries.

The most common way of breaking symmetries is by adding 
\emph{lex-leader constraints}~\cite{CGLR96Symmetry}.
We will write
$\lex$ to denote the lexicographic order on assignments induced by
the sequence of variables $x_1, \dots, x_m$, \ie  
$\assignment \lex \beta$ 
if $\assignment = \beta$ or 
if there is an~$i \leq m$ such that
$\assignment(x_j)=\beta(x_j)$
for all~$j < i$
and 
$\assignment(x_i) < \beta(x_i)$.
 Given a set~$\generatorset$ of symmetries of $\formf$,
a lex-leader constraint is a formula $\LLconstr$ such that
$\assignment$ satisfies $\LLconstr$ if and only if
$
\assignment \lex \assignment\circ\perm
$
for each $\perm \in G$.
The actual choice of a set $\generatorset$ of symmetries to
break,
as well as the
choice of variables on which to define the lexicographic order,
has 
a significant effect on the quality of
the breaking constraints 
\cite{DBBD16ImprovedStatic},
but this is an orthogonal concern to the goals of the current paper,
which is to show how to certify an encoding of lex-leader constraints
using the dominance rule.  


Let $\{x_{i_1},\ldots,x_{i_n}\}$ be
the \emph{support} of~%
$\perm $ 
(\ie
all variables $x$ such that $\perm (x)\neq
x$), ordered 
so that $i_j\leq i_k$ 
if and only if
$j\leq k$. 
Then the constraints 
\begin{subequations}
\begin{align}
 &y_0  \geq 1 & \label{eq:init}\\
 &\olnotwindex{ y}{j-1} + \olnotwindex{ x}{i_j} + \perm (x_{i_j})   \geq 1  
&&  1\leq j \leq n \label{eq:breaking}\\
 &\olnotwindex{ y}{j} +  y_{j-1}   \geq 1 
&&  1\leq j < n \label{eq:Ts1}\\
 &\olnotwindex{ y}{j} +  \olnot{ \perm (x_{i_j})} +  x_{i_j}  \geq 1 
&&  1\leq j <  n  \label{eq:opposite}\\
 &y_{j}  +  \olnotwindex{ y}{j-1}  +  \olnotwindex{ x}{i_j}  \geq 1  
&&  1\leq j < n\\ 
 &y_{j}  +  \olnotwindex{ y}{j-1}  +  \perm ( x_{i_j})  \geq 1  &&  1\leq j < n
\label{eq:TsLast} 
\end{align}
\end{subequations}
%
form
a lex-leader
constraint for $\perm $
(for fresh variables $y_j$).
The intuition 
behind this encoding 
is as follows:
the variable $y_j$ is true
precisely when
$\assignment$ and $\assignment\circ\perm$ 
are equal up to $x_{i_j}$
(so  $y_0$~is trivially true as claimed in~\eqref{eq:init}). 
The constraint~\eqref{eq:breaking} does the actual symmetry breaking: it
states that if
$\assignment$ and $\assignment\circ\perm$ 
are equal up to $x_{i_{j-1}}$, then 
$x_{i_j}\leq \perm (x_{i_j})$ must hold.
The constraints
\mbox{\refeq{eq:Ts1}--\refeq{eq:TsLast}}
encode the definition of 
$y_j$ as $y_{j-1} \land
(x_{i_j}\lequiv \perm (x_{i_j}))$.

To derive
the clausal constraints
\mbox{\eqref{eq:init}--\eqref{eq:TsLast}}
in our proof system,
assume that we have a configuration
$(\coreset, \derivedset, \pordernot, 
\vecx,
\objval)$
where assignments are compared lexicographically
on 
the subset of variables  
$\vecx = \set{\varx_1,\dots,\varx_m}$
according to~\orderConstraint as
defined   
in~\refeq{eq:encode-LL}.
%
%
Let $\perm$ be a syntactic symmetry of 
$\baseset$ (\ie such that 
\mbox{$\subst{\baseset}{\perm} \synteqnospace \baseset$})
with support contained in~$\vecx$.
As a first step, we use the \domrulelong rule  in
\cref{sec:dominiance-rule}
to derive
the pseudo-Boolean constraint
\begin{equation}
  \label{eq:CLL}
 \LLc 
 \synteq 
 \textstyle
 \sum_{i=1}^{m}
 2^{m- i} \cdot (
 \perm(\varx_i) 
 - \varx_i
 ) \geq 0
\end{equation}
expressing that $\perm(\vecx)$ is 
greater than or equal to~$\vecx$.
We emphasize that the constraint~$\LLc$ only exists in the proof and 
is nothing that the SAT solver will
see---since it only understands clauses---but this constraint will
help us to construct efficient derivations of the clausal constraints
that will be used by the solver to enforce symmetry breaking.

To see how the constraint~$\LLc$ in~\eqref{eq:CLL}
can be derived by the \domruleshort rule,
note first that in SAT solving we are dealing with
decision problems, and so there is no need to worry about the trivial
objective function
$\objfunc \synteq 0$.
Second, observe that we have
\mbox{$\order{\vecx}{\subst{\vecx}{\perm} }\synteq\{\LLc\}$},
According to \cref{def:dominance-rule},
we need to show that derivations as
in~\mbox{\refeq{eq:dom:one}--\refeq{eq:dom:two}}
exist.
To argue that \eqref{eq:dom:one} holds, 
we note that $\lnot \LLc$ expresses that $\vecx$ is 
strictly larger than $\perm(\vecx)$, and hence 
this 
implies $\pordernot(\subst{\vecx}{\perm},\vecx)$.
Since these are single pseudo-Boolean constraints,
this is easily verifiable by \syntimplication.
It is also easy to verify that \eqref{eq:dom:two}~is  true,
since we have both~$\LLc$ and its negation among the premises on
the left-hand side,
and for any pseudo-Boolean constraint~$\constrc$ it holds that adding
$\constrc$ and~$\neg \constrc$ together yields
$0 \geq 1$.

Once we have~$\LLc$ as in~\eqref{eq:CLL}, we can use this to obtain
the constraints
\mbox{\eqref{eq:init}--\eqref{eq:TsLast}}.
Since all \mbox{$y_j$-variables} are fresh, it is straightforward to
derive all constraints~\eqref{eq:init} and \eqref{eq:Ts1}--\eqref{eq:TsLast}
using \redrulelong as explained by~\citet{GN21CertifyingParity}.
It is important to note that in these derivations the witness
substitutions only operate on the new, fresh \mbox{$y_j$-variables}.
Hence, we have
$\subst{\vecx}{\witness} = \vecx$,
and so
$\pordernot(\subst{\vecx}{\witness},\vecx)$ holds by
the reflexivity of~$\pordernot$.
A more interesting challenge is to derive the 
constraints~\eqref{eq:breaking}, and this is where we need~\LLc.

Recall that our assumption is that
the support of $\perm$ is $
\{x_{i_1},\dots, x_{i_n}\}$ with $i_j\leq i_k$ 
if and only if $j\leq k$.  
Note first
that for all variables~$x_i$ that are not
in the support of \perm,
the difference
$ \perm(\varx_i) - \varx_i $
disappears since 
$\perm(x_i) = x_i$.
This means that the constraint~\LLc simplifies to
\begin{equation}
\textstyle
\sum_{j=1}^{n} 2^{m- i_j} \cdot (
 \perm(\varx_{i_j}) 
 - \varx_{i_j}
 ) \geq 0
\eqcomma
\end{equation}
and this inequality can only hold if the contributions of the
variables with the largest coefficient~$2^{m- i_1}$ is non-negative.
In other words, the constraint~$\LLc$
implies that
\mbox{$ \perm(\varx_{i_1})  - \varx_{i_1} \geq 0$},
and this implication can be verified by reverse unit propagation
(RUP).
From this, in turn, we can obtain the
constraint~\eqref{eq:breaking}
for $j=1$
by
\syntimplication.


\newcommand{\substineqref}[3]{(\ref{#1}\,[{#2}={#3}])}

To derive constraints~\eqref{eq:breaking}
for $j>1$, let us introduce the notation
\begin{subequations}
  \begin{align}
\LLc(0) 
  &\synteq   
    \LLc
  \\
  \LLc(k) 
  &\synteq   
    \LLc(k-1) +  2^{m- i_k} \cdot
 \substineqref{eq:opposite}{j}{k}
\end{align}
\end{subequations}
where 
$\substineqref{eq:opposite}{j}{k}$
denotes substitution of~$j$ by~$k$ in~\eqref{eq:opposite}.  
Simplifying $\LLc(k)$ yields
\begin{equation}
	\label{eq:simplified-C-LL-k}
	\textstyle
	\sum_{j=1}^{k} 2^{m- i_j}
        \cdot  
        \olnot{y}_j + \sum_{j=k+1}^{n} 2^{m- i_j} \cdot (
	\perm(\varx_{i_j}) 
	- \varx_{i_j}
	) \geq 0
	\eqcomma
\end{equation}
which,
when combined with
all constraints~\eqref{eq:Ts1},
directly entails
the constraint~%
\eqref{eq:breaking}
\mbox{with $j=k$.} 
To see this, note that if $y_k$ is false, then
\eqref{eq:breaking}~is trivially true for $j = k + 1$.
On the other hand, if $y_k$ is true, then so
are all the preceding 
\mbox{$y_j$-variables},
and the dominant
contribution in~$\LLc(k)$ is from
$\perm(\varx_{i_k})  - \varx_{i_k}$,
which implies~$\eqref{eq:breaking}$ for $j= k$
analogously to the case  for $j= 1$.


It is important to note here that the order used for the \domrulelong
is fixed  at the beginning and remains the
same for all symmetries $\perm\in G$ to be broken.  
Since constraints are added only to
the derived set~
$\derivedset$, 
dominance rule applications for different symmetries will not
interfere with each other.
Furthermore,
in contrast to the approach 
of  \citet{HHW15SymmetryBreaking},
handling a symmetry once is enough to
guarantee complete breaking.
In \cref{sec:sym-appendix} we include a complete worked-out example of
symmetry breaking in \veripb syntax together with explanations of how
the proof logging 
syntax matches rules in our proof system.

\subsection{Relation to \drat-Style Symmetry Breaking}


As discussed in the introduction, it is currently not known whether
general symmetry breaking can be certified efficiently using \drat
proof logging.
Previous work on symmetry breaking with \drat
\cite{HHW15SymmetryBreaking}
is limited to special cases of simple symmetries.
%
%
To compare and contrast this with our work, let us describe the 
method of
\citet{HHW15SymmetryBreaking}
in our language
(which is easily done, since the \rat rule is a special case of our
\redrulelong rule),
and explain the difficulties in extending this to a
general symmetry breaking method.

%
%

\Redrulelong can easily deal with a single simple
symmetry~$\sigma$ that
swaps two variables, say $x_i$ and $x_j$ (with $j>i$). In this case,
the constraint~$\LLc$ in~\eqref{eq:CLL} simplifies to 
\begin{equation}
  x_j - x_i \geq 0
  \eqcomma
\end{equation}
which can be derived with a single application of the
\redruleshort rule with the symmetry~$\sigma$ serving as the witness.
However, this approach no longer works when several symmetries need to
be broken at the same time, or when we need to deal with more complex
symmetries.

When multiple symmetries that swap two variables are
involved, the second symmetry cannot necessarily be broken in the way
described above, since the witness for the second symmetry might
invalidate the symmetry breaking constraint for the first symmetry
(which has been added to the set of \coreconstraint{}s, and so
will appear as one of the proof obligations on the right-hand side  
in~\eqref{eq:substredundant}). 
If there are two symmetries $\sigma_1$ and~$\sigma_2$ that share a
variable, then these symmetries might need to be applied three times
in order to  obtain a lexicographically minimal assignment.
While this makes the proof logging more complicated, it is still
possible to deal with this scenario in \drat proof logging
by making use of a  sorting network  encoding as explained by
\citet{HHW15SymmetryBreaking}.

A more serious problem is when the symmetries to be broken are more
complex than just swaps of a single pair of variables. 
For the case of \emph{involutions}~$\sigma$ (\ie
where $\sigma$ is its own inverse),
some steps towards a solution have been taken
\cite[Conjecture~1]{HHW15SymmetryBreaking}.
But already a symmetry~$\sigma$ that is a cyclic shift of three
variables (\ie such that
$\sigma(x)=y$, $\sigma(y)= z$, and $\sigma(z)=x$)
brings us beyond what \drat-based proof logging symmetry breaking is
currently able to handle. 
The obstacle that arises here is that we do not know in advance
whether the permutation~$\sigma$ should be applied once or twice to an
assignment~$\assmntalpha$ to get the lexicographically smallest
assignment in the orbit of~$\assignment$ under~$\sigma$.

The beauty of the \domrulelong rule is that it completely eliminates
these problems. There is no requirement that our witness substitutions
should
generate minimal assignments---all that is needed is that the
witnesses yield smaller assignments.
And since the symmetry breaking constraints are all added to the
derived set~$\derivedset$,
we do not need to worry about what previous symmetry breaking
constraints might have been added when we are breaking the next
symmetry. Instead, symmetry breaking constraints for different
symmetries can be added independently of one another (for as long as
the order remains unchanged).

\subsection{%
  Extensions of
  Basic Symmetry Breaking}

%
%

So far we have discussed the core ideas that underlie most modern
symmetry breaking tools for SAT solving.  The tool \breakid
\mycite{breakid} extends these ideas further in a couple of ways.  We
now briefly discuss these extensions and how they are dealt with in
our proof system.

The most 
important   
contribution of \mycitet{breakid} is detecting so-called \emph{row interchangeability}. 
The goal of this optimization is to not just take an arbitrary set of
generators of the symmetry group and an arbitrary lexicographic
order, but to choose ``the right'' set of generators and ``the
right'' variable order 
(with
respect to 
which to define the lexicographic order).
\mycitet{breakid}
showed that for groups that exhibit a certain structure, breaking
symmetries of a good set of
generators
using an appropriate  order
can guarantee that the entire symmetry
group
is broken completely.  
Since our
proof  
logging techniques simply use the same lexicographic order
as the
symmetry   
breaking tool, and work for an arbitrary generator set, this 
automatically
works with the techniques described above.

Another (optional) modification to the basic symmetry breaking
techniques implemented in  \breakid is the use of a more
compact  encoding,
where the clauses~\eqref{eq:Ts1} and~\eqref{eq:opposite}
are omitted. 
Since our definition of $\LLc(k)$ uses these clauses, we cannot simply
omit them in our proof.
However, there are no restrictions on deletions from the derived
set~$\derivedset$.
Therefore, we can first derive all the symmetry breaking constraints
as described above, and then remove the superfluous clauses
from~$\derivedset$ as soon as
they are no longer needed for the proof logging derivations.

Next, \breakid has an optimization based on \emph{stabilizer subgroups}
to detect a large number of binary clauses.  Since these binary clauses
are all clauses of the form \eqref{eq:breaking} with $j=1$,
the proof logging techniques described above could support also 
this optimization
provided
that we keep track of which symmetry is used for each such binary 
clause. However, \breakid
currently does no such bookkeeping.   
While it is in principle possible to add this feature, 
we have not done so in this work.

Finally, \breakid supports \emph{partial symmetry breaking}.  
That is, instead of adding the constraints
\eqref{eq:breaking}--\eqref{eq:TsLast} for 
every~$j$, 
this is
only done for $j<L$ with $L$ a limit that can be chosen by the user.  
The reasoning behind this is that the larger $j$ gets, the weaker the
added breaking constraint is. 
By setting $L = 100$ and adding symmetry breaking clauses only
for the  $L$ first variables,
we add significantly fewer constraints while retaining most of the
symmetry breaking power.

Since we only need to do proof logging for the clauses that are
actually added by \breakid, this optimization works out-of-the-box. 
However, there is an important caveat here: 
for benchmark formulas
where there 
are 
large symmetries, e.g., symmetries permuting all the variables in the
problem,
a naive implementation of our proof logging technique will suffer from
serious performance problems even when this optimisation is used.
The reason is that in principle the order~$\pordernot$ is defined on
all variables that are permuted by the symmetries. If there are many
such variables, this order in itself can get huge
(with the largest coefficient being of exponential magnitude measured in
terms of the number of variables).
Luckily, there is a simple solution to this problem: instead of
defining the order on all variables that are permuted, we can use   
the set of variables on which we will actually do
the symmetry   
breaking
(which for each symmetry are the first $L$ variables in its support).
This is the approach implemented in our experimental evaluation.

\subsection{Experimental Evaluation of SAT Symmetry Breaking}

\begin{figure*}[t]
  \centering
  \includegraphics[width=0.8\textwidth]{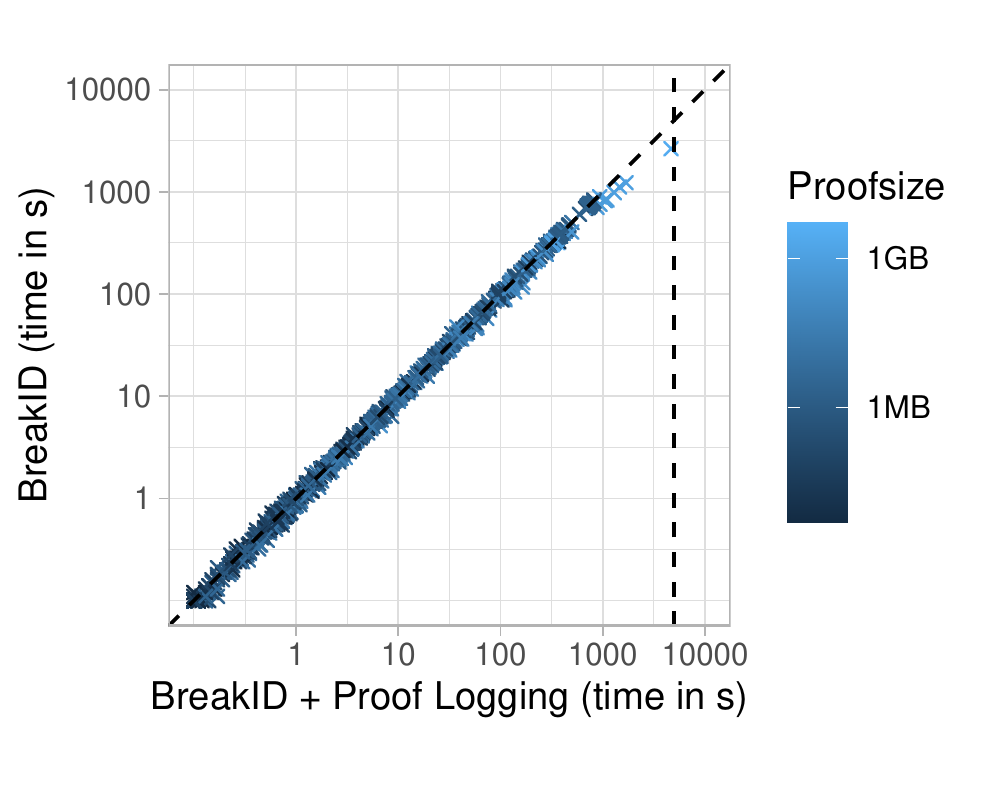}
  %
  \caption{
    Time required for symmetry breaking with and without proof logging.
}
  \label{fig:plot1}
\end{figure*}

\begin{figure*}[t]
  \centering
  \includegraphics[width=0.8\textwidth]{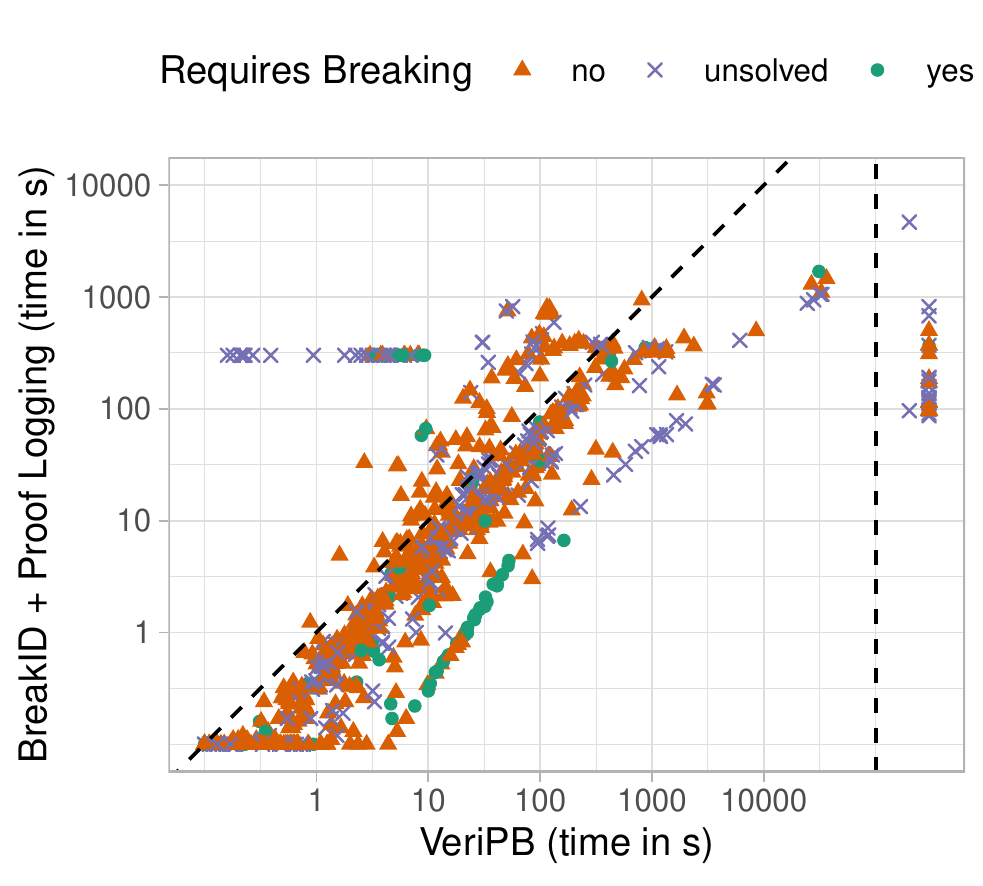}
  \caption{%
    Time for symmetry breaking with proof logging compared
    to verification time for the generated symmetry breaking constraints.
    Points
    to the right of
    the vertical dashed line indicate timeouts (left) and
    out of memory (right).}\label{fig:plot2}
\end{figure*}

\begin{figure*}[t]
  \centering
  \includegraphics[width=0.8\textwidth]{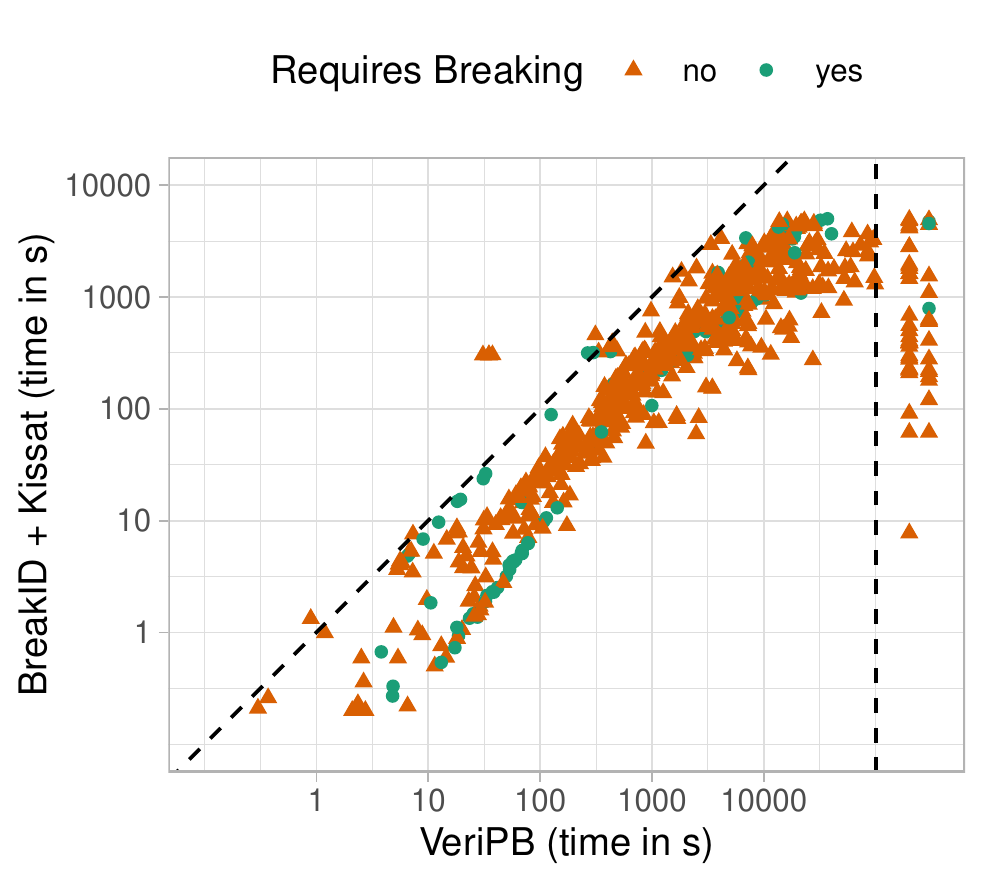}
  \caption{%
%
    Time for symmetry breaking plus SAT solving (with proof logging) compared
    to verification time for the whole solving process.
    Points
    to the right of the vertical dashed line indicate timeouts (left)
    and out of memory (right).}\label{fig:plot3}
\end{figure*}

%

To validate our approach,
we implemented pseudo-Boolean proof logging in the \veripb proof format for the
symmetry breaking method in \breakid, and modified
\kissat\footnote{\url{http://fmv.jku.at/kissat/}} to output
\veripb-proofs (since the \redruleshort rule is a generalization of
the RAT rule, this required only
purely syntactic
changes).
We
employed
the \uncheckeddeletion rule
in \cref{def:unchecked_del}   
since the generated proofs
are only used to 
certify  unsatisfiability of decision problems
and never to prove satisfiability.     

Among
the benchmark instances from all the SAT
competitions
between the years 2016 and 2020,
we selected all instances in which at least
one symmetry was detected; there were 1089 such instances in total. We
performed our
computational
experiments
on machines with dual Intel Xeon E5-2697A v4
processors with 512GBytes RAM and solid-state
drives   
(SSD)
running on the Ubuntu 20.04 operating system.
We ran twenty instances in parallel on
each machine,
where we limited each instance to 16GBytes RAM. We used a 
timeout of 5,000 seconds for solving and 100,000 seconds for
verification of the proofs produced during the solving process.


\cref{fig:plot1}
shows   
the performance overhead
for symmetry breaking, comparing for each instance the 
running time
with
and without proof logging. For most instances, the overhead is
negligible (99\% of instances are at most 32\% slower). 
Figures~\ref{fig:plot2} and~\ref{fig:plot3} display the relationship between the time
needed to generate a proof (both for SAT and UNSAT instances) and 
to verify the  correctness of  this   
proof. When
considering symmetry breaking in isolation (i.e., not solving the instance completely, but only breaking the symmetries), as plotted in \cref{fig:plot2}, 
 $1058$ instances out of $1089$ could be
verified, $2$ timed out, and $29$ terminated due to running out of memory.
$75\%$ of the instances could be verified within $3.2$ times the
time required for symmetry breaking 
and $95\%$ 
within a factor~$20$. 
The time needed for
verification is thus considerably longer than 
the time required to generate the proofs, 
but still
practical in the majority of cases. After symmetry breaking, 
$721$~instances could be solved with the SAT solver (\cref{fig:plot3}) and we
could verify $671$ instances, while for
$33$~instances
verification timed out and for 
$17$~instances
the verifier 
ran   
out of memory.
Notably, $84$~instances could only be solved by the SAT solver when
symmetry breaking clauses were added, and out of these instances we
could verify correctness for $81$~instances.

\providecommand{\filename}[1]{\path{#1}\xspace}

\providecommand{\alldiffname}{all-different\xspace}

\section{Symmetries in Constraint Programming}\label{sec:cpProblems}
In the general setting considered in
constraint programming, we must deal with variables with larger
(non-Boolean) domains
and 
with   
rich constraints supported by propagation
algorithms. One might think that a proof system based upon Boolean
variables and linear inequalities 
would
not be suitable for this
larger class of problem. However, \citet{EGMN20Justifying} showed how
to use \veripb for constraint satisfaction problems by first encoding
variables and constraints in pseudo-Boolean form, and then 
constructing
cutting planes 
proofs   
to certify the behaviour
of the \emph{\alldiffname{}} propagator,
and \citet{GMN22AuditableCP} later  extended this to a wider range
of CP constraints and propagators.
Similarly, the work we present here can also be
applied to constraint satisfaction and optimisation problems.

Recall the list of symmetry breaking constraints proposed for  the Crystal Maze puzzle 
in the introductory section.
Given the difficulties in knowing which combinations of
symmetry breaking   
constraints are valid, it would be desirable if these constraints could be
\emph{introduced as part of a proof}, rather than taken
as part of the encoding of the original problem.
This would
give a modeller immediate feedback as to whether 
the 
constraints have been
chosen correctly.
Our
cutting planes proof system enhanced with
\redrulelong and \domrulelong rules
is indeed powerful enough to express 
all three of the examples we
presented in the introduction,
and we 
have   
implemented a small tool which can write out the appropriate
proof fragments certifying that this is so;
this allows the entire Crystal Maze example to be
verified
with \veripb.

The source code to run
our tool for the Crystal Maze puzzle
is located in the
\filename{tools/crystal-maze-solver}
directory of the code and data repository~\cite{BGMN22DominanceData}
associated to this
paper.   
Full instructions for how to use the tool are given in the 
\filename{tools/crystal-maze-solver/README.md}
file.
Below we provide a summary of our work on this problem.

We modelled the Crystal Maze puzzle as a constraint satisfaction problem in the
natural way:
there is
a decision variable for each circle, whose values are the
possible numbers that can be taken, and an \alldiffname constraint over all
decision variables. We use a table constraint for each edge
for simplicity.
We also included symmetry elimination constraints.

We implemented this model 
inside a small
proof-of-concept CP solver
(which can be found in the file  \filename{src/crystal_maze.cc})
that we created for this paper. (Full proof logging
for CP is 
an entire research program in its own right,
which we do not at all claim to have carried out within the framework
of this project---what
we do claim, though, is that
our contribution
shows that symmetries do not stand in the way of this work.) When executed, the
solver compiles this high level CP model to a pseudo-Boolean model, which it
will output as 
\filename{crystal_maze.opb}. 
This is done following the framework
introduced by \citet{EGMN20Justifying}, but as well as using a one-hot (direct)
encoding of CP decision variables to
pseudo-Boolean   
variables, it additionally creates
channelled greater-or-equal PB variables for each CP variable-value.
We remark that 
the encoding of the table constraints also introduces additional auxiliary
variables.

As the solver works on the problem, it produces a file
\filename{crystal_maze.veripb} that
provides a proof that it has found all non-symmetric solutions.
The proof log will contain reverse unit
propagation (RUP) clauses that justify backtracking,
as well as cutting planes derivations that prove the correctness of propagations by
the  \alldiffname and table constraints.
Once the solver has finished, the correctness of the solver output can
be checked by feeding the two files
\filename{crystal_maze.opb}
and
\filename{crystal_maze.veripb} 
to \veripb.

To verify that the symmetry constraints introduced in the
high-level  
model are
actually valid, we can remove them from the pseudo-Boolean model
in \filename{crystal_maze.opb}
and introduce
them as part of the proof instead.  We describe how to do this editing in 
\mbox{\filename{README.md}}. 
We also include a 
script \filename{make-symmetries.py} that
will output the
necessary proof fragment to reintroduce the symmetry constraints. 
The output of this script
can be verified on top of the reduced pseudo-Boolean model using
our  modified version of \veripb supporting the \domrulelong rule,
and this verification can be performed
with or without the remainder of the proof---that is, we can
verify both
that the constraints introduced are valid (in that they do not
alter the satisfiability of the model), and that they line up with the 
actual execution of the solver.

Interestingly, although symmetries can be broken in different ways in
high-level
constraint programming   
models (including through lexicographic and value
precedence constraints), when we encode the problem in pseudo-Boolean
form these differences largely disappear, 
and after creating a suitable order we can 
easily  
re-use the SAT proof logging techniques  for symmetry breaking that we
discussed above.
So, although a full proof-logging constraint 
solver does
not yet exist, we can confidently claim that
symmetry breaking should not block the road towards this goal.



\section{Lazy Global Domination in
  Maximum Clique
  Solving   
}
\label{sec:max-clique}

\citet{GMMNPT20CertifyingSolvers} showed how \veripb can be used to
implement proof logging for a wide range of maximum clique algorithms,
observing that the cutting planes proof system is rich enough to
certify a
rich variety
of bound and inference functions used by various
solvers (despite cutting planes not knowing what a graph or clique is).
However, there is one clique-solving technique in the literature that
is \emph{not} amenable to cutting planes reasoning.
In order to solve problem instances that arise
from a distance-relaxed clique-finding problem,
\citet{MP16Finding} enhanced their maximum clique algorithm with a
\emph{lazy global domination} rule that works as follows.
Suppose that the solver has constructed a candidate clique~$\cliquec$ 
and is considering to extend~$\cliquec$ by two vertices
$v$ and~$w$, where the neighbourhood of
$v$ excluding $w$ is a (non-strict) superset of the neighbourhood of
$w$ excluding $v$. Then if the solver first tries~$v$ and rejects it, 
there is no need to branch on~$w$ as well
(since any clique including~$w$ could be exchanged for at least as
good a clique including~$v$).  

In principle, it should be possible to introduce additional
constraints certifying this kind of reasoning 
in advance using \redrulelong,
without the need for the full  dominance breaking framework in
\refsec{sec:proof-system} (with some technicalities involving
consistent orderings for tiebreaking).
However, due to the prohibitive cost of computing the full vertex
dominance relation in advance,
\citeauthor{MP16Finding} instead implement a form of \emph{lazy}
dominance detection, which only triggers following a backtrack.
To provide proof logging for this,
we cannot derive the constraints justifying global domination steps in the proof log in
advance, but must instead be able to
introduce vertex dominance constraints 
on the fly  
precisely when they
are used. It is hard to see how to achieve this with the
\redruleshort rule, but it is possible using \domrulelong. 
In what follows, we first present a pseudo-Boolean encoding of the maximum clique problem and then a high-level description of the max clique solver used.
Afterwards, we discuss how to add certification to it with the \redruleshort rule and the \domruleshort rule.


\providecommand{\vdomrel}{\succ_{G}}
\providecommand{\vneighbour}[1]{N(#1)}

\providecommand{\algnameformat}[1]{\mathsf{#1}}
\providecommand{\algsuffixformat}[1]{\mathrm{#1}}

\providecommand{\ltvert}[1]{\mathit{LT}(#1)}
\providecommand{\gtvert}[1]{\mathit{GT}(#1)}

\providecommand{\cliquesearch}{\algnameformat{MaxCliqueSearch}}
\providecommand{\cliquecurr}{C_{\algsuffixformat{curr}}}
\providecommand{\cliquebest}{C_{\algsuffixformat{best}}}
\providecommand{\vrem}{V_{\algsuffixformat{rem}}}
\providecommand{\vnew}{V'}
\providecommand{\erem}{E_{\algsuffixformat{rem}}}
\providecommand{\grem}{G_{\algsuffixformat{rem}}}
\providecommand{\colclass}{S}

\SetKw{KwDownTo}{downto}
\SetKw{KwAnd}{and}


\subsection{The Maximum Clique Problem}
Throughout this 
   section 
we let
$G = (V,E)$ 
denote  an undirected, unweighted graph 
without self-loops with vertices $V$ and edges $E$. 
We write
$\vneighbour{u}$
to denote the \emph{neighbours} of a vertex~$u \in V$,
\ie the set of vertices
$\vneighbour{u} = \mbox{$\setdescr{w}{(u,w) \in E}$}$
that are adjacent to~$u$ in the graph,
and define neighbours of sets of vertices in the natural way by
taking unions
$
\vneighbour{U} = \Union_{u \in U} \vneighbour{u}$.
We say that
$u$ \emph{dominates}~$v$ if
\begin{equation}
	\label{eq:vertex-dominance}
	\vneighbour{u} \setminus \set{v}
	\supseteq
	\vneighbour{v} \setminus \set{u}
\end{equation}
holds, which intuitively says that the neighbourhood of~$u$ is at
least as large as that of~$v$. It is straightforward to verify that
this domination relation is transitive.

When representing the maximum clique problem
in pseudo-Boolean form, 
we overload notation and identify every 
vertex~$v \in V$
with a Boolean variable, where $v=1$ means that 
the vertex~$v$ is in the clique.
The task is to maximize
$\sum_{v \in V} v$
under the constraints that all chosen vertices should be neighbours,
but since, 
syntactically
speaking, we require an objective function to be
minimized, we obtain
\begin{subequations}
	\begin{align}
		\label{eq:clique-obj}
		&\min \textstyle \sum_{v \in V} \olnot{v} 
		\\
		\label{eq:clique-edge-constraint}
		&
		\olnot{v} + \olnot{w} \geq 1
		&&
		[\text{for all $(v, w) \in V^2\setminus E$ with $v\neq w$}] 
	\end{align}
\end{subequations}
as the formal pseudo-Boolean specification of the problem.

\subsection{High-Level Description of the Max Clique Solver}


      \begin{algorithm*}[t] 
        \caption{Max clique algorithm without dominance.}
        \label{alg:max-clique}
        $\cliquesearch(G, \vrem, \cliquecurr, \cliquebest):$ 
        \\
        $\erem \gets E(G) \intersection (\vrem \times \vrem)$\;
        $\grem \gets (\vrem, \erem)$\;
        \If{$\setsize{\cliquecurr} > \setsize{\cliquebest}$}
        {
          $\cliquebest \gets \cliquecurr$\;  \label{line:best-clique-update}
        }
        $(\colclass_1, \ldots, \colclass_m) \gets$
        colour classes in colouring of $\grem$ \label{line:coloring}\; 
        $j \gets m$\;
        \While{$j \geq 1$ \KwAnd 
          $\setsize{\cliquecurr}+ j > \setsize{\cliquebest}$ \label{line:loop}}
          {
            \For{$v \in \colclass_j$ \label{line:innermost-loop}}
            {
              $\cliquebest \gets \cliquesearch(
              G, 
              \vrem \intersection \vneighbour{v},
              \cliquecurr \union \set{v}, \cliquebest)$\;
            }
            $\vrem \gets \vrem \setminus \colclass_j$\;
            $j \gets j - 1$\;
          }
          \mbox{ } 
          \Return $\cliquebest$\;
      \end{algorithm*}

%



At a high level, the maximum clique solver
of~\citet{MP16Finding}
\emph{before}
addition of vertex dominance breaking, works as described in 
Algorithm~\ref{alg:max-clique}. The first call to the 
$\cliquesearch$
algorithm is with parameters~$G$,
$\vrem = V$, 
\mbox{$\cliquecurr=\emptyset$}, 
and \mbox{$\cliquebest=\emptyset$}.

When $\cliquesearch$ is called with a candidate clique~$\cliquecurr$,
the best solution so far~$\cliquebest$, and a subset of
vertices~$\vrem$, it considers the residual graph
$\grem = (\vrem, \erem)$ assumed to be defined on all vertices
in \mbox{$V \setminus \cliquecurr$} that are neighbours of all
$c \in \cliquecurr$.
Thus,    
the set~$\vrem$  contains all vertices to
which~$\cliquecurr$ could possibly be extended.
The algorithm produces a colouring of~$\grem$
(where as usual adjacent vertices are assigned a different colour),
which we assume results in $m$~disjoint colour classes
$(\colclass_1, \ldots, \colclass_m)$
such that
$\vrem = \Union_{i=1}^{m} \colclass_i$. 
It is clear that any clique extending~$\cliquecurr$
can contain at most one vertex from
each   
colour class~$\colclass_i$,
since vertices of the same colour class are non-adjacent.
The clique search algorithm now iterates over all colour classes in
the order
$\colclass_{m}, \colclass_{m-1}, \ldots, \colclass_{1}$.
Whenever the clique is extended with a new vertex, a new recursive
call to $\cliquesearch$ is made.
Therefore, when we reach~$\colclass_j$ in the loop,
we are considering the case when all
vertices in
$\colclass_{m}, \colclass_{m-1}, \ldots, \colclass_{j+1}$
have been rejected.
For this reason, if the condition
\mbox{$\setsize{\cliquecurr} + j > \setsize{\cliquebest}$}
fails to hold, we know that the current clique
candidate cannot possibly be extended to a clique that is larger than
what we have already found in~$\cliquebest$.
At the end of the first call 
$\cliquesearch(G, V, \cliquecurr=\emptyset, \cliquebest=\emptyset)$,
after completion of all recursive subcalls,
the vertex set
$\cliquebest$
will be a clique of maximum size in~$G$.
A certifying version of essentially this algorithm with \veripb proof
logging was presented by
\citet{GMMNPT20CertifyingSolvers}.
It might be worth noting in this context that one quite
interesting challenge is to certify the backtracking
performed when the condition
$\setsize{\cliquecurr}+ j > \setsize{\cliquebest}$
fails, and this is one place where the strength of the 
pseudo-Boolean reasoning in the cutting planes
proof system is very helpful
(as opposed to the clausal reasoning in, e.g.,~\drat,
where certifying this type of counting arguments is quite
challenging).

The vertex dominance breaking of~\citet{MP16Finding}
is based on the following observation:
If the algorithm is about to consider
$v \in \colclass_j$ in the innermost for loop
on line~\ref{line:innermost-loop}
in Algorithm~\ref{alg:max-clique},
but has previously considered a  vertex
$u \in \Union_{i=j}^{m} \colclass_i$ 
that dominates~$v$
in the sense of~\refeq{eq:vertex-dominance},
then it is safe to ignore~$v$.
This is so since if
the algorithm would find a solution that includes~$v$ but not~$u$,
then we
could   
swap~$u$ for~$v$ and obtain a solution that is at least as good.

In pseudo-Boolean notation,
this  type of reasoning could be enforced by 
adding the constraint
$u + \olnot{v} \geq 1$
to the formula,
but there is no way this can be semantically derived from
the constraints~\refeq{eq:clique-obj}
or the requirement to minimize~\refeq{eq:clique-edge-constraint}.
Therefore, the proof logging method in
\cite{GMMNPT20CertifyingSolvers}
is inherently unable to deal with such constraints.

In general, the vertex dominance breaking as described above does not need to break ties consistently.
By this we mean that if $u$ and $v$ dominate each other, in principle it might happen that in a given branch of the search tree, $u$ is chosen to dominate $v$, while in another one, $v$ is chosen to dominate $u$, simply because of the order in which nodes are considered. 
While in principle, 
it should be possible to adapt our proof logging methods 
to work in this case, the argument is subtle. 
Luckily, it turns out that in practical implementations, tie breaking only happens in a consistent manner. 
\begin{fact}\label{fact:constbreak}
	In the  vertex dominance breaking of~\citet{MP16Finding}, there exists a total order $\vdomrel$ on the set $V$ of vertices such that whenever $v$ is ignored because $u$ has previously been considered, $u \vdomrel v$. 
\end{fact}
Moreover, this order $\vdomrel$ is known before the algorithm starts: $u \vdomrel v$ holds if $u$ has a larger degree than $v$, or in case they have the same degree but the identifier used to represent $u$ internally is larger than that of $v$. 
To see that this order indeed guarantees consistent tie breaking, we
provide some properties of the actual implementation of the algorithm.%
\footnote{When inspecting the two conditions below, we can see that
  they rely on a very subtle argument, namely that when constructing a
  greedy colouring, we should iterate over all nodes with the same
  degree in the \emph{opposite} order of the way we iterate over nodes
  in line~\ref{line:loop}. This (strange)
  condition
  is satsisfied by many
  max clique algorithms due to a happy coincidence of an efficient
  data structure and colouring algorithm introduced by
  \citet{TomitaK07}.} 
\begin{enumerate}
\item
  If $u$ and $v$ dominate each other and are not adjacent,
  then   
  $u$ and $v$
  are guaranteed to be in the same colouring class.  
  If furthermore $u\vdomrel v$, $u$ is considered before $v$ in the
  loop in Line~\ref{line:loop} (due to the order in which this for
  loop iterates over the nodes).
  
\item
  If $u$ and $v$ dominate each other, \emph{are} adjacent, and
  satisfy  
  $u\vdomrel v$, then $u$ is assigned a \emph{larger} colouring class
  than $v$ (due to the order in which the
  greedy
  colouring algorithm
  in Line \ref{line:coloring} iterates over the nodes). Hence, also in
  this case $u$ will be considered before~$v$.      
\end{enumerate}
%
%
In what follows below, we will explain: 
\begin{itemize}
\item 
  first, how the redundance rule introduced to \veripb by
  \citet{GN21CertifyingParity} could in principle be used to provide
  proof logging for vertex dominance breaking, although with
  potentially impractical
  overhead;  
  and
\item 
  then, how the dominance rule introduced in this paper 
  can be used to resolve the practical problems in a very simple way.
\end{itemize}
An implementation for both techniques can be found in the
code and data  repository~\cite{BGMN22DominanceData}.

\subsection{Vertex Dominance with the
  \REDRULELONG
  Rule}

In order to provide proof logging for vertex dominance breaking using
the redundance rule, we could in theory proceed as
follows.
First, we let the solver check the vertex dominance 
condition~\refeq{eq:vertex-dominance} for all pairs of vertices~$u, v$
in~$V$. 
Before starting the solver, 
we add all pseudo-Boolean constraints for
vertex dominance breaking  using the redundance rule.
For all $u,v$ such that $u$ dominates~$v$ and $u \vdomrel v$, we 
derive the
\emph{vertex dominance breaking constraint}
\begin{equation}
	\label{eq:vdom-constraint}
	u +   \olnot{v} \geq 1
	\eqcomma
\end{equation}
doing so \emph{in decreasing order} 
for~$u$   
\wrt~$\vdomrel$.
Our witness for the redundance rule derivation
of~\refeq{eq:vdom-constraint} will be 
$
\witness = 
\set{
	u \mapsto v,
	v \mapsto u
}
$,
\ie $\witness$ will simply swap the dominating and dominated vertices.
Hence, the objective function~\refeq{eq:clique-obj}
is syntactically unchanged after substitution  by~$\witness$,
and so the condition in
\refeq{eq:substredundant}
that the objective should not increase is always vacuously satisfied.

We need to argue that deriving the
vertex dominance breaking constraints~\refeq{eq:vdom-constraint} 
is valid in our proof system.
Towards this end, suppose we are in the middle of the process of adding such
constraints and are currently considering $u +   \olnot{v} \geq 1$
for
$u$ dominating~$v$ and $u \vdomrel v$.
Let  $\coreunionderived$ be the set of constraints in the current
configuration. 
In order to add
$u +   \olnot{v} \geq 1$,
we need to show that
\begin{subequations}  
	\begin{equation}
		\label{eq:sr-startingpoint}
		\coreunionderived \union
		\set{\negcwp{ u +   \olnot{v} \geq 1}}  
	\end{equation}
	can be used to derive all constraints in
	\begin{equation}
		\label{eq:sr-target}
		\restrict{(\coreunionderived \union
			\set{ u +   \olnot{v} \geq 1})}
		{\witness}
	\end{equation}
\end{subequations}
by the cutting planes method (\ie without any extension rules).

Starting with the vertex dominance constraint being added, 
note that from the negated constraint 
$
\negcwp{u +   \olnot{v} \geq 1}
\synteq
\olnot{u} + v \geq 2
$
in~\refeq{eq:sr-startingpoint}
we immediately obtain
\begin{subequations}
	\begin{align}
		\label{eq:prop-u}
		\olnot{u} &\geq 1 \\
		\label{eq:prop-v}
		v &\geq 1  
	\end{align}
\end{subequations}
as reverse unit propagation (RUP) constraints, 
meaning that the weaker
pseudo-Boolean
constraint
%
$
\restrict{( u +   \olnot{v} \geq 1)}  
{\witness}
\synteq
v + \olnot{u} \geq 1
$
is also RUP \wrt the constraints in~\refeq{eq:sr-startingpoint}.

Consider next any non-edge constraints
$\olnot{x} + \olnot{y} \geq 1$
in~\refeq{eq:clique-edge-constraint}
in the original formula.    
Clearly, such constraints are only affected by~$\witness$ if
$\set{u,v} \intersection \set{x,y} \neq \emptyset$;
otherwise they are present in both~\refeq{eq:sr-startingpoint}
and~\refeq{eq:sr-target} and there is nothing to prove.
Any non-edge constraint 
$\olnot{v} + \olnot{y} \geq 1$
containing~$\olnot{v}$ will after application
of~$\witness$ contain~$\olnot{u}$,
and will hence be 
RUP \wrt~\refeq{eq:prop-u} and hence also \wrt~\refeq{eq:sr-startingpoint}.
For non-edge constraints
$\olnot{u} + \olnot{y} \geq 1$
with
$y \neq v$,
substitution by~$\witness$ yields
$
\restrict{(\olnot{u} + \olnot{y} \geq 1)}{\witness}
\synteq
\olnot{v} + \olnot{y} \geq 1
$.
Since by assumption
$u$ dominates~$v$ and $y\neq v$ is not a neighbour of~$u$,
it follows from~\refeq{eq:vertex-dominance}
that $y$ is not a neighbour of~$v$ either. Hence,
the input formula 
in~\refeq{eq:sr-startingpoint}
already contains the desired non-edge constraint
$\olnot{v} + \olnot{y} \geq 1$.

It remains to analyse what happens to vertex dominance breaking
constraints
\begin{equation}
	\label{eq:previous-vdom-constraint}
	x + \olnot{y} \geq 1
\end{equation}
that have already been added to~$\derivedset$
before the dominance breaking constraint
$u + \olnot{v} \geq 1$
that we are considering now.
Again, such a constraint
is only
affected by~$\witness$ if
$
\set{u,v} \intersection \set{x,y} \neq \emptyset
$;
otherwise it is present in both~\refeq{eq:sr-startingpoint}
and~\refeq{eq:sr-target}.
We obtain the following case analysis.
\begin{enumerate}
	\item 
	$x=u$:
	In this case,
	$
	\restrict{(x + \olnot{y} \geq 1)}{\witness}
	\synteq
	v + \olnot{\witness(y)} \geq 1
	$,
	which is 
	RUP \wrt
	$v \geq 1$
	in~\refeq{eq:prop-v}
	and hence also \wrt~\refeq{eq:sr-startingpoint}.
	
	\item 
	$x=v$:
	This is impossible, since
	$u \vdomrel v$
	and any dominance breaking constraints 
	with $v=x$ as the dominating vertex will be added only once we are
	done with~$u$ as per the description right below~\refeq{eq:vdom-constraint}.

	\item 
	$y=u$:
	In this case
	$x$ dominates $u$. Since $x \vdomrel u$,
	$u \vdomrel v$, 
	and $u$ dominates~$v$,
	by transitivity we have 
	$x \vdomrel v$ and also that $x$ dominates~$v$. Hence, the
	breaking constraint $x + \olnot{v} \geq 1$ has already been added
	to~$\derivedset$.
	But since $u \neq x \neq v$, we see that our desired constraint is
	$
	\restrict{(x + \olnot{u} \geq 1)}{\witness}
	\synteq
	x + \olnot{v} \geq 1
	$,
	which   is precisely this previously added constraint.
	
	\item 
	$y=v$:
	Here we see that the desired constraint
	$
	\restrict{(x + \olnot{y} \geq 1)}{\witness}
	\synteq
	\witness(x) + \olnot{u} \geq 1
	$
	is again RUP 
	\wrt~\refeq{eq:sr-startingpoint}.
\end{enumerate}
This concludes our proof that all vertex dominance breaking
constraints that are consistent with our constructed linear
order~$\vdomrel$ can be added and certified by the 
redundance rule before the solvers starts searching for cliques.

So all of this works perfectly fine in theory.  The problem that rules
out this approach in practice, however, is that the solver will not
have the time to compute 
the dominance relation between vertices 
in advance,
since this is far too costly and does not pay off in general.
Instead the solver designed by \citet{MP16Finding}
will detect and apply vertex dominance relations on the fly during search.
And from a proof logging perspective this is too late---during search, when
$\coreunionderived$ will also contain constraints certifying any
backtracking made, 
our proof logging approach above no longer works.
The constraints added to the proof log to certify backtracking are no
longer possible to derive when substituted by~$\witness$
as in~\refeq{eq:sr-target},
for the simple reason that they are not semantically implied
by~\refeq{eq:sr-startingpoint}.
One possible way around this would be to run the solver twice---the
first time to collect all information about what vertex dominance
breaking will be applied, and then the second time to do the actual
proof logging---but this seems like quite a cumbersome approach.

%
%
%
%
%
%
%

We deliberately discuss this problem in some detail here, because this is an
example of an important and nontrivial
challenge that shows up also in other settings when designing proof
logging for other algorithms. 
It is not sufficient to just come up with  a proof
logging system that is strong enough in principle to certify the solver reasoning
(which the redundance rule is for the clique
solver with vertex dominance breaking, as shown above).
It is also crucial that the solver 
have 
enough information 
available at the right time 
and can extract this information efficiently enough to actually be
able to emit the required proof logging commands with low enough overhead.
For constraint programming solvers, it is not seldom the case that the
solver knows for sure that some variable should propagate to a value, 
because the domain has shrunk to a singleton, 
or that the search should backtrack because some variable domain is
empty, but that the solver cannot reconstruct the detailed derivation steps
required to certify this without incurring a massive overhead in
running time (e.g., since the reasoning has been performed with
bit-parallel logical operations).
It is precisely for this reason that it is important that our proof
system allow 
adding
RUP
constraints.
This makes it possible for the solver to claim facts that it knows to
be true, and that it knows can be easily verified, while leaving the work
of actually producing a detailed proof to the proof checker.

%

\subsection{Vertex Dominance with the
  \DOMRULELONG
  Rule}

Let us now discuss how the \domruleshort rule can be used to
provide proof logging for lazy global domination.
As was the case for the redundance rule, we will make use of
\cref{fact:constbreak}.
Before starting the proof logging, we use the
order change rule to activate the lexicographic order on the the
assignments to the vertices/variables induced by $\vdomrel$.

Suppose now that the solver is running and that the current candidate
clique is~$\cliquecurr$.
The solver has an ordered list of unassigned candidate vertices that
it is iterating over when considering how to enlarge this clique,
and this list is defined by the colour classes
$
(\colclass_{m}, \colclass_{m-1}, \ldots, \colclass_{1})
$.
(We note that this ordered list depends on~$\cliquecurr$,
and would be different for a different clique $\cliquecurr'$.)
Suppose the next vertex in that list is~$v$. Then when it is time to
make the next decision 
on line~\ref{line:innermost-loop}
in Algorithm~\ref{alg:max-clique}
about enlarging the clique, 
the solver enhanced with proof logging does the following:
\begin{itemize}
\item 
  If there exists a vertex $u$ that has already been
  considered in the current iteration and that dominates $v$
  (and for which it hence holds that~$u\vdomrel v$), 
  then the solver discards~$v$ by vertex dominance and adds the constraint 
  $u + \olnot{v} \geq 1$
  by the dominance rule with witness
  $
  \witness = 
  \set{
    u \mapsto v,
    v \mapsto u
  }
  $.  We will
  explain in detail 
  below 
  why this is possible. 
  
\item 
  Otherwise, 
  the solver enlarges~$\cliquecurr$ with~$v$   and makes 
  a recursive call.
\end{itemize}

When the solver has explored all ways of enlarging~$\cliquecurr$ and
is about to backtrack, here is what will happen on the proof logging
side (where we refer to~\cite{GMMNPT20CertifyingSolvers} for a more
detailed description of how proof logging for backtracking CP solvers
works in general):
\begin{enumerate}
\item%
  \label{item:previous-backtracking-step}%
  For every~$u$ that was explored in an enlarged clique 
  $\cliquecurr \union \set{u}$,
  when backtracking 
  the solver will already have added
  $
  \olnot{u} + \sum_{w \in \cliquecurr} \olnot{w} \geq 1  
  $
  as a RUP constraint.
  
\item 
  The solver now inserts the explicit cutting planes derivation required
  to show that
  the inequality
  $\setsize{\cliquecurr}+ j > \setsize{\cliquebest}$
  must hold.
  
\item 
  After this, the solver adds the claim that
  $\sum_{w \in \cliquecurr} \olnot{w} \geq 1$
  is a RUP constraint.
\end{enumerate}
We need to argue why
$\sum_{w \in \cliquecurr} \olnot{w} \geq 1$
will be accepted as a RUP constraint, allowing the solver to
backtrack.
The RUP check for $\sum_{w \in \cliquecurr} \olnot{w} \geq 1$ 
propagates $w = 1$ for all $w \in \cliquecurr$.
This in turn propagates $u = 0$ for all explored vertices~$u$ by the 
backtracking constraints for $\cliquecurr \union \set{u}$
added in step~\ref{item:previous-backtracking-step}.
The vertex dominance breaking constraints then propagate $v = 0$ 
for all vertices~$v$ discarded because of vertex domination.
At this point, the proof checker has the same information that the
solver had when it detected that the colouring constraint forced
backtracking. This means that the proof checker will unit propagate to
contradiction, and so the backtracking constraint
$\sum_{w \in \cliquecurr} \olnot{w} \geq 1$
is accepted as a RUP constraint.

We still need to explain how and why the
pseudo-Boolean dominance rule applications allow deriving the constraint 
$u + \olnot{v} \geq 1$ in case $u$ dominates $v$ (and hence $u\vdomrel v$). 
Recall that the order used in our proof is the lexicographic order induced by $\vdomrel$.
This means that if vertices/variables $u$ and~$v$ are assigned 
by~$\assmntalpha$
in such a way as to violate a dominance breaking constraint
$u + \olnot{v} \geq 1$,
then
$\compassmtafterassmt{\assmntalpha}{\witness}$
will flip
$u$ to~$1$
and $v$ to~$0$ to produce a lexicographically smaller assignment
(since $v$ is considered before~$u$ in the lexicographic order).

The conditions for the dominance rule are that we have to exhibit
proofs of~\refeq{eq:dom:one}
and~\refeq{eq:dom:two}.
In this discussion, let us focus on~\refeq{eq:dom:one}
which says that starting with the constraints
\begin{subequations}
	\begin{equation}
		\label{eq:dr-start}
		\coreset \union 
		\derivedset{} \union 
		\set{\negcwp{ u +   \olnot{v} \geq 1}}  
	\end{equation}
	and using only cutting planes rules, we should be able to derive
	\begin{equation}
		\label{eq:dr-target}
		\subst{\coreset}{\witness}
		\union
		\porderformdom
		\union
		\objwitnessconditionset  
		\eqperiod
	\end{equation}
\end{subequations}

Note first that our lexicographic order in fact
does \emph{not} in itself respect the 
objective function~\refeq{eq:clique-edge-constraint}.
However, since~$\witness$ just swaps two variables it leaves the
objective syntactically unchanged, meaning that the inequality
$\objwitnesscondition$
in~\refeq{eq:dom:one}
is seen to be trivially true.

As in our analysis of the redundance rule,
from~\refeq{eq:dr-start}
we obtain
$  \olnot{u} \geq 1 $
and
$    v \geq 1  $
as in~\refeq{eq:prop-u}--\refeq{eq:prop-v},
and
$\porderformdom$
is easily verified to be RUP \wrt these constraints, 
since what the formula says after cancellation 
is precisely that $v \geq u$.

It remains to consider the pseudo-Boolean constraints in the solver
constraint database $\coreunionderived$.
The crucial difference from the redundance rule is that we
no longer have to worry about proving
$\subst{\derivedset}{\witness}$ 
in~\refeq{eq:dr-target}---we only need to show how to derive
$\subst{\coreset}{\witness}$.
But this means that all we need to consider are the
non-edge constraints
in~\refeq{eq:clique-edge-constraint}.
We
already explained in our analysis for the redundance rule derivation that the fact that $u$ dominates~$v$
means that for any non-edge constraints affected by~$\witness$
their substituted versions are already there as input constraints or
are easily seen to be RUP constraints.
In addition to these non-edge constraints there might also be all
kinds of interesting derived constraints in~$\derivedset$, but the
dominance rule says that we can ignore those constraints.

Finally, although we skip the details here, it is not hard to argue
analogously to what has been done above to show that 
$\negc{\constrc} \synteq
\negcwp{ u +   \olnot{v} \geq 1}$
and
$\porderformdomneg$
in
\refeq{eq:dom:two}
together unit propagate to contradiction. 
This concludes our discussion 
of how to certify vertex dominance breaking in the 
maximum clique solver by \citet{MP16Finding}
using the pseudo-Boolean dominance rule introduced in this paper.



\providecommand{\jnshortcite}[1]{\shortcite{#1}}
\providecommand{\jnciteNameRef}[2]{\citeA{#1}}
\providecommand{\jnciteA}[1]{\citeA{#1}}
\providecommand{\jnciteinorby}[1]{by~\citeA{#1}\xspace}

\section{Conclusion}\label{sec:conclusion}

In this paper, we introduce a method for showing the validity of
constraints obtained by symmetry or dominance breaking by adding
simple, machine-verifiable certificates of correctness.  Using as our
foundation the cutting planes method
\cite{CCT87ComplexityCP}
for reasoning about
pseudo-Boolean constraints 
(also known as \mbox{$0$--$1$} linear
inequalities), and building on and extending the version of the
\veripb tool developed \jnciteinorby{GN21CertifyingParity}, we present
a proof logging method in which symmetry and dominance breaking is
easily expressible.  Our method is a strict extension of
\drat~\cite{HHW13Trimming,HHW13Verifying,WHH14DRAT}
and other similar methods earlier used for
solver proof logging,
which means that we can produce efficient proofs of validity for SAT
solving with fully general symmetry breaking, and we provide a
thorough evaluation showing that this approach is feasible in practice.
Since \veripb can also certify cardinality and parity (XOR) reasoning,
we now have for the first time a unified proof logging
method for SAT solvers using all of these enhanced solving methods.
To demonstrate that our proof logging approach is not limited to
Boolean satisfiability,
in this work   
we also present applications to symmetry
breaking in constraint programming and vertex domination in 
maximum clique solving.

From a theoretical point of view, it would be interesting to
understand better the power of the \domrulelong rule.
\drat viewed as a proof system is closely related to
\emph{extended resolution}
\cite{Tseitin68ComplexityTranslated}
in that these two proof systems have the same proof power up to
polynomial factors~\cite{KRH16ExtendedResolution},
and extended resolution, in turn, is polynomially equivalent to the
\emph{extended Frege} proof system~\cite{CR79Relative},
which is one of the strongest
proof systems studied in proof complexity. However, there are
indications that the cutting planes proof system equipped with the
\domrulelong rule might be strictly stronger than extended Frege
\cite{KT23personalcommunication}.

Another question is whether such a strong derivation rule as
\domrulelong  is necessary to generate efficient proofs of validity
for general symmetry breaking constraints, or whether
\redrulelong is enough.
To phrase this cleanly as a proof complexity problem, we can restrict
our attention to decision problems and also fix the order to be
lexicographic order over some set of variables~$\vecx$.
Let us define cutting planes with symmetry breaking
to be the cutting planes proof system extended with a rule
that allows to derive the constraint
$\vecx \lex \restrict{\vecx}{\sigma}$
for any symmetry~$\sigma$ of the input formula~$\formf$.
Now we can ask whether cutting planes with
\redrulelong
efficiently simulates this cutting planes proof system
with symmetry breaking. To the best of our understanding, this question is
wide open.

From a proof logging perspective, a
natural next problem to investigate
is whether our certification method is strong
enough to capture other solving techniques such as those used 
for SAT-based optimisation in so-called MaxSAT solvers.
A crucial component of several MaxSAT solving techniques is
the translation of pseudo-Boolean constraints to CNF.  This is used in
linear SAT-UNSAT (LSU) solvers 
for adding the \modelimproving constraints
\cite<e.g.,>{qmaxsat,pacose}, in core-guided solvers for
reformulation of the objective function \cite<e.g.,>{RC2}, and in implicit
hitting set solvers that make use of abstract cores
\cite{BBP20AbstractCores}.  
While developing proof logging
methods
that can support the full
range of modern MaxSAT solving techniques remains a formidable
challenge, we want to point out that significant progress has
been made of late.
The proof system introduced
\jnciteinorby{GN21CertifyingParity}
has been used to certify correctness of the translations of
pseudo-Boolean constraints into CNF for a range of encodings
\cite{GMNO22CertifiedCNFencodingPB},
and quite similar ideas have been employed to design proof logging for 
\amodelimproving solver for unweighted MaxSAT
\cite{VWB22QMaxSATpb}.
Very recently, this has been extended also to state-of-the-art
core-guided MaxSAT solvers \cite{BBNOV23CertifiedCGSS}.

Another intriguing problem is how to design efficient proof logging for 
\emph{symmetric learning} as explored
\jnciteinorby{DBB17SymmetricExplanationLearning}.
In contrast to the symmetry breaking techniques considered in the
current paper, when using symmetric learning one can only derive
constraints that are semantically implied by the input
formula. However, if at some point the solver derives a
constraint~$\constrd$, and if $\symsigma$ is a symmetry of the formula,
then it is clear that the permuted constraint~$\symsigma(\constrd)$
is also implied by the formula and so should be
sound to derive in a single extra step. 

If we wanted to argue formally about such symmetric learning, we could
apply~$\symsigma$ 
to the whole derivation leading up to~$\constrd$ to 
get a proof that $\symsigma(\constrd)$ can be derived.
Taking this observation one step further, if from the solver execution we can
extract a subderivation~$\subderivpi$ 
showing that a constraint~$\constrd$ is implied by constraints 
$\constrc_1, \ldots, \constrc_m$,
and if
$\symsigma(\constrc_1), \ldots, \symsigma(\constrc_m)$
have also been derived, then it
should 
be valid to use~$\subderivpi$ as a
``lemma'' to conclude~$\symsigma(\constrd)$.
The usefulness of such lemmas in the context of proof logging has been
discussed by, e.g.,~\jnciteA{KM21SolvingGraphHomomorphism}.
The question,
however,   
is whether such reasoning can be
incorporated in the  \veripb framework with the current set of
derivation rules in the proof system.  
One way of attempting to do this could of course be to
add
special
rules for symmetric learning or lemmas, but this goes against the goal
of keeping the proof logging system as simple as possible, so that
derivations are obviously sound. As a case in point, it is worth
noting that the naive way of adding such dedicated rules for
symmetries or lemmas to our
\redrulelong
and
\domrulelong
rules would result in an unsound proof system.

The use of lemmas in proofs 
has been studied in the so-called 
\emph{substitution  Frege} proof system~\cite{CR79Relative}, 
and it has been shown that
%
%
extended Frege
has the same deductive strength as substitution Frege  except
possibly for a polynomial overhead in proof 
size~\cite{KP89PropositionalProofSystems}.
However, it is not clear whether such results can be scaled down 
from Frege systems to
cutting planes,
so that cutting planes with \redrulelong can be made to simulate
the usage of lemmas as described above,
and whether such reasoning could be implemented efficiently enough in \veripb
to allow a formalization of reasoning with lemmas
that would be feasible in practice.

As noted above, we have also shown in this work that our proof logging
techniques can be used beyond SAT solving and SAT-based optimization
for other combinatorial solving paradigms such as
graph solving algorithms
and
constraint programming (CP), 
and 
\citet{EGMN20Justifying} and \citet{GMN22AuditableCP}
have made important contributions towards providing
\veripb-style proof logging for a full-blown CP solver.
For mixed integer linear programming (MIP) solvers, 
\citet{CGS17Verifying} and \citet{EG21ComputationalStatusUpdate} 
have also developed limited proof logging support (using another proof
format), but
this method still seems
quite far from being able to support advanced
MIP techniques. It would be very exciting if all of these different
proof logging techniques could be strengthened to provide full proof
logging support for state-of-the-art CP and MIP solvers.

Finally, we want to point out that another important research
direction in proof logging is to develop
formally verified proof checkers,
so that we can be sure when proof verification passes that this is not
due to bugs in the proof checker but that the claimed result is
guaranteed to be valid.
Such verified checkers have been built for \drat and other clausal
proof logging systems 
\cite{CMS17EfficientCertified,CHHKS17EfficientCertified,Lammich20EfficientVerified,THM21Cake},
and it would be highly desirable to obtain such tools for
pseudo-Boolean proof logging with \veripb as well.
As a first step in this direction,
a formally verified pseudo-Boolean proof checker  for decision
problems
was recently submitted to and used in
the SAT Competition 2023~\cite{BMMNOT23DocumentationVeriPB}.

%
%
%

\section*{Acknowledgements}

The authors gratefully acknowledge fruitful and stimulating discussions on proof logging with
Jeremias Berg,
Armin Biere,
Jo Devriendt,
Jan Elffers,
Ambros Gleixner,
\mbox{Marijn Heule},
Daniela Kaufmann,
Daniel Le Berre,
Matthew McIlree,
Magnus Myreen,
Yong Kiam Tan,
James Trimble,
and many other colleagues whom we have probably forgotten and to whom
we apologize.
We are also grateful to the anonymous \emph{AAAI} and \emph{JAIR} reviewers,
whose comments helped us to improve the exposition considerably and also
to fix some mistakes in the definitions.

Part of this work was carried out while taking part in the semester program
\emph{Satisfiability: Theory, Practice, and Beyond}
in the spring of 2021 at the 
Simons Institute for the Theory of Computing at UC Berkeley,
and in the extended reunion of this semester program in the spring of 2023.
This work has also benefited greatly from discussions during the
Dagstuhl Seminars 22411
\emph{Theory and Practice of SAT and Combinatorial Solving}
and 23261
\emph{SAT Encodings and Beyond}.

Stephan Gocht and Jakob Nordström were supported by
the Swedish Research Council grant \mbox{2016-00782},
and Jakob Nordström also received funding from 
the Independent Research Fund Denmark grant \mbox{9040-00389B}.
Ciaran McCreesh was supported by a Royal Academy of Engineering research fellowship.
Bart Bogaerts was supported by  Fonds Wetenschappelijk Onderzoek -- Vlaanderen (project G0B2221N), by the Flemish Government (AI Research Program), and by TAILOR, a project funded by EU Horizon 2020 research and innovation programme under GA No 952215.



\appendix
\section{A 
Proof Logging Example for   
Symmetry Breaking}
\label{sec:sym-appendix}


\lstset{aboveskip=10pt,belowskip=10pt,numbers=left,xleftmargin=2.4em,numberstyle      = \scriptsize\color{black}L,breaklines=true,
  postbreak=\mbox{\textcolor{red}{$\hookrightarrow$}\space},,escapechar=|,literate={~} {$\sim$}{1}}
\def\ContinueLineNumber{\lstset{firstnumber=last}}
\def\StartLineAt#1{\lstset{firstnumber=#1}}
\let\numberLineAt\StartLineAt
\ContinueLineNumber

\makeatletter
\newcommand{\mathleft}{\@fleqntrue\@mathmargin0pt}
\newcommand{\mathcenter}{\@fleqnfalse}
\makeatother

\newcounter{proofFormula}
\newcounter{localFormula}
\newcommand{\formulaLine}[1]{\stepcounter{proofFormula}%
\begin{equation}\tag{C\arabic{proofFormula}}
 #1
\end{equation}%
}
\newcommand{\ALformulaLine}[1]{\stepcounter{proofFormula}%
	\tag{C\arabic{proofFormula}}
		#1
}
\newcommand{\MLformulaLine}[1]{\stepcounter{proofFormula}%
	\begin{equation}\tag{C\arabic{proofFormula}}\begin{array}{l}
		#1\end{array}
	\end{equation}%
}

\newcommand\localFormulaLine[1]{\stepcounter{localFormula}%
  \begin{equation}
    \tag{T\arabic{localFormula}}
    #1
  \end{equation}%
}
\newcommand\ALlocalFormulaLine[1]{\stepcounter{localFormula}%
	\tag{T\arabic{localFormula}}
		#1
}

\newcommand\skipformula{\stepcounter{proofFormula}}



\newcommand\pigeon[2]{p_{#1 #2}}
\newcommand\olnotpigeon[2]{\overline{p_{#1 #2}}}
\renewcommand\olnotpigeon[2]{\overline{p}_{#1 #2}}
\newcommand\code[1]{\texttt{#1}\xspace}
\newcommand\uvar[2]{u_{#1 #2}}
\newcommand\vvar[2]{v_{#1 #2}}
\newcommand\wvar[2]{w_{#1 #2}}
\newcommand\yvar[1]{y_{#1}}
\newcommand\olnotyvar[1]{\overline{y_{#1}}}
\renewcommand\olnotyvar[1]{\overline{y}_{#1}}

In this appendix, we present a 
fully worked-out 
example of
symmetry breaking 
 using the \domrulelong rule for the well-known
pigeonhole principle formulas claiming that
$n+1$~pigeons can be mapped to $n$~pigeonholes in a one-to-one fashion.
\citet{Haken85Intractability} showed that resolution proofs of
unsatisfiability for such formulas requires a number of clauses that
scales exponentially with~$n$, and since conflict-driven clause
learning SAT solvers can be seen to search for 
resolution proofs \cite{BKS04TowardsUnderstanding},
this means that solvers without enhanced reasoning methods will have
to run for exponential time.
However, since pigeonhole principle formulas are fully symmetric 
with respect to both pigeons and holes, it is possible to add symmetry
breaking constraints encoding that without loss of generality
pigeon~$1$ resides in hole~$1$,
pigeon~$2$ resides in hole~$2$,
et cetera,
and once such symmetry breaking constraints have been added the
problem becomes trivial.

What we show in this appendix is how the \breakid tool can break
pigeonhole principle symmetries in a fully automated fashion, and
produce proofs of correctness for the symmetry breaking clauses that the
proof checker will accept. The proofs for these clauses
can then be concatenated with the SAT solver
proof log
(rewritten from \drat to \veripb-format) and fed to \veripb
to provide end-to-end verification for the whole solving process.

We present the symmetry breaking proof logging in the syntactic format
used by \veripb, so that the reader will be able to see what 
pseudo-Boolean proof files actually look like. 
To keep the example manageable, we consider a very small instance of the
pigeonhole principle with only $4$~pigeons and $3$~pigeonholes.
We encode the pigeonhole principle formula using variables~$\pigeon ij$,
with the intended interpretation that
$\pigeon ij$ is true if pigeon~$i$ resides in hole~$j$. 
The input for the symmetry breaking preprocessor consists
of a CNF formula which can be written in pseudo-Boolean form as
\begin{align}
  \ALformulaLine{
    \label{eq:first-pigeon-ax}
\pigeon11 + \pigeon12 + \pigeon13 \geq 1 } \\\ALformulaLine{
\pigeon21 + \pigeon22 + \pigeon23 \geq 1 } \\\ALformulaLine{
  \pigeon31 + \pigeon32 + \pigeon33 \geq 1 } \\\ALformulaLine{
  \label{eq:last-pigeon-ax}
  \pigeon41 + \pigeon42 + \pigeon43 \geq 1 } \\\ALformulaLine{
  \label{eq:first-hole-ax}
\olnotpigeon11 + \olnotpigeon21 \geq 1 } \\\ALformulaLine{
\olnotpigeon11 + \olnotpigeon31 \geq 1 } \\\ALformulaLine{
\olnotpigeon11 + \olnotpigeon41 \geq 1 } \\\ALformulaLine{
\olnotpigeon21 + \olnotpigeon31 \geq 1 } \\\ALformulaLine{
\olnotpigeon21 + \olnotpigeon41 \geq 1 } \\\ALformulaLine{
\olnotpigeon31 + \olnotpigeon41 \geq 1 } \\\ALformulaLine{
\olnotpigeon12 + \olnotpigeon22 \geq 1 } \\\ALformulaLine{
\olnotpigeon12 + \olnotpigeon32 \geq 1 } \\\ALformulaLine{
\olnotpigeon12 + \olnotpigeon42 \geq 1 } \\\ALformulaLine{
\olnotpigeon22 + \olnotpigeon32 \geq 1 } \\\ALformulaLine{
\olnotpigeon22 + \olnotpigeon42 \geq 1 } \\\ALformulaLine{
\olnotpigeon32 + \olnotpigeon42 \geq 1 } \\\ALformulaLine{
\olnotpigeon13 + \olnotpigeon23 \geq 1 } \\\ALformulaLine{
\olnotpigeon13 + \olnotpigeon33 \geq 1 } \\\ALformulaLine{
\olnotpigeon13 + \olnotpigeon43 \geq 1 } \\\ALformulaLine{
\olnotpigeon23 + \olnotpigeon33 \geq 1 } \\\ALformulaLine{
\olnotpigeon23 + \olnotpigeon43 \geq 1 } \\\ALformulaLine{
  \label{eq:last-hole-ax}
\olnotpigeon33 + \olnotpigeon43 \geq 1 } 
\end{align}
where
constraints \eqref{eq:first-pigeon-ax}--\eqref{eq:last-pigeon-ax}
represent that each pigeon resides in
at least one hole,
and constraints
\eqref{eq:first-hole-ax}--\eqref{eq:last-hole-ax}
enforce  that each hole is occupied by at most one pigeon (by
specifying for every pigeonhole and every pair of distinct pigeons
that it cannot be the case that both of these pigeons reside in the
hole). 
When \veripb is used for SAT proof logging, the proof checker
parses
CNF formulas in the standard DIMACS format used by SAT solvers,
but the CNF formula will be reprented internally in the proof checker
as a set of pseudo-Boolean constraints as above.

\subsection{%
  Starting the Proof and
  Introducing the Order}
\label{sec:sym-appendix-order}
A \veripb proof starts with a proof header (stating which version of the proof system is used) and an instruction to load the input formula 
\begin{lstlisting}
pseudo-Boolean proof version 2.0
f 22
\end{lstlisting}
where the number~$22$ specifies the number of
pseudo-Boolean constraints in the input.
All constraints in the proof file will be numbered consecutively
starting with the input constraints, and will be referred to in the
derivations by these numbers. 

To prepare for the symmetry breaking, \breakid then introduces a
pre-order that compares two binary sequences in lexicographic order.
This pre-order is defined by inserting the lines
\begin{lstlisting}
pre_order exp22|\label{exp22name}|
	vars
		left  u1 u2 u3 u4 u5 u6 u7 u8 u9 u10 u11 u12  |\label{uvars}|
		right v1 v2 v3 v4 v5 v6 v7 v8 v9 v10 v11 v12  |\label{vvars}|
		aux
	end

	def
		-1 u12 1 v12 -2 u11 2 v11 -4 u10 4 v10 -8 u9 8 v9 -16 u8 16 v8 -32 u7 32 v7 -64 u6 64 v6 -128 u5 128 v5 -256 u4 256 v4 -512 u3 512 v3 -1024 u2 1024 v2 -2048 u1 2048 v1 >= 0; |\label{orderdef}|
	end

	transitivity
		vars
			fresh_right w1 w2 w3 w4 w5 w6 w7 w8 w9 w10 w11 w12   |\label{extravars}|
		end
		proof
			proofgoal #1
				pol 1 2 + 3 +|\label{addtrans}|
			qed -1 |\label{transqed}|
		qed
	end
end
\end{lstlisting}
in the proof file.
The pre-order is 
named \texttt{exp22} on Line \ref{exp22name},
so that we can refer to
it
later in the proof.
Lines \ref{uvars} and \ref{vvars} introduce
placeholder names for the
$2 \times 12$~variables used to define the order.
Line \ref{orderdef} then provides
the exponential encoding~\eqref{eq:encode-LL}  of the
claim 
that the sequence of variables
\code{u1}, \ldots, \code{u12}
is lexicographically
smaller than or equal to
\code{v1}, \ldots, \code{v12}.

To prove that the pseudo-Boolean formula consisting of the single
constraint on Line~\ref{orderdef} defines a transitive relation, 
a third set of placeholder variables
\code{w1}, \ldots, \code{w12}
is declared on  Line~\ref{extravars}, 
and these variables are used in a formal derivation as specified
in~\eqref{eq:order-trans}
on page~\pageref{eq:order-trans}
that 
$\porderform{\vecu}{\vecv}$
and
$\porderform{\vecv}{\vecw}$
together imply
$\porderform{\vecu}{\vecw}$.
The \veripb tool has a certain preference for proofs by contradiction,
however, so the way this is established is by showing that
$\porderform{\vecu}{\vecv}$
and
$\porderform{\vecv}{\vecw}$
together with the negation
$\lnot \porderform{\vecu}{\vecw}$
is contradictory (and if the order consistent of more than one PB
constraint, one negates the constraints one by one and derives
contradiction for every negated constraint.
%
For the order on Line~\ref{orderdef} we get the three constraints
\begin{align}
  \ALlocalFormulaLine{
  -\uvar12 + \vvar12 -2 \uvar11 +2 \vvar11 -4
  \uvar10 +4 \vvar10  -\dots  &\geq 0
  \label{transone}}
\\\ALlocalFormulaLine{- \vvar12 + \wvar12 -2 \vvar11 +2 \wvar11 -4
  \vvar10 +4 \wvar10  -\dots  &\geq 0
  \label{transtwo}}
  \\\ALlocalFormulaLine{
  \uvar12 - \wvar12 +2 \uvar11 -2 \wvar11 +4 \uvar10 -4 \wvar10  +\dots  &\geq -1, \label{transthree}}
\end{align}
where we use the labels \eqref{transone}--\eqref{transthree} to
emphasize that these are not 
constraints learned in the proof system, but
are 
temporary constraints, local to the proof of transitivity.  
Line~\ref{addtrans} is an instruction in reverse polish notation to add
the constraints \eqref{transone}, \eqref{transtwo},
and~\eqref{transthree} together,
resulting (after simplification)
in the constraint  
\localFormulaLine{0 \geq -1 \eqperiod \label{transfour}}
Line \ref{transqed} then concludes the subproof, which is possible
since the last derived constraint
(which, using relative indexing, is what the number~$-1$ refers to)
is indeed a conflicting constraint.
%
Note that in order to prove that the pseudo-Boolean constraint defines
a pre-order we also need to show that the relation it defines is
reflexive,
but for a simple order like the one in this example everything in the
formula trivializes when you substitute the same variables
for both the $u$- and the $v$-sequence, and so \veripb
can figure out that reflexivity holds by itself without the help of
any written proof.

%
%
%

The proof continues with the instruction 
\begin{lstlisting}
load_order exp22 p21 p22 p23 p11 p12 p13 p31 p32 p33 p41 p42 p43 
\end{lstlisting}
which specifies that the order \texttt{exp22} should be used and
should be applied to all variables in the formula in the specified order.
The reader might find slightly odd that in the order specified above
all variables related to pigeon~$2$ are ordered before those
referring to pigeon~$1$, followed by variables mentioning pigeons~$3$
and~$4$.  In other words, in the lex-leader order, assignments are
sorted with respect to pigeon~$2$ first.
The reason for this seemingly strange choice is
that the entire proof presented in this appendix is actually generated
fully automatically by \breakid.
This tool only takes a propositional logic formula as input, and has
no information about high-level interpretations of what the different
variables mean, or which ordering of these variables might seem more
or less natural from the point of view of a human observer.

\subsection{Logging the Breaking of a
  First Symmetry}

Now that the order has been defined and has been proven to be an
order, we can start adding symmetry breaking constraints to the proof
log.  The first symmetry that \breakid considers is  
\begin{equation}
  \label{eq:first-symmetry}
  \pi :=  ( \pigeon11 \pigeon43 ) ( \pigeon12 \pigeon42 ) ( \pigeon13
  \pigeon41 )  ( \pigeon21 \pigeon23 ) ( \pigeon31 \pigeon33 )
  \eqcomma  
\end{equation}
which is the symmetry that
simultaneously swaps pigeons~$1$ and~$4$ and pigeonholes~$1$ and~$3$.
We remark that the questions of why \breakid chooses to break this
particular symmetry, and how it finds the symmetry in the first place,
are interesting and nontrivial questions, but they are not relevant
for our work. This is so since we are not trying to construct new symmetry
breaking tools, but to design proof logging methods to certify the
correctness of the symmetry breaking constraints added by existing
tools. From the point of view of proof logging, it is mostly
irrelevant to dwell on the possible reasons why a specific
symmetry was chosen,  or how it was found. Instead,
we can just take the symmetry as a given and focus on proof
logging for the symmetry breaking constraints.

As explained in our discussion of symmetry breaking in
\cref{sec:symmetrybreaking}, in order to break
a given symmetry we first use the dominance rule to  derive a
pseudo-Boolean exponential encoding 
of a lex-leader by the proof lines
\begin{lstlisting}
dom -1 p43 1 p11 -2 p42 2 p12 -4 p41 4 p13 -8 p33 8 p31 -32 p31 32 p33 -64 p13 64 p41 -128 p12 128 p42 -256 p11 256 p43 -512 p23 512 p21 -2048 p21 2048 p23 >= 0 ;  p11 -> p43 p12 -> p42 p13 -> p41 p21 -> p23 p23 -> p21 p31 -> p33 p33 -> p31 p41 -> p13 p42 -> p12 p43 -> p11 ; |\label{line:dom1}| begin 
	proofgoal #2|\label{line:domproof1}|
		pol -1 -2 + |\label{line:dom:add}|
	qed -1|\label{line:domproof2}|
\end{lstlisting}
which we will now discuss in more detail.

Line~\ref{line:dom1} above
says that the dominance rule should be used to derive
(and add to the derived set~$\derivedset$)
the constraint~\eqref{eq:CLL},
which expresses that the assignment to the variables should not become
lexicographically smaller when the symmetry is applied.
The variables related to pigeon~$2$ occur with the highest
coefficients, since they were given the highest priority when the
order was instantiated.
Notice that the pseudo-Boolean constraint specified on
Line~\ref{line:dom1} contains two occurrences of every variable.
This is because the constraint has been generated automatically, and
\veripb will perform cancellations to simplify the constraint before
storing it internally.

An application of the dominance rule needs to provide more information
to \veripb than just the  constraint to be derived by dominance,
however. The proof should also specify
\begin{itemize}
\item
  the witness, which in this case is just the symmetry, written as
  a substitution at  the end of Line~\ref{line:dom1}, and  
\item
  explicit derivations (\emph{subproofs})  for all constraints
  (\emph{proof goals})
  where the proof checker
  cannot automatically figure out such a  derivation---in this case,  we
  have a single such proof goal that is taken care of on
  Lines~\ref{line:domproof1}--\ref{line:domproof2}. 
\end{itemize}
Let us discuss in more detail what subproofs are required. In order 
to apply the dominance rule
to derive the constraint~$\constrc$ using witness~$\witness$,
we need to establish that derivations
\begin{subequations}
  \begin{alignat}{1}
      & \coreset \union 
      \derivedset{} \union 
      \set{\negc{\constrc}} \cpderives \subst{\coreset}{\witness}
      \union
      \porderformdom
      \union
      \objwitnessconditionset\label{eq:dom:one:repeatedPH}
      \\
      & \coreset
      \union \derivedset
      \union  \set{\negc{\constrc}} \union  \porderformdomneg
     \cpderives
      \bot
      \label{eq:dom:two:repeatedPH}
\end{alignat}
\end{subequations}
exist.
In the concrete case of the first symmetry breaking constraint for the
formula consisting of the clauses
\eqref{eq:first-pigeon-ax}--\eqref{eq:last-hole-ax}, 
\veripb will generate the following numbered list of derivations
(or \emph{proof obligations}) that it expects to see:
\begin{enumerate}
\item
  $\coreset \union 
  \derivedset{} \union 
  \set{\negc{\constrc}} \cpderives \porderformdom$ 
\item
  $\coreset \union 
  \derivedset{} \union 
  \set{\negc{\constrc}
  }
  \union
  \porderformdomneg
  \cpderives
  \bot
  $
\item
  $
  \coreset \union 
  \derivedset{} \union 
  \set{\negc{\constrc}} \cpderives \objwitnesscondition
  $ 
\item [4--25.]
  $
  \coreset \union 
  \derivedset{} \union
  \set{\negc{\constrc}} \cpderives \constrd$
  for each 
  $\constrd \in \subst{\coreset}{\witness}$.
\end{enumerate}
Except for the second proof obligation, all of them can be proved
automatically.
Perhaps the simplest case is the third obligation,  which is trivial
since we are dealing with a decision problem for which the objective
function is the constant
$f=0$.
And since $\witness$ is a syntactic symmetry of the set of core
constraints~$\coreset$ (which at this point is the CNF formula in the
input),
the proof obligations $4$--$25$ are also trivial.

Let us describe how the subproof for the second proof obligation on
Lines~\ref{line:domproof1}--\ref{line:domproof2}
is checked by \veripb.
First, \veripb makes available in the subproof the constraint $\lnot \constrc$,
which (after simplification) equals
\MLformulaLine{%
  255 \cdot \pigeon11 +126 \cdot \pigeon12 +60 \cdot \pigeon13 +1536
  \cdot \pigeon21 +1536 \cdot \olnotpigeon23
   \\
 \qquad 
 + 24 \cdot \pigeon31 +24
  \cdot \olnotpigeon33 +60 \cdot \olnotpigeon41 +126 \cdot \olnotpigeon42 +255
  \cdot \olnotpigeon43 \geq 2002\label{constr23} \eqcomma
}
giving this constraint the next available number~$23$ as
constraint reference ID.
Next, \veripb makes available the constraint
$\porderformdomneg$, which, after simplification, equals  
\MLformulaLine{255 \cdot \olnotpigeon11 +126 \cdot \olnotpigeon12 +60
  \cdot \olnotpigeon13 +1536 \cdot \olnotpigeon21 +1536 \cdot  \pigeon23
  \\
  \qquad
  + 24
  \cdot \olnotpigeon31 +24 \cdot \pigeon33 +60 \cdot \pigeon41 +126 \cdot \pigeon42 +255
  \cdot \pigeon43 \geq 2001 \eqperiod
  \label{constr24}
}
As explained above, all constraints are referred to by numbers, but
as we have already seen 
\veripb also allows relative indexing.
Therefore, Line~\ref{line:dom:add} simply states
in reverse polish notation 
that the two most recently 
generated constraints (i.e., \eqref{constr24} and~\eqref{constr23})
should be added, which yields the inequality
\begin{equation}
  255 +126 +60  +1536 +1536  +24  +24  +60 +126  +255  
  \geq
  2002+2001
\end{equation}
or, in simplified form,
\formulaLine{%
  0 \geq 1
\eqcomma
  \label{constr25}}
which   
is contradictory.
Line~\ref{line:domproof2} ends the proof by giving the (relative) identifier of the contradicting constraint that was just derived. 
Finally, when this proof is finished and all other proof obligation
have been automatically checked,
the desired symmetry breaking constraint
 \MLformulaLine{%
- \pigeon43 +1 \cdot \pigeon11 -2 \cdot \pigeon42 +2 \cdot \pigeon12 -4 \cdot \pigeon41 +4
\cdot \pigeon13 -8 \cdot \pigeon33 +8 \cdot \pigeon31 
\\
\qquad
- 32 \cdot \pigeon31 +32
\cdot \pigeon33 -64 \cdot \pigeon13 +64 \cdot \pigeon41  
-128 \cdot \pigeon12 +128 \cdot \pigeon42 
\\
\qquad
- 256 \cdot \pigeon11 +256 \cdot \pigeon43
- 512 \cdot \pigeon23 +512 \cdot \pigeon21 -2048 \cdot \pigeon21 +2048
\cdot \pigeon23 
\geq 0
\label{constr26}
}  
is added 
to the derived set~$\derivedset$.
Once all proof obligations have been taken care of, the
constraints~\eqref{constr23}, \eqref{constr24},
and~\eqref{constr25} are
automatically erased by the proof checker 
and can no longer be referred to
without triggering an error, since they are only relevant for the
derivations in the  subproof.

The inequality~\eqref{constr26} is a lex-leader constraint for the
symmetry we are considering, but this constraint is nothing that the SAT solver
knows about or can understand, since the solver only operates with
disjunctive clauses. What happens next in the proof, therefore, is
that the constraint~\eqref{constr26} is converted to a set of clauses
on the form \eqref{eq:init}--\eqref{eq:TsLast} 
that the solver can use for symmetry breaking.

First, \eqref{eq:init} is added with the redundance rule with the instruction 
\begin{lstlisting}
red 1 y0 >= 1 ; y0 -> 1
\end{lstlisting}
which contains both the constraint 
\formulaLine{\yvar0 \geq 1} and the witness $\yvar0\mapsto 1$ to apply
the redundance rule. All proof obligations are checked automatically
by \veripb.

In our chosen lexicographic order, the first variable is
$\pigeon21$. Therefore, the first clause for symmetry breaking is 
\begin{equation}
  \label{eq:symmetry-breaking-clause-1}
  \olnotpigeon21 \lor \pi(\pigeon21) \synteq \olnotpigeon21 \lor
  \pigeon23 
  \eqcomma  
\end{equation}
which is a simplification of~\eqref{eq:breaking}, omitting the
trivially true variable~$\yvar0$.  
This clause is implied by the constraint \eqref{constr26},
which can be seen by weakening away all other variables in it (\ie
adding literal axioms to cancel them).
Instead of providing an explicit derivation, we can 
simply add
the clause~\eqref{eq:symmetry-breaking-clause-1}
it with the reverse unit propagation rule and let \veripb figure out
the details, which we do by inserting the line
\begin{lstlisting}
rup 1 ~p21 1 p23 >= 1 ;
\end{lstlisting}
that derives the desired constraint
\formulaLine{%
  \olnotpigeon21 + \pigeon23 \geq 1
  \eqperiod 
  \label{constr28}
}
Next, the fresh variable $\yvar1$ is introduced with four redundance
rule applications  
\begin{lstlisting}
red 1  p23 1 ~y0  1 y1  >= 1 ; y1 -> 1
red 1 ~p21 1 ~y0  1 y1  >= 1 ; y1 -> 1
red 1 ~y1  1  y0        >= 1 ; y1 -> 0
red 1 ~y1  1 ~p23 1 p21 >= 1 ; y1 -> 0
\end{lstlisting}
with witnesses mapping~$\yvar1$ to either
$0$ or to~$1$, resulting in the constraints 
\begin{align}
\ALformulaLine{ \pigeon23 + \olnotyvar0 +  \yvar1  \geq 1 }
\\\ALformulaLine{ \olnotpigeon21 + \olnotyvar0 + \yvar1   \geq 1  }
  \\\ALformulaLine{
\olnotyvar1 +  \yvar0  \geq 1
\label{constr31}}
\\\ALformulaLine{ \olnotyvar1 + \olnotpigeon23 + \pigeon21   \geq 1\label{constr32} }
\end{align}
corresponding to the 
clauses~\eqref{eq:Ts1}--\eqref{eq:TsLast}.

Before repeating this procedure for the next variable~$\yvar2$, we use the
recently derived constraints to cancel out the dominant terms in
constraint \eqref{constr26} with the instructions  
\begin{lstlisting}
pol 26 32 2048 * +
del id 26
\end{lstlisting}
The first line above adds $2048$ times \eqref{constr32}
to \eqref{constr26}, yielding
\MLformulaLine{%
  255 \cdot \olnotpigeon11 +126 \cdot \olnotpigeon12 +60 \cdot \olnotpigeon13+ 512
  \cdot \pigeon21 +512 \cdot \olnotpigeon23 
  \\
  \qquad
  + 24 \cdot \olnotpigeon31 +24
  \cdot \pigeon33 +60 \cdot \pigeon41 +126 \cdot \pigeon42 +255 \cdot \pigeon43 +2048
  \cdot \olnotyvar1 \geq 977
  \label{constr33}
\eqperiod
}
The second line deletes \eqref{constr26} from
$\derivedset$ since it will no longer be required.  

The next variable in lexicographic order after~$\pigeon21$
is~$\pigeon22$. However, since our symmetry~$\pi$  maps 
$\pigeon22$ to itself, no symmetry breaking clauses are added for it.
The variable after that in the ordering is $\pigeon23$, which is mapped to
$\pigeon21$, resulting in the (conditional on~$y_1$) symmetry breaking
constraint  
\formulaLine{\olnotyvar1 + \olnotpigeon23 + \pigeon21 \geq 1}
obtained by adding the line
\begin{lstlisting}
rup 1 ~y1 1 ~p23 1 p21 >= 1 ;
\end{lstlisting}
to the proof.
After this, the next fresh variable~$\yvar2$ is introduced
in the same way as~$\yvar1$ 
using the redundance rule in the instructions
\begin{lstlisting}
red 1  p21 1 ~y1  1 y2   >= 1 ; y2 -> 1
red 1 ~p23 1 ~y1  1 y2   >= 1 ; y2 -> 1
red 1 ~y2  1  y1         >= 1 ; y2 -> 0
red 1 ~y2  1 ~p21 1 p23  >= 1 ; y2 -> 0
\end{lstlisting}
yielding the clauses
\begin{align}
\ALformulaLine{\pigeon21 + \olnotyvar1 + \yvar2  \geq 1}
\\\ALformulaLine{\olnotpigeon23 + \olnotyvar1 + \yvar2 \geq 1}
\\\ALformulaLine{\yvar1 + \olnotyvar2 \geq 1 \label{constr37}}
\\\ALformulaLine{\olnotpigeon21 + \pigeon23 + \olnotyvar2 \geq 1\label{constr38}}
\end{align}
As before, our pseudo-Boolean symmetry breaking constraint is
simplified with  
\begin{lstlisting}
pol 33 38 512 * +
del id 33
\end{lstlisting}
where the first instruction again cancels out the dominant terms
in~\eqref{constr33},
replacing them by a $y$\nobreakdash-variable,
to express a conditional symmetry
breaking constraint, resulting in  
\MLformulaLine{%
  255 \cdot \olnotpigeon11 +126 \cdot \olnotpigeon12 +60 \cdot \olnotpigeon13
  +24 \cdot \olnotpigeon31 +24 \cdot \pigeon33
  \\ \qquad
 + 60 \cdot \pigeon41 +126 \cdot \pigeon42 +255
  \cdot \pigeon43 +2048 \cdot \olnotyvar1 +512 \cdot \olnotyvar2 
  \geq   465
\label{constr39}
} 
and the second instruction deletes constraint~\eqref{constr33}.

The next variable in our chosen order is~$\pigeon 11$.
Since~$\pigeon 11$ is mapped to $\pigeon43$, we want to have the
symmetry breaking clause
\formulaLine{%
  \olnotpigeon11 + \pigeon43 + \olnotyvar2 \geq 1
  \eqcomma
\label{constr40}
}
which can be derived by the proof line
\begin{lstlisting}
rup 1 ~y2 1 ~p11 1 p43 >= 1 ;
\end{lstlisting}
To see that the clause~\eqref{constr40} indeed follows by reverse unit
propagation, consider what happens if all literals in the clause are
set to false.  Whenever $\yvar2$ is true, so are $\yvar1$
and~$\yvar0$ (by \eqref{constr37} and~\eqref{constr31}). If furthermore
$\pigeon43$~is false  and $\pigeon11$~is true,
then \eqref{constr39} simplifies to  
\begin{equation}
 126 \cdot \olnotpigeon12 +60 \cdot \olnotpigeon13 +24 \cdot \olnotpigeon31 +24
  \cdot \pigeon33 +60 \cdot \pigeon41 +126 \cdot \pigeon42  \geq 465
\eqcomma
\end{equation}
which can never  
be satisifed since the coefficients on the left only add up to $420$.  

The process of introducing a new variable is the same as before,
appending the 
constraints
\begin{align}
\ALformulaLine{\pigeon43 + \olnotyvar2 + \yvar3 \geq 1}
\\\ALformulaLine{\olnotpigeon11 + \olnotyvar2 + \yvar3 \geq 1}
\\\ALformulaLine{\yvar2 + \olnotyvar3 \geq 1}
\\\ALformulaLine{\pigeon11 + \olnotpigeon43 + \olnotyvar3 \geq 1\label{constr44}}
\end{align}
to the derived set~$\derivedset$.
The last constraint
can then again be used to simplify \eqref{constr39} with the instruction
\begin{lstlisting}
pol 39 44 256 * +
\end{lstlisting}
yielding the constraint
\MLformulaLine{%
\pigeon11 +126 \cdot \olnotpigeon12+ 60 \cdot \olnotpigeon13 +24 \cdot \olnotpigeon31+
24 \cdot \pigeon33+ 60 \cdot \pigeon41 
\\
\qquad
+126 \cdot \pigeon42 + 
\olnotpigeon43
+ 2048
\cdot \olnotyvar1 +512 \cdot \olnotyvar2+ 256 \cdot \olnotyvar3 \geq 211
} 
This process continues in the same way for all variables
in the  lexicographic ordering that are not mapped to themselves, or
\emph{stabilized}, 
by~$\pi$.

\subsection{Logging the Breaking of More Symmetries}

As the \breakid execution continues, more symmetries are detected and
broken, and the corresponding symmetry breaking clauses are derived
in the proof. The process is completely analogous to what is described
above for the first symmetry.
For our concrete toy example formula with $4$~pigeons and $3$~holes,
\breakid next breaks the symmetries
\begin{subequations}
\begin{align}
 &( \pigeon11 \pigeon12 ) ( \pigeon21 \pigeon32 ) ( \pigeon22 \pigeon31 ) ( \pigeon23 \pigeon33 ) ( \pigeon41 \pigeon42 ) \\ 
 &( \pigeon21 \pigeon11 ) ( \pigeon22 \pigeon12 ) ( \pigeon23 \pigeon13 ) \\
 &( \pigeon11 \pigeon31 ) ( \pigeon12 \pigeon32 ) ( \pigeon13 \pigeon33 ) \\
 &( \pigeon31 \pigeon41 ) ( \pigeon32 \pigeon42 ) ( \pigeon33 \pigeon43 ) \\
 &( \pigeon21 \pigeon22 ) ( \pigeon11 \pigeon12 ) ( \pigeon31 \pigeon32 ) ( \pigeon41 \pigeon42 ) \\
 &( \pigeon22 \pigeon23 ) ( \pigeon12 \pigeon13 ) ( \pigeon32 \pigeon33 ) ( \pigeon42 \pigeon43 )
\end{align}
\end{subequations}
in the order listed
(where, just to help decode the notation, 
the first of these symmetries swaps holes $1$ and~$2$ and
simultaneously swaps pigeons $2$ and~$3$).
It is crucial to note here that when we break later symmetries in this
list, we do not have to worry about previously added symmetry breaking
clauses. There is no interaction between the different symmetry
breaking derivations, since all the symmetry breaking constraints are
added to the derived set~$\derivedset$, whereas the core
set~$\coreset$ containing the proof obligation for the dominance rule
applications only consists of the symmetric input formula.

\subsection{Proof for the Preprocessed Formula}

The symmetry breaking performed here only serves as a
preprocessing
step for the solving; it is not a complete proof leading to a
contradiction.
A revision of the \veripb proof format
used in the SAT Competition 2023
\cite{BMMNOT23DocumentationVeriPB},
and currently under further development, 
will allow proofs for preprocessing steps, establishing that the
formulas before and after preprocessing are equisatisfiable
(or, for optimization problems, that the two formulas have the same
optimal value for the objective function).

Very briefly, the way a proof for preprocessing will work is that at
the end of the proof all constraints to be output should be moved to
the core set~$\coreset$, after which the derived set~$\derivedset$ is
emptied.  The constraints in~$\coreset$ then constitute the output
formula, which is guaranteed to be equisatisfiable to the input
formula if all deletion steps are instances of checked deletion as
described in \cref{def:deletion-rule} in \cref{sec:del}.
We remark that this is quite similar to the concept of
\emph{finalization}. 

%
%
%
%
%
%

For our toy example, in order to prove equisatisfiability of the
formula after having added symmetry breaking clauses one should
end the proof with the lines  
\lstset{firstnumber=492}
\begin{lstlisting}
 core id 27 28 29 30 31 32 34 35 36 37 38 40 41 42 43 44 46 47 48 49 50 52 57 58 59 60 61 62 64 65 66 67 68 70 71 72 73 74 76 77 78 79 80 82 87 88 89 90 91 92 94 95 96 97 98 100 105 106 107 108 109 110 112 113 114 115 116 118 123 124 125 126 127 128 130 131 132 133 134 136 141 142 143 144 145 146 148 149 150 151 152 154 155 156 157 158 160 165 166 167 168 169 170 172 173 174 175 176 178 179 180 181 182 184 |\label{line:movecore}|
 output EQUISATISFIABLE PERMUTATION |\label{line:outputType}|
 * #variable= 39 #constraint=136 |\label{line:outputHeader}|
 1 2 3 4 5 6 7 8 9 10 11 12 13 14 15 16 17 18 19 20 21 22 27 28 29 30 31 32 34 35 36 37 38 40 41 42 43 44 46 47 48 49 50 52 57 58 59 60 61 62 64 65 66 67 68 70 71 72 73 74 76 77 78 79 80 82 87 88 89 90 91 92 94 95 96 97 98 100 105 106 107 108 109 110 112 113 114 115	 116 118 123 124 125 126 127 128 130 131 132 133 134 136 141 142 143 144 145 146 148 149 150 151 152 154 155 156 157 158 160 165 166 167 168 169 170 172 173 174 175 176 178 179 180 181 182 184 |\label{line:constraints}|
 conclusion NONE |\label{line:conclusion}|
 end pseudo-Boolean proof
\end{lstlisting}
%
The intended semantics here is that
Line~\ref{line:movecore}
will move
all the symmetry breaking clauses that have been derived to the core
set~$\coreset$.
Line~\ref{line:outputType} claims that the input and the set of
constraints that is now in the core are 
equisatisfiable.
At the time of
writing, the public version of \veripb does not yet support
non-trivial output statements, (i.e., does not provide support for
checking that
the list of constraint identifiers are
precisely the set of constraints in the
core set at the of the proof), and only 
\verb+output NONE+
is supported in the version of \veripb
used in the SAT Competition 2023
\cite{BMMNOT23DocumentationVeriPB}.
However, support for outputting formulas at the end of the proof is
currently being developed and is expected to be available in a not too
distant future.
Line~\ref{line:outputHeader} states that 
the output formula
has $39$ variables and $136$ constraints and
Line~\ref{line:constraints} lists the IDs of those $136$ constraints.  
Finally, the conclusion line~\ref{line:conclusion}
states that neither satisfiability nor unsatisfiability can be
concluded from this proof.  
More details about the intended format and semantics of the output
section in \veripb proofs can be found in the technical
documentation for the SAT Competition 2023
\cite{BMMNOT23DocumentationVeriPB},
but it should be emphasized again that these features of the proof
checker are currently under development and minor changes could and
should be expected on the way from the current tentative specification
to the final finished product.

\vskip 0.2in
\bibliography{refLocal,refArticles,refOther,refBooks,krrlib,paper}
\bibliographystyle{theapa}

\end{document}